\documentclass{article}

\PassOptionsToPackage{numbers, compress}{natbib}

 \usepackage[preprint]{neurips_2026}

\usepackage{microtype}
\usepackage{graphicx}
\usepackage{subcaption}
\usepackage{booktabs} %

\usepackage{amsmath}
\usepackage{amssymb}
\usepackage{mathtools}
\usepackage{amsthm}

\usepackage[colorlinks]{hyperref}
\usepackage{url}
\usepackage{booktabs}
\usepackage{amsfonts}
\usepackage{nicefrac}
\usepackage{enumitem}
\usepackage{xcolor}
\usepackage{makecell}
\usepackage{graphicx}
\usepackage{subcaption}
\usepackage{wrapfig}
\usepackage{algorithm}
\usepackage{algorithmic}
\usepackage{tikz}
\usepackage{dsfont}
\usepackage{bbm}
\usepackage{multirow}

\usepackage[capitalize,noabbrev]{cleveref}

\usepackage{amsmath,amsfonts,bm}

\newcommand{\figleft}{{\em (Left)}}

\newcommand{\figright}{{\em (Right)}}
\newcommand{\figtop}{{\em (Top)}}
\newcommand{\figbottom}{{\em (Bottom)}}

\def\eqref#1{equation~\ref{#1}}

\def\1{\bm{1}}

\DeclareMathAlphabet{\mathsfit}{\encodingdefault}{\sfdefault}{m}{sl}
\SetMathAlphabet{\mathsfit}{bold}{\encodingdefault}{\sfdefault}{bx}{n}

\def\gA{{\mathcal{A}}}

\def\gD{{\mathcal{D}}}

\def\gG{{\mathcal{G}}}
\def\gH{{\mathcal{H}}}

\def\gL{{\mathcal{L}}}

\def\gN{{\mathcal{N}}}

\def\gS{{\mathcal{S}}}
\def\gT{{\mathcal{T}}}

\def\gX{{\mathcal{X}}}

\def\gZ{{\mathcal{Z}}}

\newcommand{\E}{\mathbb{E}}

\newcommand{\R}{\mathbb{R}}

\newcommand{\KL}{D_{\mathrm{KL}}}

\DeclareMathOperator*{\argmax}{arg\,max}
\DeclareMathOperator*{\argmin}{arg\,min}

\usetikzlibrary{automata, positioning, arrows.meta, calc}

\tikzset{
    ->, 
    shorten >=1pt,
    auto,
    node distance=1.5cm,
    state_style/.style={
        state,
        draw=black,
        thick,
        fill=white,
        minimum size=0.6cm,
        font=\sffamily
    },
    edge_style/.style={
        draw=black,
        thick,
        -{Stealth[round, scale=0.9]},
        font=\sffamily
    },
    loop_style/.style={
        min distance=0.9cm,
        looseness=4
    },
    stoch_style/.style={
        draw=black,
        thick,
        dashed,
        -{Stealth[round, scale=0.9]},
        font={\sffamily}
    },
}

\makeatletter
\newcommand{\newreptheorem}[2]{\newtheorem*{rep@#1}{\rep@title}\newenvironment{rep#1}[1]{\def\rep@title{#2 \ref*{##1}}\begin{rep@#1}}{\end{rep@#1}}}
\makeatother

\newtheorem{proposition}{Proposition}
\newreptheorem{proposition}{Proposition}
\newtheorem{lemma}{Lemma}
\newreptheorem{lemma}{Lemma}
\newtheorem{remark}{Remark}
\newreptheorem{remark}{Remark}
\newtheorem{corollary}{Corollary}
\newreptheorem{corollary}{Corollary}
\newtheorem{definition}{Definition}
\newreptheorem{definition}{Definition}

\DeclarePairedDelimiter\norm{\lVert}{\rVert}
\DeclarePairedDelimiter\abs{\lvert}{\rvert}

\definecolor{mypurple}{HTML}{7A4988}
\definecolor{myred}{HTML}{d62728}
\definecolor{mygreen}{HTML}{2ca02c}
\definecolor{myblue}{HTML}{1f77b4}
\definecolor{mygray}{HTML}{a9a9a9}
\definecolor{myorange}{HTML}{ff8c00}

\hypersetup{urlcolor=myblue, citecolor=myblue, linkcolor=myblue}

\providecommand{\ph}[2][]{{\protect\color{black}{#1#2}}} %

\usepackage[textsize=tiny]{todonotes}

\title{Can We Really Learn One Representation to Optimize All Rewards?}

\author{
    \begin{minipage}{\textwidth}
    \centering
    Chongyi Zheng\thanks{Equal contribution.}$\hspace{0.35em}^{1}$ \quad Royina Karegoudra Jayanth$^{*1}$ \quad Benjamin Eysenbach$^1$ \\
    \vspace{0.3em}
    \normalfont $^1$Princeton University \\
    \vspace{0.5em}
    \texttt{chongyiz@princeton.edu \qquad rj5498@princeton.edu}
    \end{minipage}
}

\begin{document}

\maketitle

\begin{abstract}
    As unsupervised pretraining becomes increasingly ubiquitous in reinforcement learning, a more thorough theoretical understanding of these methods becomes of equal importance to their empirical success. We focus on the setting of unsupervised learning via interaction, where the forward-backward (FB) representation learning serves as a prototypical and popular example. 
    In this paper, we shed light on FB by formally contextualizing the method within a broader class of recent methods that use regression to obtain a low-rank approximation of a successor measure ratio. Our analysis clarifies when FB representations can exist and how the low-rank approximation converges in practice. 
    Building upon the theory, we propose a variant of FB that is both more amenable to theoretical understanding and simpler to optimize in practice. 
    Experiments in didactic settings, as well as in $10$ state-based and image-based continuous control domains, demonstrate that our method converges to desired representations with $10^5 \times$ smaller errors than FB, achieving $+24\%$ improved zero-shot performance on average. We also demonstrate that zero-shot policies inferred by our algorithm provide an efficient initialization if the user prefers further fine-tuning on downstream tasks. 
    Our project website is available at \url{https://chongyi-zheng.github.io/onestep-fb}.

\end{abstract}

\vspace{-0.5em}
\section{Introduction}
\label{sec:intro}
\vspace{-0.5em}

\begin{wrapfigure}[18]{R}{0.5\textwidth}
    \vspace{-1.85em}
    \centering
    \includegraphics[width=0.99\linewidth]{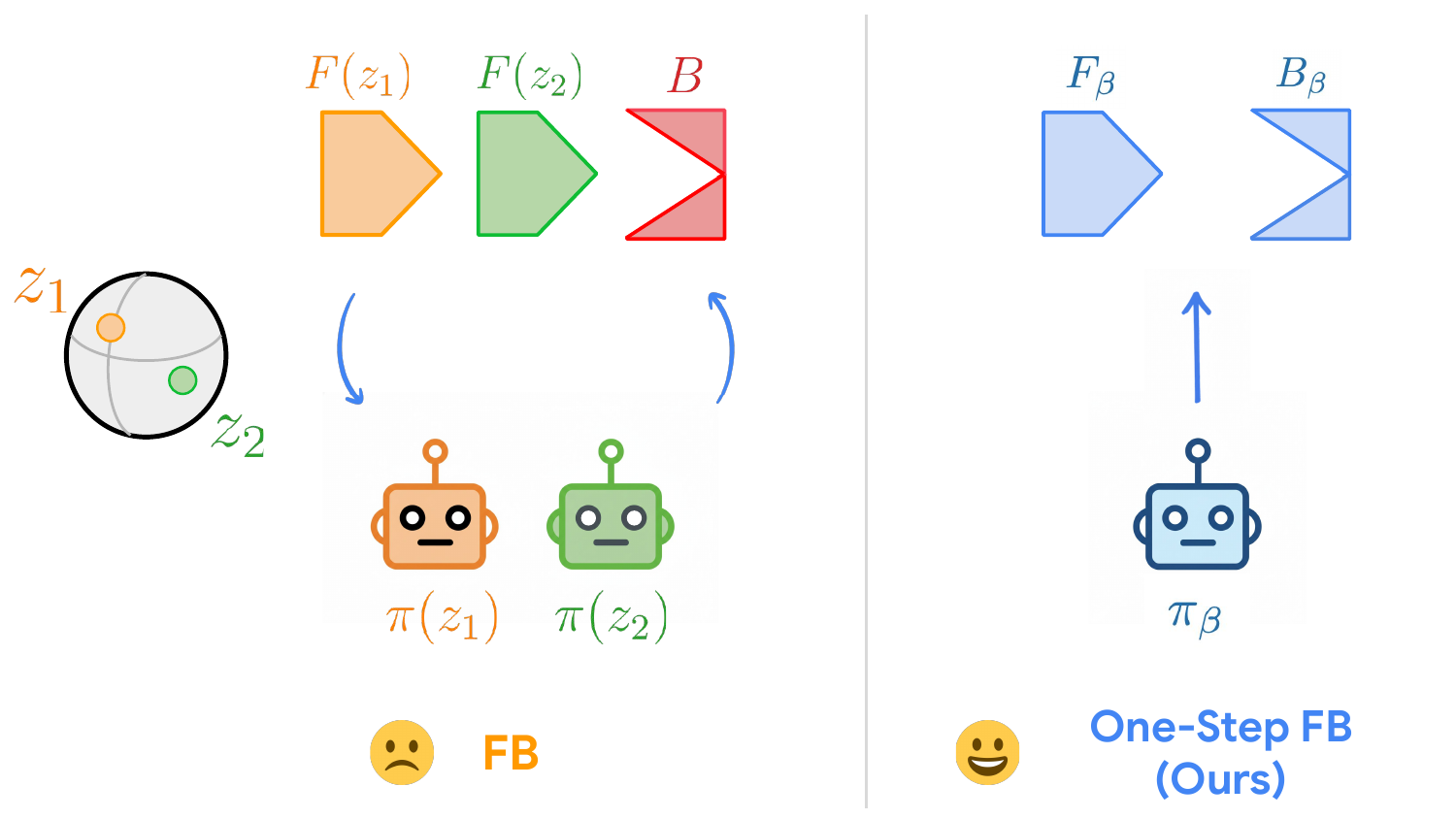}
    \caption{How can we learn a library of policies to quickly maximize new rewards? \figleft \, Forward-backward representation learning (FB)~\citep{touati2021learning} learns bilinear representations to acquire new policies. \figright \, Our theoretical analysis of this method reveals some optimization challenges, which are alleviated through a simplified method that enjoys stable convergence.}
    \label{fig:demystify-fb}
    \vspace{-1em}
\end{wrapfigure}

Large-scale pre-training has reshaped how we build learning systems in vision~\citep{radford2021learning, grill2020bootstrap, assran2023self} and language~\citep{team2023gemini, achiam2023gpt}: a foundation model is trained once on broad data, and then adapted to specific tasks with little or no updates.
In the context of RL, such models are known as behavioral foundation models (BFMs)~\citep{tirinzoni2025zeroshot, sikchi2025fast, li2025bfm}, and ideally acquire behavioral knowledge from unsupervised (reward-free) interactions and later specialize to new tasks with minimal additional learning. Similar to large language models (LLMs), this paradigm can be interpreted as in-context learning for RL: the reward in example trajectories induces a prompt, and the pre-trained BFM responds with the optimal behavior directly.

Forward-backward (FB) representation learning~\citep{touati2021learning} is a prominent attempt to realize a BFM. FB proposes to pre-train a pair of representations that can be combined with a downstream reward to obtain a reward-maximizing policy. Prior work based on FB usually contextualizes this method as learning a low-rank approximation of a successor measure ratio~\citep{touati2022does, bagatella2025td, agarwal2025proto}. To unpack this promise, we will study when such representations can exist and be learned. Broadly, this paper gets at the question:
\begin{gather*}
    \text{\emph{Can one really learn one representation to optimize all rewards in practice?}} 
\end{gather*}
In this paper, we extend the theoretical analysis of FB, aiming to provide additional insights into how and why it works in practice. First, we start by examining the assumption that the ground-truth FB representations always exist, answering the question: \emph{When do the ground-truth FB representations exist?} Answering this question helps us identify the challenge in using low-rank approximation to capture all optimal behaviors. Second, we explicitly reinterpret the FB representation objective as a regression loss, revealing insights into \emph{what the algorithm optimizes} for practitioners. This reinterpretation reveals a connection with fitted Q-evaluation (FQE)~\citep{ernst2005tree, riedmiller2005neural, fujimoto2022should}, motivating us to study the convergence of the low-rank approximation. Third, we therefore construct a new Bellman operator, which is realized by the learning procedure of the FB algorithm. Using this new Bellman operator, we identify another challenge in the convergence of the low-rank approximation. This challenge mainly comes from the circular dependency in FB: the representations and the policy depend on each other.

Building upon our new insights for FB, we propose a simpler alternative to it called \textbf{one-step FB}. 
Our algorithm breaks the circular dependency (Fig.~\ref{fig:demystify-fb}) by learning representations for a \emph{fixed} behavioral policy. In doing so, we can pre-train one step of policy improvement over all the behavioral value functions.
Through didactic experiments, we demonstrate that FB can struggle to converge, while our proposed variant converges to 
$10^5$ smaller errors. Experiments on $8$ state-based benchmark domains and $2$ image-based benchmark domains show that one-step FB is a competitive method for unsupervised pre-training, achieving $+24\%$ improved zero-shot performance on average. In addition, our method provides an efficient initialization for fine-tuning with off-the-shelf RL algorithms. Overall, our method serves as a simpler and more stable plug-and-play alternative to FB, which may appeal to RL practitioners.

\vspace{-0.5em}
\section{Preliminary}
\label{sec:preliminary}
\vspace{-0.5em}

We first define the notation and background mathematics for our analysis. A conceptual description of the prior related work can be found in Appendix~\ref{appendix:related-work}.

We consider a controlled Markov process (CMP)~\citep{bhatt1996occupation} defined by a state space $\gS$, an action space $\gA$, a probability measure of initial states $p_0 \in \Delta(\gS)$, a probability measure of environmental transitions $p: \gS \times \gA \to \Delta(\gS)$, and a discount factor $\gamma \in [0, 1)$, where $\Delta(\ph{\gX})$ denotes the set of all possible probability distributions over a space $\ph{\gX}$.  
When equipped with a reward function $r: \gS \times \gA \to \R$, the CMP becomes a Markov decision process (MDP)~\citep{sutton1998reinforcement}. With slight abuse of notation, we use \emph{probability measure} to denote either the probability mass in discrete CMPs or the probability density in continuous CMPs.

\textbf{Successor measure and Q-value.} For a CMP and a policy $\pi: \gS \to \Delta(\gA)$, the successor measure~\citep{dayan1993improving, janner2020gamma, touati2021learning, eysenbach2022contrastive, zheng2024contrastive, myers2024learning, zheng2025intention} $M^{\pi}: \gS \times \gA \to \Delta(\gS \times \gA)$ defines the probability measure of reaching a \emph{future} state-action pair $(s_f, a_f)$ starting from a current state-action pair $(s, a)$. 
Prior work~\citep{dayan1993improving, barreto2017successor, blier2021learning} has shown that the successor measure is the unique fixed point of a Bellman equation:
\begin{align}
     M^\pi(s_f, a_f \mid s, a) &= (1 - \gamma) \delta(s_f, a_f \mid s, a) + \gamma \E_{ p(s' \mid s, a), \, \pi(a' \mid s')} \left[ M^\pi(s_f, a_f \mid s', a') \right],
     \label{eq:successor-measure-bellman-eq}
\end{align}
where $\delta(\ph{s_f}, \ph{a_f} \mid s, a)$ is the delta measure\footnote{\footnotesize The delta measure is an indicator function for discrete MDPs and a Dirac delta function for continuous MDPs.} centered at the state-action pair $(s, a)$.
For a discrete CMP, the successor measure can be written as a matrix $M^\pi \in \R^{|\gS \times \gA| \times |\gS \times \gA|}$ that is full rank (i.e., $\text{rank}(M^\pi) = |\gS \times \gA|$)~\citep[Lemma 1.6 and Corollary 1.5]{agarwal2019reinforcement}; see Appendix~\ref{appendix:successor-measure-matrix}.

The successor measure can be used to express the Q-value $Q^{\pi}_r(\ph{s}, \ph{a})$ for any reward function $r$:
\begin{align}
    Q_r^{\pi}(s, a) &= \E_{(s_f, a_f) \sim M^{\pi}(s_f, a_f \mid s, a)} [r(s_f, a_f)]. 
    \label{eq:relate-sr-to-q}
\end{align}
This connection disentangles the estimation of the successor measure (pre-training) and the estimation of the Q-value (fine-tuning) into two separate phases, resembling the learning paradigm in LLMs~\citep{brown2020language, ouyang2022training}. Next, we will make this resemblance precise.

\textbf{Unsupervised pre-training in RL and zero-shot RL.}
Algorithms for unsupervised pre-training in RL typically involve two steps: \emph{(Step 1)} pre-training a set of policies and their successor measures in a CMP, and \emph{(Step 2)} performing zero-shot policy adaptation for a specific reward function.
The unsupervised pre-training mainly considers the offline setting~\citep{lange2012batch} (Sec.~\ref{sec:method}), where learning happens on a dataset of transitions $\gD = \{ (s, a, s', a') \}$ collected by some behavioral policy $\pi_{\beta}: \gS \to \Delta(\gA)$. 
We will use zero-shot RL to denote unsupervised pre-training algorithms where policy adaptation does not require updating the neural networks, \textbf{independent of acquiring optimal behaviors}. After adapting the zero-shot policy, one can further fine-tune it using off-the-shelf RL algorithms. See Appendix~\ref{appendix:url-components} for additional components in \emph{Step 1} and \emph{Step 2}.

This work aims to analyze and simplify a prior state-of-the-art pre-training method called \emph{forward-backward representation learning} (FB)~\citep{touati2021learning}. We will include an overview of this algorithm next.

\textbf{Forward-backward representation learning.} FB is an instance of the zero-shot RL algorithms~\citep{bagatella2025td, park2024foundation, zheng2025intention, agarwal2025proto}. During pre-training, FB uses forward-backward representation functions $F(\ph{s}, \ph{a}, \ph{z})$ and $B(\ph{s_f}, \ph{a_f})$ to parameterize the policy $\pi(\ph{a} \mid \ph{s}, \ph{z})$ and its successor measure $M^{\pi}(\ph{s_f}, \ph{a_f} \mid \ph{s}, \ph{a}, \ph{z})$.
Importantly, both the latent-conditioned policy and its associated successor measure depend on the forward representation, forming a circular dependency. See Appendix~\ref{appendix:def-gt-fb-reprs} for the formal definition of ground-truth forward-backward representations $F^{\star}(s, a, z)$ and $B^{\star}(s_f, a_f)$. Prior work~\citep{touati2021learning, touati2022does}~\emph{assumes} the existence of ground-truth FB representations. %
This assumption raises the question: \emph{When do the ground-truth FB representations exist?} We will answer this question using tools from linear algebra and rank matching in Sec.~\ref{subsec:fb-existence}.

After pre-training the latent-conditioned policy and its successor measure, FB finds the optimal policy for any reward function by setting the latent variable to an expectation over backward representations: $z_r = \E_{\rho(s_f, a_f)} [B^{\star}(s_f, a_f) r(s_f, a_f)]$.
See Appendix~\ref{appendix:def-gt-fb-reprs} for the formal definition. 
In the following sections, 
We will also study the convergence of the FB algorithm in practice (Sec.~\ref{subsec:fb-convergence}), motivating us to derive a simpler zero-shot RL algorithm (Sec.~\ref{sec:method}).

\textbf{Least-squares importance fitting.} The goal of probability measure ratio estimation is to predict the ratio $p(\ph{x}) / q(\ph{x})$ between two probability measures $p \in \Delta(\ph{\gX})$ and $q \in \Delta(\ph{\gX})$ over some space $\ph{\gX}$. The most widely adopted approach casts this problem as classification~\citep{qin1998inferences, cheng2004semiparametric, gutmann2010noise}. 
However, least-squares importance fitting (LSIF)~\citep{kanamori2009least, kanamori2008efficient} casts the measure ratio estimation as regression.\footnote{The same loss recurs under different names in the literature~\citep{nachum2019dualdice, kato2025riesz}.}
Specifically, LSIF fits a function $g: \gX \to \R$ to the target measure ratio $p(\ph{x}) / q(\ph{x})$ using samples.
\begin{align}
    \gL_{\text{LSIF}}(g) &= \frac{1}{2} \E_{q(x)} \big[ \big( g(x) - p(x) / q(x) \big)^2 \big] = \frac{1}{2} \E_{q(x)} \big[ g(x)^2 \big] - \E_{p(x)} \left[ g(x) \right] + \text{const.},
    \label{eq:lsif}
\end{align}
where the constant is independent of the learned ratio $g(\ph{x})$.
Compared with the more popular classification loss, this LSIF loss remains well defined when $g(\ph{x})$ is negative. We call $q(\ph{x})$ the \emph{anchor} measure and call $p(\ph{x})$ the \emph{target} measure. In Sec.~\ref{subsec:fb-repr-obj}, we will reinterpret the FB representation objective using the LSIF loss function with a special parameterization of the ratio function.

\vspace{-0.75em}
\section{Understanding Forward-Backward Representation Learning}
\label{sec:understanding-fb}
\vspace{-0.5em}

In this section, we study FB through the lens of linear algebra, LSIF, and contraction mapping, focusing on providing new analysis that extends the theory in FB.
The goal of our theoretical analysis is threefold.
\begin{enumerate}[topsep=1pt, itemsep=1pt, parsep=0pt, partopsep=0pt]
    \item[\S \ref{subsec:fb-existence}] We examine the assumption in prior work that ground-truth FB representations always exist, showing strict rank and dimensionality constraints. 
    Our analysis indicates that a low-rank approximation may induce arbitrary errors in capturing optimal behavior for some rewards.%
    
    \item[\S \ref{subsec:fb-repr-obj}] We next reinterpret the FB representation objective as a temporal-difference variant of a regression loss (Eq.~\ref{eq:lsif}), drawing a connection with FQE. This connection motivates us to study the convergence of the practical FB. %
    
    \item[\S \ref{subsec:fb-convergence}] We construct a new Bellman operator to describe the learning procedure of the practical FB. The failure of applying the Banach fixed-point theorem to this new Bellman operator results in unclear convergence of the practical FB algorithm. %
\end{enumerate}
The main challenge in analyzing and providing guarantees for FB lies in the circular dependency (Fig.~\ref{fig:demystify-fb}). We thus introduce a variant of FB that breaks the circular (Sec.~\ref{sec:method}). The resulting algorithm is not only simpler but also enjoys more stable learning and clearer convergence (Sec.~\ref{sec:experiments}).

\subsection{When Do the Ground-Truth FB Representations Exist?}
\label{subsec:fb-existence}

Prior work assumes that the ground-truth FB representations exist (Definition 1 of~\citet{touati2021learning} and Theorem 2 of~\citet{touati2022does}), raising the question: \emph{When do the ground-truth FB representations exist?} We will explicitly study this question using tools from linear algebra and rank matching (Lemma~\ref{lemma:successor-measure-expr-and-rank}). Specifically, we introduce two sets of necessary conditions for the ground-truth FB representations to hold: Proposition~\ref{prop:fb-existence} provides insights to understand the practical FB algorithm, and Proposition~\ref{prop:fb-cylinder-proj} reveals a criterion for examining its convergence.

The first set of necessary conditions starts with the rank of the ground-truth FB representations. We show that unless the representation dimension is infinitely large, the FB's zero-shot policies are not necessarily optimal for maximizing all rewards.

\begin{proposition}[Informal]
    \label{prop:fb-existence}
    Given any discrete CMP, a finite latent space $\gZ$, and a marginal measure $\rho$, any FB representation matrices $F^{\star}_\gZ \in \R^{\abs{\gZ \times \gS \times \gA} \times d}$ and $B^{\star} \in \R^{d \times \abs{\gS \times \gA}}$ that encode this CMP's successor measure (Definition~\ref{def:fb-definition} and Definition~\ref{def:fb-adaptation}) must satisfy the following properties:
    \begin{enumerate}[topsep=1pt, itemsep=1pt, parsep=0pt, partopsep=0pt]
        \item The representation dimension $d$ is at least $\abs{\gS \times \gA}$: $d \geq \abs{\gS \times \gA}$.
        \item The rank of the matrix $F^{\star}_{\gZ}$ is at least $\abs{\gS \times \gA}$ and at most $d$: $\abs{\gS \times \gA} \leq \text{rank}(F^{\star}_Z) \leq d$.
        \item The rank of the matrix $B^{\star}$ equals to $\abs{\gS \times \gA}$: $\text{rank}(B^{\star}) = \abs{\gS \times \gA}$.
        \item The matrix $B^{\star}$, components of matrix $F^{\star}_\gZ$, the successor measures must satisfy:
        \begin{align*}
            B^{\star} = F^{\star +}_{z_1} M^{\pi(\ph{a} \mid \ph{s}, z_1)} / \rho = \cdots = F^{\star +}_{z_{\abs{\gZ}}} M^{\pi(\ph{a} \mid \ph{s}, z_{\abs{\gZ}})} / \rho,
        \end{align*}
        where $\ph{X}^+$ denotes the pseudoinverse~\citep{moore1920reciprocal, Bjerhammar1951ApplicationOC} of the matrix $\ph{X}$.
    \end{enumerate}
\end{proposition}
See Appendix~\ref{appendix:fb-repr-existence} for the complete discussion and a proof. While our rank conditions are derived with finite states and actions, we extend them to continous CMPs and identify an irreducible error in the low-rank approximation used in the practical FB. We discuss two important implications from our Proposition~\ref{prop:fb-existence} next.

First, our rank analysis implies some challenges in using neural networks to express the FB representations. In machine learning, one usually invokes the Universal Function Approximation Theorem~\citep{hornik1989multilayer} to argue that they can perfectly express a \emph{finite-dimensional} function of interest with large enough neural networks. However, our Proposition~\ref{prop:fb-existence} suggests that when the states and actions are continuous, both the ground-truth $F^\star$ and $B^{\star}$ lie in an \emph{infinite-dimensional} Hilbert space~\citep{berlinet2011reproducing}. In this case, arbitrarily expressive neural networks still cannot represent the desired representations.

\begin{corollary}[Learnability]
    \label{corollary:finite-d-is-insufficient}
    For continuous CMPs with $\abs{\gS \times \gA} \to \infty$, the inner product $\langle F^{\star}(s, a, z), B^{\star}(s, a) \rangle$ lies in an infinite-dimensional Hilbert space $\gH$, which contains the latent space $\gZ \subset \gH$. Thus, neural networks are not able to express the inner product.
\end{corollary}
The challenge not only stems from fitting the inner product (one can fit the inner product as a kernel function~\citep{berlinet2011reproducing, lu2019deeponet}), but also stems from the infinite-dimensional latent variable $z \in \gH$ in the input of the forward representations. As mentioned in prior work~\citep{touati2021learning, touati2022does}, using a finite representation dimension results in a low-rank approximation of the ground-truth FB. 

Second, when learning a low-rank approximation for the desired ratio (Eq.~\ref{eq:fb-sm-ratio}) as in the practical FB, our rank analysis suggests that we can incur arbitrary errors on the optimal Q-value predictions.
\begin{corollary}[Arbitrary Errors; Informal]
    \label{corollary:low-rank-arbitrary-error}
    For the low-rank FB representations ($d < \abs{\gS \times \gA}$) learned by the practical FB algorithm, there exist some reward functions such that errors in the optimal Q-value prediction are arbitrarily large.
\end{corollary}
See Appendix~\ref{appendix:low-rank-arbitrary-err} for the complete statement and the proof. Unlike the function approximation errors in neural networks~\citep{hornik1989multilayer}, these arbitrary errors in the optimal Q-value prediction are irreducible. Therefore, the corresponding zero-shot policy may not be optimal for maximizing the reward.

Our second set of necessary conditions studies the ground-truth forward representations under reward transformations.
Recall that the Q-value is equivariant to an arbitrary positive affine transformation on the reward~\citep{russell1995modern, ng1999policy} (See Lemma~\ref{lemma:q-equivariant}). 
We translate this into an invariance property that the ground-truth forward representations must satisfy:
\begin{proposition}[Informal]
    \label{prop:fb-cylinder-proj}
    For a scalar $\nu > 0$, and an offset $\xi \in \R$, the ground-truth forward representations $F^{\star}$ are invariant under affine transformations on the reward, i.e., $F^{\star}(s, a, z_{\nu r  + \xi}) = F^{\star}(s, a, z_r)$.
\end{proposition}
See Appendix~\ref{appendix:equivariant-to-affine-r} for further discussions and the proof. This proposition underscores a \emph{necessary} condition: any $F$ that failed to satisfy Proposition~\ref{prop:fb-cylinder-proj} must not equal $F^{\star}$ (contrapositive). We will use this proposition as the criterion for examining whether the practical FB converges to the ground-truth representations (Sec.~\ref{subsec:fb-didactic-exp}). In the next section, we reinterpret the representation objective in FB. Our understanding provides insights to simplify the algorithm in Sec.~\ref{sec:method}.

\vspace{-0.5em}
\subsection{What Does the FB Representation Objective Minimize?} 
\label{subsec:fb-repr-obj}

We now reinterpret the representation objective used in FB.
The main idea is to derive a temporal-difference (TD) variant of the LSIF loss (Eq.~\ref{eq:lsif}) that minimizes a Bellman error similar to FQE. This TD LSIF loss ends up being equivalent to the representation loss used in~\citet{touati2021learning}.

In the context of LSIF, FB chooses to set the ratio function in Eq.~\ref{eq:lsif} as an inner product: $g(s, a, z, s_f, a_f) \triangleq F(s, a, z)^{\top} B(s_f, a_f)$, set the target measure to $p(s, a, z, s_f, a_f) \triangleq M^{\pi}(s_f, a_f \mid s, a, z)$, and set the anchor measure to $q(s, a, z, s_f, a_f) \triangleq \rho(s_f, a_f)$, resulting in the following loss: 
\begin{align}
    \gL_{\text{MC FB}}(F, B) = \frac{1}{2} \E_{\substack{p^{\pi_{\beta}}(s, a), p_{\gZ}(z), \\ \rho(s_f, a_f)}} \left[ \big( F(s, a, z)^{\top} B(s_f, a_f) - \frac{M^{\pi}(s_f, a_f \mid s, a, z)}{\rho(s_f, a_f)} \big)^2 \right].
    \label{eq:mc-fb}
\end{align}
We call this loss the Monte Carlo (MC) forward-backward representation loss $\gL_{\text{MC FB}}$ because, as mentioned in Sec.~\ref{sec:preliminary}, computing it requires \emph{on-policy} samples from the successor measure $M^{\pi}$.

We next derive a temporal-difference version of this same loss.
First, we replace the successor measure in $\gL_{\text{MC FB}}$ using the recursive Bellman equation in Eq.~\ref{eq:successor-measure-bellman-eq}. Second, we use target FB representation functions $\bar{F}$ and $\bar{B}$ to replace the ground-truth ratio at the next time step, akin to the target networks used in deep Q-learning~\citep{mnih2015human}. The resulting loss function minimizes a Bellman error\footnote{A similar formulation has been mentioned in prior work (See Appendix B of~\citet{touati2022does}), but from the perspective of minimizing a matrix norm. }: 
\begin{align}
     \gL_{\text{TD FB}}(F, B) &= \frac{1}{2} \E_{\substack{p^{\pi_{\beta}}(s, a), \, \rho(s_f, a_f), \\ p_{\gZ}(z), \, p(s' \mid s, a), \, \pi(a' \mid s', z) }} \left[ \left( F(s, a, z)^{\top} B(s_f, a_f) - y \right)^2 \right], \label{eq:td-fb-brm} \\
     y &= (1 - \gamma) \delta(s_f, a_f \mid s, a) / \rho(s_f, a_f) + \gamma \bar{F}(s', a', z)^{\top} \bar{B}(s_f, a_f). \nonumber
\end{align}
See Appendix~\ref{appendix:FB-repr-loss} for the complete derivation. 
We call this loss the temporal-difference (TD) forward-backward representation loss $\gL_{\text{TD FB}}$. Like FQE, the TD FB loss can be computed using transition samples and target networks. Unlike FQE, this loss function estimates the successor measure ratio instead of the Q-value. Rearranging terms in $\gL_{\text{TD FB}}$, we recover the original FB representation objective\footnote{We recover the FB representation objective up to a constant scaling factor $1 - \gamma$.}: 
\begin{align}
    \gL_\text{TD FB}(F, B) &= \frac{1}{2} \E_{ \substack{p^{\pi_{\beta}}(s, a), \, \rho(s_f, a_f), \\ p_{\gZ}(z), \, p(s' \mid s, a), \, \pi(a' \mid s', z) } } \big[ \big( F(s, a, z)^{\top} B(s_f, a_f) - \gamma \bar{F}(s', a', z)^{\top} \bar{B}(s_f, a_f) \big)^2 \big] \nonumber \\
    &\hspace{1.1em} - (1 - \gamma) \E_{p^{\pi_{\beta}}(s, a), \, p_{\gZ}(z)} \big[ F(s, a, z)^{\top} B(s, a) \big].
    \label{eq:td-fb}
\end{align}
See Appendix~\ref{appendix:FB-repr-loss} for the derivation. Importantly, we can now interpret the representation objective in FB as
performing approximate value iteration, which has a clear connection with the standard Bellman operator and the Banach fixed-point theorem~\citep{banach1922operations}. These theoretical connections motivate us to study whether FB admits a similar convergence guarantee in practice.

\subsection{Does the Practical FB Algorithm Admit a Stable Convergence?}
\label{subsec:fb-convergence}

Our analysis proceeds in two steps. 
\emph{First}, the resemblance between FB (Eq.~\ref{eq:td-fb}) and FQE motivates us to define a new Bellman operator, called the FB Bellman operator. Similar to the relationship between the standard Bellman operator and FQE, we interpret the FB algorithm as iteratively applying the FB Bellman operator using samples from the dataset. \emph{Second}, we use the Banach fixed-point theorem~\citep{banach1922operations} to analyze the asymptotic fixed point of the FB Bellman operator, studying whether FB admits approximate convergence. 

Similar to Sec.~\ref{subsec:fb-existence}, we consider discrete CMPs with a finite number of states and actions. We also assume the transition measure $p(\ph{s'} \mid \ph{s}, \ph{a})$ and the marginal probability measure $\rho(\ph{s_f}, \ph{a_f})$ are known. Under this setup, we can define a new \emph{FB Bellman operator}.
\begin{definition}
For any two functions $f: \gS \times \gA \times \gZ \to \gZ$ and $b: \gS \times \gA \to \gZ$ that induce the latent-conditioned policy $\pi(a \mid s, z) = \delta \left( a \mid \argmax_{a \in \gA} f(s, a, z)^{\top} z \right)$, the FB Bellman operator $\gT_{\text{FB}}$ applies to the inner product of $f(s, a, z)$ and $b(s_f, a_f)$:
\begin{align*}
    & \gT_{\text{FB}} \left( f(s, a, z)^{\top} b(s_f, a_f) \right) \triangleq  (1 - \gamma) \frac{\delta(s_f, a_f \mid s, a)}{\rho(s_f, a_f)} + \gamma \E_{p(s' \mid s, a), \pi(a' \mid s', z)} \left[ f(s', a', z)^{\top} b(s_f, a_f) \right].
\end{align*}
\end{definition}
As discussed in Sec.~\ref{subsec:fb-repr-obj}, FB's representation objective can be viewed as minimizing a Bellman error (Eq.~\ref{eq:td-fb-brm}). Comparing the functional form of Eq.~\ref{eq:td-fb-brm} to the FB Bellman operator, we can draw a key observation between FB and the FB Bellman operator: the TD FB loss (Eq.~\ref{eq:td-fb}) in FB is iteratively applying the FB Bellman operator $\gT_{\text{FB}}$ to the FB representations from the previous iteration, resembling the bridge between FQE and the standard Bellman operator~\citep{mnih2015human, lillicrap2015continuous}.

The standard Bellman operator is a $\gamma$-contraction and admits a unique fixed point~\citep{banach1922operations}. Unfortunately, the FB Bellman operator is \emph{not} a $\gamma$-contraction because of the circular dependency between the latent-conditioned policy and its associated successor measure (See Definition~\ref{def:fb-definition} and Fig.~\ref{fig:demystify-fb}).
\begin{proposition}[Informal]
    \label{prop:fb-gamma-contraction}
    The FB Bellman operator $\gT_{\text{FB}}$ is not a $\gamma$-contraction. Thus, the Banach fixed-point theorem is not applicable to the FB Bellman operator.
\end{proposition}
See Appendix~\ref{appendix:FB-algo-fixed-point} for the proof. Our analysis does \emph{not} suggest that iteratively applying the FB Bellman operator cannot converge to a fixed point, just that the standard proof strategy is not applicable. Indeed, both our discussion in Sec.~\ref{subsec:fb-existence} and prior work~\citep{touati2022does} have already revealed that there exist \emph{multiple} fixed points for the FB Bellman operator.\footnote{For example, applying a rotation (orthonormal) matrix to both FB representations does not change their inner products.} Therefore, whether the FB algorithm converges stably to a fixed-point remains an open problem. Answering this question might require tools such as the Lefschetz fixed-point theorem~\citep{lefschetz1926intersections} or the Lyapunov stability~\citep{lyapunov1992general}, which we leave for future research. One alternative method that might converge is to first fit the changing successor measure using a single network and then conduct bilinear decompositions into FB representations. However, the circular dependency (Fig.~\ref{fig:demystify-fb}) persists in this variant. In Sec.~\ref{subsec:fb-didactic-exp}, we will use didactic experiments to demonstrate that FB struggles to converge in practice. 

\vspace{-0.5em}
\section{A Simplified Algorithm for Unsupervised Pre-training in RL}
\label{sec:method}
\vspace{-0.5em}

In this section, we derive a variant of FB, building upon our theoretical understanding in Sec.~\ref{sec:understanding-fb}.
Unlike FB, we take as input a dataset sampled from some \emph{behavioral policy} $\pi_{\beta}(\ph{a} \mid \ph{s})$ and fit the successor measure ratio of this fixed policy (Fig.~\ref{fig:demystify-fb}). 
Our method is conceptually simpler: pre-training consists of one step of policy improvement over all behavioral value functions on the dataset. 
Empirically, our proposed variant of FB achieves more stable convergence and higher performance (Sec.~\ref{sec:experiments}). 
Similar to prior work~\citep{park2024foundation, zheng2025intention, ghosh2023reinforcement}, our method also provides an efficient policy initialization for further fine-tuning.

\subsection{Breaking the Circular Dependency in FB}
\label{subsec:onestep-fb-repr-obj}

In the same way that FB uses TD FB loss to fit successor measure ratios, we optimize a forward representation function $F_{\beta}$ and a backward representation function $B_{\beta}$ to fit the fixed successor measure ratio of the behavioral policy $\pi_{\beta}(\ph{a} \mid \ph{s})$. We use notations similar to FB, but $F_{\beta}$ and $B_{\beta}$ are semantically different from $F$ and $B$, highlighting the dependency on the behavioral policy using the subscript $\beta$.
We also introduce a new TD one-step FB loss to learn the new FB representations:
\begin{align}
     \gL_{\text{TD one-step FB}}(F_{\beta}, B_{\beta}) &= \frac{1}{2} \E_{ \substack{p^{\pi_{\beta}}(s, a) \, \rho(s_f, a_f), \\ p(s' \mid s, a), \, {\color{orange} \pi_{\beta}(a' \mid s')} } } \big[ \big( {\color{orange} F_{\beta}(s, a)}^{\top} B_{\beta}(s_f, a_f) - \gamma {\color{orange} \bar{F}_{\beta}(s', a')}^{\top} \bar{B}_{\beta}(s_f, a_f) \big)^2 \big] \nonumber \\
    &\hspace{1.1em} - (1 - \gamma) \E_{p^{\pi_{\beta}}(s, a)} \big[ {\color{orange} F_{\beta}(s, a)}^{\top} B_{\beta}(s, a) \big],
    \label{eq:td-onestep-fb}
\end{align}
where $\bar{F}_{\beta}$ and $\bar{B}_{\beta}$ are target representation functions. Unlike the TD FB loss, this TD one-step FB loss samples the next action $a'$ from the behavioral policy $\pi_{\beta}(\ph{a'} \mid \ph{s'})$ and the forward representation function $F_{\beta}$ does \emph{not} depend on a latent variable. Importantly, the TD one-step FB loss admits a clear fixed-point as the behavioral policy is \emph{fixed}: the learning procedure is a supervised learning problem. In theory, we can also first regress the fixed successor measure ratio and then conduct bilinear decompositions into $F_{\beta}$ and $B_{\beta}$, enjoying stable convergence.

Similar to FB, we additionally regularize the backward representations to be orthonormal~\citep{touati2021learning, touati2022does}:
\begin{align}
    \gL_{\text{ortho}}(B_{\beta}) &= \E_{\substack{\rho(s_f, a_f) \\ \rho(s_f', a_f')}} \left[ (B_{\beta}(s_f, a_f)^{\top} B_{\beta}(s_f', a_f') )^2 - \norm*{B_{\beta}(s_f, a_f)}_2^2 - \norm*{B_{\beta}(s_f', a_f')}_2^2 \right].
    \label{eq:onestep-fb-ortho-reg}
\end{align}
The complete new representation objective contains both the TD one-step FB loss and the orthonormalization regularization, with $\lambda_{\text{ortho}}$ controlling the regularization strength:
\begin{align}
    \gL(F_{\beta}, B_{\beta}) = \gL_{\text{TD one-step FB}}(F_{\beta}, B_{\beta}) + \lambda_{\text{ortho}} \gL_{\text{ortho}}(B_{\beta}).
    \label{eq:onestep-fb-repr-obj}
\end{align}
In Appendix~\ref{appendix:onestep-fb-svd-connection}, we discuss a connection between our method and the singular value decomposition (SVD) of the behavioral successor measure. We will use the learned forward representations to derive a latent-conditioned policy in our algorithm.

\vspace{-0.5em}
\subsection{Learning a Policy Using Forward Representations}
\vspace{-0.5em}

Our approach for learning a policy is similar to FB.
Specifically, we select actions to maximize the inner products between the forward representation $F_{\beta}$ and a latent variable $z$ sampled from the latent prior $p_\gZ$. For discrete CMPs, we use a softmax policy with temperature $\tau_{\text{policy}}$:
\begin{align*}
    \pi(a \mid s, z) =  \frac{\exp \left( \tau_{\text{policy}} F_{\beta}(s, a)^{\top} z \right)}{\sum_{a' \in \gA} \exp \left( \tau_{\text{policy}} F_{\beta}(s, a')^{\top} z \right)}, \quad z \sim p_{\gZ}(z).
\end{align*}
For CMPs with continuous actions, we explicitly learn a Gaussian policy by the reparameterized policy gradient trick~\mbox{\citep{haarnoja2018soft}} with an additional behavioral-cloning regularization~\citep{fujimoto2021minimalist} for the offline setting, with $\lambda_{\text{BC}}$ controlling the regularization strength:
\begin{align}
    \gL(\pi) = -\E_{ p^{\pi_{\beta}}(s, a_{\beta}), \, p_{\gZ}(z), \, \pi(a \mid s, z) } \left[ F_{\beta}(s, a)^{\top} z + \lambda_{\text{BC}} \log \pi(a_{\beta} \mid s, z) \right].
    \label{eq:onestep-fb-policy-obj}
\end{align}

\begin{algorithm}[t]
	\caption{One-step forward-backward representation learning}
	\label{alg:onestep-fb}
	\begin{algorithmic}[1]
        \STATE \algorithmicinput: The dataset $\gD$, the forward representation $F_{\beta}^{\theta}$, the backward representation $B_{\beta}^{\omega}$, the latent-conditioned policy $\pi^{\eta}$, the latent prior $p_\gZ$, the target forward representation $F_{\beta}^{\bar{\theta}}$, the target backward representation $B_{\beta}^{\bar{\omega}}$.
        \FOR{each iteration}
            \STATE Sample a batch of transitions $\{(s, a, s', a') \} \sim \gD$ and a batch of latents $\{z\} \sim p_{\gZ}(z)$.
            \STATE Train the forward representation $F_{\beta}^{\theta}$ and the backward representation $B_{\beta}^{\omega}$ by minimizing $\gL(\theta, \omega)$ (Eq.~\ref{eq:onestep-fb-repr-obj}).
            \STATE Train the policy $\pi^{\eta}$ by minimizing $\gL(\eta)$ (Eq.~\ref{eq:onestep-fb-policy-obj}).
            \STATE Update the target forward representations $F_{\beta}^{\bar{\theta}}$ and the target backward representations $B_{\beta}^{\bar{\omega}}$ using Polyak averages. 
        \ENDFOR
		\STATE \algorithmicoutput: $F_{\beta}^{\theta}$, $B_{\beta}^{\omega}$, and $\pi^{\eta}$.
	\end{algorithmic}
\end{algorithm}

After pre-training, we can use the policy and representations to adapt to any reward function, akin to FB. Given a downstream task, we infer the task-specific latent variable $z_r^{\beta} = \E_{\rho(s_f, a_f)} \left[ B_\beta(s_f, a_f) r(s_f, a_f) \right]$ and use it to index the latent-conditioned policy $\pi(a \mid s, z_r^{\beta})$. Importantly, this policy adaptation performs one step of policy improvement on the behavioral Q-value (Appendix~\ref{appendix:onestep-fb-policy-improvement}), similar to the generalized policy improvement in~\citet{barreto2017successor}.

\textbf{Algorithm summary.} 
In Alg.~\ref{alg:onestep-fb}, we summarize our new algorithm, one-step FB, and our open-source implementation is available online\footnote{\url{https://github.com/chongyi-zheng/onestep-fb}}. 
Starting from the existing FB algorithm, implementing our method makes two simple changes: \emph{(1)} remove the latent variable from the input of the forward representation, and \emph{(2)}, in the representation loss, sample the next action from the dataset instead of the policy. We use neural networks to parameterize the new FB representations $F_{\beta}^{\theta}(\ph{s}, \ph{a})$ and $B_{\beta}^{\omega}(\ph{s_f}, \ph{a_f})$, and the policy $\pi^{\eta}(\ph{a} \mid \ph{s}, \ph{z})$.

\vspace{-0.5em}
\section{Experiments}
\label{sec:experiments}
\vspace{-0.5em}

Our experiments study whether one-step FB is a simpler and more stable variant of FB. First, we will use a simple discrete CMP to verify the theory in Sec.~\ref{sec:understanding-fb}, empirically testing whether FB and one-step FB converge to their desired representations. Second, we compare one-step FB to prior work in standard offline RL and offline-to-online RL benchmarks, measuring zero-shot performance and the benefits of offline fine-tuning. Following prior work~\mbox{\citep{park2025flow}}, all experiments show means and standard deviations across $8$ random seeds for state-based tasks (4 random seeds for image-based tasks).

\subsection{The Failure Mode of the FB Algorithm}
\label{subsec:fb-didactic-exp}

\begin{wrapfigure}[10]{R}{0.5\textwidth}
    \vspace{-1.4em}  %
    \centering
    \begin{tikzpicture}
        \node[state_style] (s0) {$s_0$};
        \node[state_style, accepting, left of=s0] (s1) {$s_1$};
        \node[state_style, accepting, right of=s0] (s2) {$s_2$};
        \path[edge_style] (s0) edge [loop above, loop_style, in=60, out=120] node {$a_0$} (s0);
        \path[edge_style] (s0) edge [bend left=0] node[below] {$a_1$} (s1);
        \path[edge_style] (s0) edge [bend right=0] node[below] {$a_2$} (s2);
    \end{tikzpicture}
    \caption{\footnotesize \textbf{The three-state CMP.} Agents start from state $s_0$ and take action $a_i$ ($i = 0, 1, 2$) to determinstically transit into state $s_i$. States $s_1$ and $s_2$ are both absorbing states. Sections~\ref{subsec:fb-didactic-exp} and~\ref{subsec:onestep-fb-didactic-exp} will use this simple MDP to study the convergence of the FB and the one-step FB algorithms.}
    \label{fig:didactic-cmp}
\end{wrapfigure}
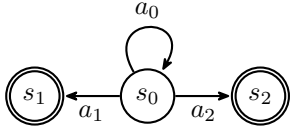

\begin{figure*}[t]
    \vspace{-3em}
    \centering
    \includegraphics[width=.99\linewidth]{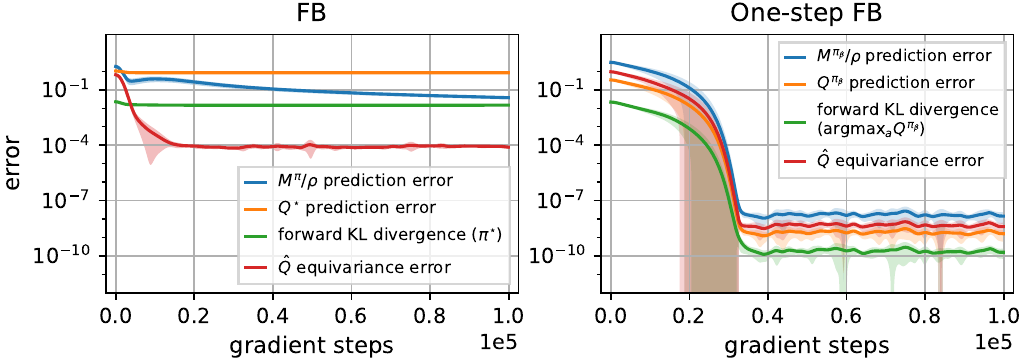}
    \caption{\footnotesize \textbf{Learning FB representations in the three-state CMP (Fig.~\ref{fig:didactic-cmp}).} \figleft \, After training for $10^5$ gradient steps, FB fails to converge to a pair of ground-truth FB representations. \figright \, Given a fixed policy, one-step FB exactly fits the ground-truth one-step FB representations within $4 \times 10^4$ gradient steps, suggesting that our method is simpler and more stable. These observations are consistent with our theory (Sec.~\ref{subsec:fb-convergence}) and the motivation for developing a new method (Sec.~\ref{subsec:onestep-fb-repr-obj}). 
    }
    \label{fig:didactic-exp}
    \vspace{-2em}  %
\end{figure*}

\emph{Does the practical FB algorithm converge to the fixed point characterized in Sec.~\ref{subsec:fb-existence}?}
We test the convergence of the FB algorithm by training it for a long time in a simple CMP (Fig.~\ref{fig:didactic-cmp}) with three states and three actions. We choose this discrete CMP because we can compute the successor measure and the optimal Q-value analytically. Using the MC FB loss (Eq.~\ref{eq:mc-fb}) for FB, which is the analytical analogy of the TD FB loss (Eq.~\ref{eq:td-fb}), we learn the FB algorithm for $10^5$ gradient steps and then analyze prediction errors and forward KL divergence of the latent-conditioned policy.
See Appendix~\ref{appendix:didactic-exp-fb} for implementation details.
We will track several metrics (See Appendix~\ref{appendix:didactic-exp-fb} for formal definitions) with the aim of answering the following questions:
\begin{enumerate}[topsep=1pt, itemsep=1pt, parsep=0pt, partopsep=0pt]
    \item Do the learned representations accurately reflect the successor measure ratio? 
    We compute the successor measure ratio prediction error $\epsilon_{\text{SMR}}$.
    \item Do the learned representations accurately reflect the Q values of reward-maximizing policies (Definition~\ref{def:fb-adaptation})? We measure this error as the optimal Q-value prediction error $\epsilon_{Q^{\star}}$.
    \item How similar are the learned policies to the reward-maximizing policies (Definition~\ref{def:fb-adaptation})? We measure the forward KL divergence between the latent-conditioned policy and the optimal policy $\text{KL}_{\pi^{\star}}$.
    \item Do the optimal Q-value predictions satisfy the equivariance property of universal value functions (Proposition~\ref{prop:fb-cylinder-proj})? We measure the equivariance error of Q predictions $\epsilon_{\text{equiv}}$.
\end{enumerate}

Along several metrics (Fig.~\ref{fig:didactic-exp}~\figleft), we observe high errors, even asymptotically, suggesting that the FB algorithm might not converge. For example, the prediction error of the successor measure ratio converges to $\epsilon_{\text{SMR}} = 4 \times 10^{-2}$ (contrary to Definition~\ref{def:fb-definition}). Similarly, the high policy KL divergence ($\text{KL}_{\pi^{\star}} = 10^{-2}$)
indicate that the FB algorithm failed to enable optimal policy adaptation (contrary to Definition~\ref{def:fb-adaptation}).
Importantly, since the optimal Q-value admits equivariance to an affine transformation, failing to satisfy this property ($\epsilon_{\text{equiv}} = 10^{-4}$) provides key evidence to show that the FB algorithm is \emph{not} converging to the optimal Q-value.
In Appendix~\ref{appendix:confounding}, we discuss potential confounding effects to clarify the observation that the practical FB algorithm fails to converge.

\vspace{-0.5em}
\subsection{The Convergence of the One-Step FB Algorithm}
\label{subsec:onestep-fb-didactic-exp}
\vspace{-0.5em}

We next perform a similar analysis of our simplified algorithm (one-step FB), checking whether the theoretically promised properties (Sec.~\ref{sec:method}) are borne out in practice.
Comparing one-step FB to FB is challenging because they have different fixed points. The closest apples-to-apples comparison is to measure whether one-step FB converges to its respective fixed point (see Sec.~\ref{sec:method}). For example, we will measure whether the representations learned by one-step FB encode the successor measure of a fixed policy $\pi_{\beta}(\ph{a} \mid \ph{s})$, not any reward-maximizing policy. 

We reuse the three-state CMP in Fig.~\ref{fig:didactic-cmp} and train the one-step FB representations for $10^5$ gradient steps using the MC one-step FB loss (Eq.~\ref{eq:mc-onestep-fb}). See Appendix~\ref{appendix:didactic-exp-onestep-fb} for implementation details. We will track metrics similar to Sec.~\ref{subsec:fb-didactic-exp}, aiming to answer questions related to the Q-value of the fixed policy $Q^{\pi_\beta}_r$ (See Appendix~\ref{appendix:didactic-exp-onestep-fb} for formal definitions).

Results in Fig.~\ref{fig:didactic-exp}~\figright~suggest that all these metrics converge to small numbers ($\leq 10^{-7}$) within $4 \times 10^{4}$ gradient steps, helping to verify the convergence of one-step FB. In particular, the learned one-step FB representations are equivariant to an affine transformation in rewards ($\epsilon_{\text{equiv}} = 5 \times 10^{-9}$), which is consistent with the property of any Q-value (Lemma~\ref{lemma:q-equivariant}).

Overall, these didactic experiments show that one-step FB enjoys strong convergence properties. Our next section studies whether these benefits carry over into higher-dimensional continuous control tasks on standard benchmarks.

\vspace{-0.75em}
\subsection{Comparing One-Step FB to Prior Unsupervised RL Methods}
\label{subsec:offline-rl-experiments}
\vspace{-0.5em}

\begin{table*}[t]
\vspace{-1.5em}
\caption{\footnotesize \textbf{Zero-shot evaluation on ExORL and OGBench benchmarks.} One-step FB achieves the best or near-best performance on $5$ out of $10$ domains, outperforming FB by $1.4 \times$ on average (${\color{mygreen} +}{\color{myred} -}$). Following prior work~\citep{park2025flow}, we average results over 8 seeds (4 seeds for image-based tasks) and bold values within $95\%$ of the best performance for each domain. See Table~\ref{tab:offline-zero-shot-rl-eval} for full results. }
\label{tab:offline-zero-shot-rl-eval-agg}
\begin{center}
\setlength{\tabcolsep}{3.5pt}
\scalebox{0.725}
{
\begin{tabular}{lccccccc}
\toprule
Domain & Laplacian & BYOL-$\gamma$ & TD-JEPA & ICVF & HILP & FB & One-Step FB \\
\midrule
\texttt{walker (4 tasks)} & $228 \pm 2$ & $227 \pm 2$ & $316 \pm 12$ & $\mathbf{619} \pm 23$ & $393 \pm 108$ & $400 \pm 40$ & $379 \pm 26$ ({\color{myred}-21})\\
\texttt{cheetah (4 tasks)} & $125 \pm 41$ & $127 \pm 39$ & $185 \pm 7$ & $187 \pm 13$ & $116 \pm 78$ & $271 \pm 46$ & $\mathbf{378} \pm 56$ ({\color{mygreen}+107})\\
\texttt{quadruped (4 tasks)} & $462 \pm 35$ & $496 \pm 35$ & $461 \pm 115$ & $546 \pm 37$ & $352 \pm 59$ & $246 \pm 31$ & $\mathbf{645} \pm 15$ ({\color{mygreen}+399})\\
\texttt{jaco (4 tasks)} & $3 \pm 1$ & $3 \pm 0$ & $13 \pm 4$ & $\mathbf{23} \pm 3$ & $20 \pm 5$ & $10 \pm 4$ & $\mathbf{22} \pm 4$ ({\color{mygreen}+12})\\
\midrule
\texttt{antmaze large navigate (5 tasks)} & $9 \pm 1$ & $21 \pm 2$ & $14 \pm 9$ & $23 \pm 3$ & $\mathbf{34} \pm 2$ & $25 \pm 5$ & $30 \pm 9$ ({\color{mygreen}+5}) \\
\texttt{antmaze teleport navigate (5 tasks)} & $3 \pm 1$ & $16 \pm 4$ & $4 \pm 1$ & $\mathbf{29} \pm 3$ & $19 \pm 6$ & $16 \pm 8$ & $11 \pm 6$ ({\color{myred}-5}) \\
\texttt{cube single play (5 tasks)} & $6 \pm 2$ & $13 \pm 2$ & $8 \pm 4$ & $13 \pm 2$ & $\mathbf{30} \pm 8$ & $2 \pm 1$ & $3 \pm 2$ ({\color{mygreen}+1}) \\
\texttt{scene play (5 tasks)} & $4 \pm 1$ & $15 \pm 8$ & $6 \pm 1$ & $8 \pm 6$ & $\mathbf{19} \pm 6$ & $6 \pm 4$ & $8 \pm 2$ ({\color{mygreen}+2}) \\
\midrule
\texttt{visual cube single play (5 tasks)} & - & $11 \pm 4$ & $13 \pm 3$ & - & $8 \pm 1$ & $12 \pm 3$ & $\mathbf{14} \pm 3$ ({\color{mygreen}+2}) \\
\texttt{visual scene play (5 tasks)} & - & $3 \pm 1$ & $11 \pm 3$ & - & $4 \pm 1$ & $13 \pm 2$ & $\mathbf{16} \pm 4$ ({\color{mygreen}+3}) \\
\bottomrule
\end{tabular}
}
\end{center}
\vspace{-2em}
\end{table*}

We now compare one-step FB to prior unsupervised RL algorithms, measuring the zero-shot adaptation performance on downstream tasks. While our previous sections have shown that one-step FB enjoys strong convergence properties, one might wonder whether it forgoes some degree of performance by only performing one step of policy improvement. Our experiments will show that, empirically, this is not the case. We defer the rationales for selecting prior methods to Appendix~\ref{appendix:baselines}. Appendix~\ref{appendix:envs-and-datasets} includes the detailed discussions about environments, datasets, and evaluation protocols. Prior work~\citep{park2024foundation, bagatella2025td, agarwal2025proto} that evaluated on the same benchmarks reported inconsistent results. To make a fair comparison, we implement both our and prior methods from scratch and use the same hyperparameters whenever possible (See Appendix~\ref{appendix:implementations-and-hyperparameters}).

We report results in Table~\ref{tab:offline-zero-shot-rl-eval-agg}, aggregating over $4$ tasks in each domain of ExORL and $5$ tasks in each domain of OGBench, and present the full results in Table~\ref{tab:offline-zero-shot-rl-eval}. These results show that one-step FB matches or surpasses prior unsupervised RL methods on $5$ out of $10$ domains. In particular, one-step FB achieves $+1.4 \times$ improvement over FB on average. On ExORL benchmarks, while prior methods ICVF and HILP are stronger on \texttt{walker} domain than both FB and one-step FB, our method performs on par or better than the best-performing baseline in other domains ($+17\%$ on average). On goal-conditioned domains from OGBench, while one-step FB is \emph{not} the best-performing method on state-based tasks, it outperforms prior methods by $15\%$ when taking in pixels as inputs directly. We conjecture that the state-based OGBench domains are challenging for one-step FB because the sparse reward function (goal-conditioned indicator rewards) induces a single backward representation. In contrast, some better-performing baselines are explicitly learning a goal-conditioned distance function, e.g., HILP and ICVF. Taken together, one-step FB is a competitive unsupervised pre-training algorithm for RL. In Appendix~\ref{appendix:fine-tuning-experiments}, we further study whether one-step FB provides an efficient policy initialization for online fine-tuning.

\textbf{Additional experiments.} In Appendix~\ref{appendix:dataset-effects}, we study the effects of dataset quality on our algorithm. In Appendix~\ref{appendix:confounding}, we investigate the confounding effects in our didactic experiments. Appendix~\ref{appendix:ablation} studies key components of one-step FB: the behavioral-cloning regularization coefficient $\lambda_{\text{BC}}$, the orthonormalization regularization coefficient $\lambda_{\text{ortho}}$, the reward weighting temperature $\tau_{\text{reward}}$, and the representation dimension $d$.

\vspace{-0.5em}
\section{Conclusion and Limitations}
\label{sec:conclusions}
\vspace{-0.5em}

How much computation can the RL algorithm prefetch? The FB framework offers one compelling framework for studying this question, and this paper offers some theoretical considerations on what is required for an unsupervised pre-training method for RL, which led to a simpler method with stronger convergence. Our goal is to help the community interpret, use, and build upon FB-style methods, working towards a future of universal pre-training for RL.

\textbf{Limitations.} While we explain why classical fixed-point analysis fails for FB, we do not provide a full alternative convergence theory for the practical FB algorithm. As mentioned at the end of Sec.~\ref{subsec:fb-convergence}, answering this question might require other tools in functional analysis. Practically, the zero-shot policies inferred by one-step FB might be sub-optimal, similar to prior work~\citep{park2024foundation, zheng2025intention, ghosh2023reinforcement}, if obtaining the optimal Q-value requires multiple steps of policy improvement.

\section*{Acknowledgments}

We thank Seohong Park for providing helpful feedback on drafts and the project website of this work. This work used the Della cluster provided by Princeton Research Computing, as well as the Ionic and Neuronic clusters maintained by the Department of Computer Science at Princeton University. This material is based upon work supported by the National Science Foundation under Award No. 2441665. Any opinions, findings, conclusions, or recommendations expressed in this material are those of the authors and do not necessarily reflect the views of the National Science Foundation. Some figures in this work use Twemoji, an open-source emoji set created by Twitter and licensed under CC
BY 4.0.

\clearpage

\begin{table}[H]
\caption{\footnotesize \textbf{Zero-shot evaluation on ExORL and OGBench benchmarks.} We present the full zero-shot evaluation results on $16$ ExORL tasks and $30$ OGBench tasks. In each domain, we pre-train different methods and evaluate zero-shot performance on a set of tasks. We aggregate the results over 8 seeds (4 seeds for image-based tasks) and bold values within $95\%$ of the best performance for each task.}
\label{tab:offline-zero-shot-rl-eval}
\begin{center}
\setlength{\tabcolsep}{4pt}
\resizebox{\textwidth}{!}{%
\begin{tabular}{llccccccc}
\toprule
Domain & Task & Laplacian & BYOL-$\gamma$ & TD-JEPA & ICVF & HILP & FB & One-Step FB \\
\midrule
\multirow{5}{*}{\texttt{walker}} 
& \textbf{\texttt{overall}} & $228 \pm 2$ & $227 \pm 2$ & $316 \pm 12$ & $\mathbf{619} \pm 23$ & $393 \pm 108$ & $400 \pm 40$ & $379 \pm 26$ ({\color{myred}-21}) \\
& \texttt{flip} & $243 \pm 5$ & $242 \pm 4$ & $381 \pm 34$ & $\mathbf{538} \pm 19$ & $332 \pm 135$ & $277 \pm 100$ & $388 \pm 30$ \\
& \texttt{run} & $89 \pm 1$ & $89 \pm 2$ & $112 \pm 22$ & $\mathbf{258} \pm 20$ & $136 \pm 36$ & $194 \pm 20$ & $198 \pm 30$ \\
& \texttt{stand} & $389 \pm 3$ & $387 \pm 5$ & $387 \pm 11$ & $\mathbf{858} \pm 31$ & $691 \pm 126$ & $621 \pm 93$ & $524 \pm 60$ \\
& \texttt{walk} & $192 \pm 4$ & $192 \pm 4$ & $384 \pm 47$ & $\mathbf{821} \pm 40$ & $413 \pm 195$ & $506 \pm 34$ & $407 \pm 63$ \\
\midrule
\multirow{5}{*}{\texttt{cheetah}} 
& \textbf{\texttt{overall}} & $125 \pm 41$ & $127 \pm 39$ & $185 \pm 7$ & $187 \pm 13$ & $116 \pm 78$ & $271 \pm 46$ & $\mathbf{378} \pm 56$ ({\color{mygreen}+107}) \\
& \texttt{run} & $40 \pm 13$ & $42 \pm 13$ & $66 \pm 3$ & $89 \pm 7$ & $38 \pm 32$ & $43 \pm 36$ & $\mathbf{97} \pm 19$ \\
& \texttt{run backward} & $50 \pm 17$ & $48 \pm 15$ & $69 \pm 6$ & $48 \pm 9$ & $36 \pm 40$ & $125 \pm 30$ & $\mathbf{181} \pm 60$ \\
& \texttt{walk} & $185 \pm 61$ & $199 \pm 61$ & $293 \pm 15$ & $385 \pm 21$ & $195 \pm 121$ & $251 \pm 166$ & $\mathbf{424} \pm 44$ \\
& \texttt{walk backward} & $226 \pm 77$ & $220 \pm 70$ & $312 \pm 26$ & $226 \pm 38$ & $194 \pm 202$ & $663 \pm 143$ & $\mathbf{811} \pm 156$ \\
\midrule
\multirow{5}{*}{\texttt{quadruped}} 
& \textbf{\texttt{overall}} & $462 \pm 35$ & $496 \pm 35$ & $461 \pm 115$ & $546 \pm 37$ & $352 \pm 59$ & $246 \pm 31$ & $\mathbf{645} \pm 15$ ({\color{mygreen}+399}) \\
& \texttt{jump} & $554 \pm 54$ & $603 \pm 67$ & $494 \pm 175$ & $617 \pm 59$ & $321 \pm 63$ & $247 \pm 84$ & $\mathbf{707} \pm 48$ \\
& \texttt{run} & $324 \pm 22$ & $345 \pm 19$ & $363 \pm 112$ & $395 \pm 33$ & $277 \pm 56$ & $165 \pm 51$ & $\mathbf{461} \pm 11$ \\
& \texttt{stand} & $651 \pm 47$ & $700 \pm 41$ & $785 \pm 171$ & $796 \pm 57$ & $473 \pm 103$ & $388 \pm 86$ & $\mathbf{916} \pm 30$ \\
& \texttt{walk} & $318 \pm 21$ & $337 \pm 13$ & $201 \pm 100$ & $375 \pm 57$ & $339 \pm 111$ & $183 \pm 59$ & $\mathbf{496} \pm 27$ \\
\midrule
\multirow{5}{*}{\texttt{jaco}} 
& \textbf{\texttt{overall}} & $3 \pm 1$ & $3 \pm 0$ & $13 \pm 4$ & $\mathbf{23} \pm 3$ & $20 \pm 5$ & $10 \pm 4$ & $\mathbf{22} \pm 4$ ({\color{mygreen}+12}) \\
& \texttt{reach bottom left} & $3 \pm 1$ & $3 \pm 1$ & $10 \pm 3$ & $25 \pm 9$ & $\mathbf{29} \pm 9$ & $8 \pm 5$ & $6 \pm 2$ \\
& \texttt{reach bottom right} & $3 \pm 0$ & $3 \pm 1$ & $10 \pm 5$ & $21 \pm 8$ & $\mathbf{24} \pm 8$ & $7 \pm 9$ & $5 \pm 2$ \\
& \texttt{reach top left} & $2 \pm 1$ & $3 \pm 0$ & $17 \pm 6$ & $26 \pm 10$ & $6 \pm 8$ & $9 \pm 9$ & $\mathbf{53} \pm 11$ \\
& \texttt{reach top right} & $4 \pm 1$ & $3 \pm 1$ & $16 \pm 7$ & $20 \pm 7$ & $\mathbf{22} \pm 11$ & $17 \pm 9$ & $\mathbf{22} \pm 6$ \\
\midrule
\multirow{6}{*}{\texttt{antmaze large navigate}} 
& \textbf{\texttt{overall}} & $9 \pm 1$ & $21 \pm 2$ & $14 \pm 9$ & $23 \pm 3$ & $\mathbf{34} \pm 2$ & $25 \pm 5$ & $30 \pm 9$ ({\color{mygreen}+5}) \\
& \texttt{task 1} & $2 \pm 1$ & $6 \pm 3$ & $15 \pm 11$ & $4 \pm 2$ & $13 \pm 8$ & $\mathbf{46} \pm 9$ & $21 \pm 9$ \\
& \texttt{task 2} & $2 \pm 1$ & $11 \pm 4$ & $1 \pm 2$ & $9 \pm 3$ & $16 \pm 6$ & $2 \pm 3$ & $\mathbf{41} \pm 12$ \\
& \texttt{task 3} & $29 \pm 3$ & $57 \pm 8$ & $35 \pm 23$ & $67 \pm 8$ & $\mathbf{75} \pm 6$ & $31 \pm 10$ & $15 \pm 4$ \\
& \texttt{task 4} & $6 \pm 2$ & $14 \pm 5$ & $7 \pm 7$ & $18 \pm 6$ & $27 \pm 10$ & $3 \pm 2$ & $\mathbf{33} \pm 15$ \\
& \texttt{task 5} & $6 \pm 3$ & $16 \pm 4$ & $11 \pm 8$ & $18 \pm 4$ & $40 \pm 8$ & $\mathbf{44} \pm 19$ & $37 \pm 20$ \\
\midrule
\multirow{6}{*}{\texttt{antmaze teleport navigate}} 
& \textbf{\texttt{overall}} & $3 \pm 1$ & $16 \pm 4$ & $4 \pm 1$ & $\mathbf{29} \pm 3$ & $19 \pm 6$ & $16 \pm 8$ & $11 \pm 6$ ({\color{myred}-5}) \\
& \texttt{task 1} & $1 \pm 1$ & $5 \pm 3$ & $0 \pm 0$ & $\mathbf{17} \pm 9$ & $10 \pm 5$ & $8 \pm 9$ & $2 \pm 1$ \\
& \texttt{task 2} & $6 \pm 2$ & $15 \pm 7$ & $2 \pm 1$ & $\mathbf{38} \pm 17$ & $27 \pm 7$ & $11 \pm 10$ & $15 \pm 10$ \\
& \texttt{task 3} & $4 \pm 2$ & $15 \pm 3$ & $7 \pm 3$ & $\mathbf{40} \pm 4$ & $24 \pm 8$ & $23 \pm 10$ & $17 \pm 10$ \\
& \texttt{task 4} & $2 \pm 1$ & $24 \pm 12$ & $7 \pm 2$ & $\mathbf{40} \pm 6$ & $21 \pm 11$ & $25 \pm 10$ & $19 \pm 11$ \\
& \texttt{task 5} & $2 \pm 1$ & $\mathbf{23} \pm 5$ & $2 \pm 2$ & $12 \pm 11$ & $14 \pm 7$ & $13 \pm 12$ & $2 \pm 1$ \\
\midrule
\multirow{6}{*}{\texttt{cube single play}} 
& \textbf{\texttt{overall}} & $6 \pm 2$ & $13 \pm 2$ & $8 \pm 4$ & $13 \pm 2$ & $\mathbf{30} \pm 8$ & $2 \pm 1$ & $3 \pm 2$ ({\color{mygreen}+1}) \\
& \texttt{task 1} & $5 \pm 2$ & $12 \pm 4$ & $6 \pm 4$ & $13 \pm 2$ & $\mathbf{27} \pm 7$ & $1 \pm 1$ & $3 \pm 4$ \\
& \texttt{task 2} & $5 \pm 2$ & $13 \pm 3$ & $8 \pm 6$ & $13 \pm 3$ & $\mathbf{30} \pm 10$ & $3 \pm 2$ & $4 \pm 4$ \\
& \texttt{task 3} & $7 \pm 3$ & $13 \pm 4$ & $14 \pm 9$ & $14 \pm 4$ & $\mathbf{23} \pm 13$ & $3 \pm 3$ & $4 \pm 4$ \\
& \texttt{task 4} & $6 \pm 2$ & $15 \pm 6$ & $6 \pm 4$ & $11 \pm 3$ & $\mathbf{37} \pm 19$ & $2 \pm 2$ & $2 \pm 2$ \\
& \texttt{task 5} & $5 \pm 3$ & $11 \pm 3$ & $5 \pm 4$ & $13 \pm 4$ & $\mathbf{30} \pm 22$ & $1 \pm 2$ & $1 \pm 1$ \\
\midrule
\multirow{6}{*}{\texttt{scene play}} 
& \textbf{\texttt{overall}} & $4 \pm 1$ & $15 \pm 8$ & $6 \pm 1$ & $8 \pm 6$ & $\mathbf{19} \pm 6$ & $6 \pm 4$ & $8 \pm 2$ ({\color{mygreen}+2}) \\
& \texttt{task 1} & $17 \pm 6$ & $49 \pm 32$ & $30 \pm 6$ & $34 \pm 23$ & $\mathbf{66} \pm 16$ & $25 \pm 16$ & $21 \pm 8$ \\
& \texttt{task 2} & $1 \pm 1$ & $11 \pm 8$ & $0 \pm 0$ & $5 \pm 7$ & $\mathbf{14} \pm 11$ & $5 \pm 5$ & $12 \pm 4$ \\
& \texttt{task 3} & $1 \pm 1$ & $9 \pm 6$ & $0 \pm 0$ & $3 \pm 3$ & $\mathbf{12} \pm 14$ & $0 \pm 1$ & $0 \pm 0$ \\
& \texttt{task 4} & $2 \pm 1$ & $4 \pm 7$ & $0 \pm 0$ & $0 \pm 0$ & $1 \pm 1$ & $0 \pm 0$ & $\mathbf{7} \pm 4$ \\
& \texttt{task 5} & $\mathbf{0} \pm \mathbf{0}$ & $\mathbf{0} \pm \mathbf{0}$ & $\mathbf{0} \pm \mathbf{0}$ & $\mathbf{0} \pm \mathbf{0}$ & $\mathbf{0} \pm \mathbf{0}$ & $\mathbf{0} \pm \mathbf{0}$ & $\mathbf{0} \pm \mathbf{0}$ \\
\midrule
\multirow{6}{*}{\texttt{visual cube single play}} 
& \textbf{\texttt{overall}} & - & $11 \pm 4$ & $13 \pm 3$ & - & $8 \pm 1$ & $12 \pm 3$ & $\mathbf{14} \pm 3$ ({\color{mygreen}+2}) \\
& \texttt{task 1} & - & $24 \pm 16$ & $19 \pm 7$ & - & $10 \pm 6$ & $14 \pm 7$ & $\mathbf{47} \pm 12$ \\
& \texttt{task 2} & - & $10 \pm 2$ & $13 \pm 8$ & - & $\mathbf{19} \pm 5$ & $10 \pm 5$ & $11 \pm 3$ \\
& \texttt{task 3} & - & $\mathbf{16} \pm 4$ & $\mathbf{16} \pm 7$ & - & $10 \pm 7$ & $15 \pm 6$ & $3 \pm 3$ \\
& \texttt{task 4} & - & $3 \pm 3$ & $6 \pm 5$ & - & $\mathbf{13} \pm 7$ & $8 \pm 5$ & $9 \pm 8$ \\
& \texttt{task 5} & - & $1 \pm 1$ & $10 \pm 5$ & - & $\mathbf{15} \pm 3$ & $11 \pm 7$ & $2 \pm 2$ \\
\midrule
\multirow{6}{*}{\texttt{visual scene play}} 
& \textbf{\texttt{overall}} & - & $3 \pm 1$ & $11 \pm 3$ & - & $4 \pm 1$ & $13 \pm 2$ & $\mathbf{16} \pm 4$ ({\color{mygreen}+3}) \\
& \texttt{task 1} & - & $7 \pm 2$ & $53 \pm 16$ & - & $14 \pm 4$ & $66 \pm 12$ & $\mathbf{78} \pm 16$ \\
& \texttt{task 2} & - & $1 \pm 1$ & $0 \pm 1$ & - & $2 \pm 1$ & $0 \pm 0$ & $\mathbf{3} \pm 4$ \\
& \texttt{task 3} & - & $0 \pm 1$ & $0 \pm 0$ & - & $\mathbf{2} \pm 1$ & $0 \pm 0$ & $0 \pm 0$ \\
& \texttt{task 4} & - & $\mathbf{5} \pm 5$ & $0 \pm 0$ & - & $0 \pm 0$ & $0 \pm 0$ & $1 \pm 1$ \\
& \texttt{task 5} & - & $\mathbf{0} \pm 1$ & $\mathbf{0} \pm \mathbf{0}$ & - & $\mathbf{0} \pm 1$ & $\mathbf{0} \pm \mathbf{0}$ & $\mathbf{0} \pm \mathbf{0}$ \\
\bottomrule
\end{tabular}%
}
\end{center}
\end{table}

\appendix

\section{Related Work}
\label{appendix:related-work}

Our work investigates the theoretical foundations of unsupervised pre-training in RL, with a focus on the prior forward-backward (FB) representation learning algorithm~\citep{touati2021learning}. 

\paragraph{Unsupervised RL and zero-shot RL.} The broader goal of unsupervised RL is to pre-train policies from reward-free \emph{unsupervised interactions} that enable efficient adaptation to downstream tasks. Prior work has approached this via skill learning~\citep{frans2024unsupervised, park2024foundation, kim2024unsupervised, hu2023unsupervised, park2023metra, eysenbach2018diversity, zheng2024can}, intent predictions~\citep{frans2024unsupervised, zheng2025intention}, empowerment maximization~\citep{klyubin2005empowerment, choi2021variational, mohamed2015variational, shah2025structured}, or self-supervised representation learning~\citep{ni2024bridging, ma2023vip, mazoure2023contrastive, zhang2021learning, zheng2023stabilizing}. After pre-training a set of policies, these methods typically adapt one of the policies to a new reward function via continuous fine-tuning or hierarchical control~\citep{park2023metra, frans2024unsupervised, ma2023vip, zheng2025intention}. One appealing family of methods that does not require gradient-based fine-tuning of the pre-trained policy is called zero-shot RL methods~\citep{touati2021learning, park2024foundation, bagatella2025td, tirinzoni2025zeroshot, li2025bfm}. Similar to the in-context learning in LLMs~\citep{brown2020language}, zero-shot RL methods prefetch optimal policies for \emph{any} rewards during pre-training and perform in-context adaptation on downstream tasks. In this paper, through theoretical and empirical analysis, we demystify a prior SOTA zero-shot RL method (FB~\citep{touati2021learning}) and study its convergence in practice. 

\paragraph{Successor measures and successor features.} Our work builds on successor measures~\citep{dayan1993improving}, which were originally proposed to improve generalization in RL and have since been widely adopted in neuroscience to model predictive maps in the brain~\citep{Momennejad2017successor, xie2025lowrank}. In the domain of Deep RL, prior work has shown that successor measures can be learned in high-dimensional environments~\citep{kulkarni2016deep, zhang2017deep} and facilitate transfer learning across tasks~\citep{barreto2017successor}. By combining these ideas with universal value function approximators~\citep{schaul2015universal}, Universal Successor Features (USFs) generalize successor features to estimate values for any reward under any policy~\citep{borsa2019universal}. More recently, forward-backward (FB) representation learning~\citep{touati2021learning} extended this to enable zero-shot evaluation for \emph{any} reward function, forming the basis for building behavioral foundation models~\citep{tirinzoni2025zeroshot, bagatella2025td}. Our analysis of FB interprets the representation objective as estimating the successor measure of a latent-conditioned policy. However, this estimation incurs a circular dependency.

\paragraph{Density ratio estimations.} Directly estimating the ratios between two probability density functions is an important problem in machine learning. Solving this problem enables applications in two-sample testing~\citep{lopez2016revisiting}, covariate shift adaptation~\citep{sugiyama2008direct}, outlier detection~\citep{sugiyama2008direct}, mutual information estimation~\citep{belghazi2018mutual, poole2019variational}, and policy evaluation~\citep{nachum2019dualdice}. Prior work has tackled the density ratio estimation problem by minimizing a KL divergence~\citep{sugiyama2008direct}, moment matching~\citep{gretton2009covariate}, penalized convex risk minimization~\citep{nguyen2010estimating}, contrastive learning~\citep{ma2018noise, oord2018representation, poole2019variational}. Our analysis of the FB algorithm is closely related to the least-squares importance fitting approach~\citep{kanamori2008efficient, kanamori2009least} for density ratio estimation. In this paper, we show that the FB representation objective is a temporal-difference variant of the least-squares importance fitting loss in Sec.~\ref{subsec:fb-repr-obj} and is closely related to fitted Q-evaluation (FQE)~\cite{riedmiller2005neural}.

\paragraph{One-step RL.} One-step RL methods~\cite{brandfonbrener2021offline, gulcehre2020rl, peters2007reinforcement, wang2018exponentially, peters2010reps} apply one step of policy improvement to a data-generating behavioral policy. These methods have two phases: First, estimate the Q-values of the behavioral policy via regression or FQE updates. Second, optimize the policy to maximize the predicted Q-value. This formulation decouples Q-value estimation from policy extraction, encompassing a wide range of techniques, from Relative Entropy Policy Search~\cite{peters2010reps} to goal-conditioned imitation learning~\cite{savinov2018semi, ding2019, li2020generalized, chen2021decision, kumar2019rewardconditionedpolicies, eysenbach2020rewriting, eysenbach2022imitating, wantlin2025consistent}. Theoretical and empirical analysis presented in~\citet{eysenbach2022a} show that one step of policy improvement is equivalent to multi-step critic regularization, opening the door to developing a simpler suite of algorithms. Recent work~\cite{eysenbach2022contrastive, park2025transitive, brandfonbrener2021offline} has applied the idea of one-step policy improvement to develop practical RL algorithms. Similarly, our proposed method, one-step FB, adopts the principle of one-step policy improvement, breaking the circular dependency in the original FB algorithm.

\section{Preliminary}
\label{appendix:preliminary}

\subsection{The Successor Measure Matrix}
\label{appendix:successor-measure-matrix}

For policy $\pi$, we define the policy-dependent transition matrix $P^{\pi} \in \R^{\abs{\gS \times \gA} \times \abs{\gS \times \gA}}$ as $P^{\pi}_{(s, a), (s', a')} = p(s' \mid s, a) \pi(a' \mid s')$. 
\begin{lemma}[Lemma 1.6 and Corollary 1.5 of~\citet{agarwal2019reinforcement}]
\label{lemma:successor-measure-expr-and-rank}
For any policy $\pi: \gS \to \Delta(\gA)$, transition function $p(s' \mid s, a)$ and discount $\gamma \in [0, 1)$, the successor measure matrix $M^{\pi} \in \R^{\abs{\gS \times \gA} \times \abs{\gS \times \gA}}$ satisfies $M^{\pi} = (1 - \gamma) (I_{\abs{\gS \times \gA}} - \gamma P^{\pi})^{-1}$. Furthermore, $M^{\pi}$ is full rank, i.e., $\text{rank}(M^{\pi}) = \abs{\gS \times \gA}$.
\begin{proof}
    \label{proof:successor-measure-expr-and-rank}
    This result is almost an immediate consequence of the Bellman equation in Eq.~\ref{eq:successor-measure-bellman-eq}:
    \begin{align*}
        M^{\pi} = (1 - \gamma) I_{\abs{\gS \times \gA}} + \gamma P^{\pi} M^{\pi} \implies (I_{\abs{\gS \times \gA}} - \gamma P^{\pi}) M^{\pi} = (1 - \gamma) I_{\abs{\gS \times \gA}}.
    \end{align*}
    If $I_{\abs{\gS \times \gA}} - \gamma P^{\pi}$ is invertible (i.e., full rank), then successor measure must satisfy $M^{\pi} = (1 - \gamma) (I_{\abs{\gS \times \gA}} - \gamma P^{\pi})^{-1}$. Thus, the proof boils down to showing that this matrix is invertible.

    We prove that the matrix $I_{\abs{\gS \times \gA}} - \gamma P^{\pi}$ is invertible by showing its null space only contains the zero vector. For any non-zero vector $x \in \R^{\abs{\gS \times \gA}}$, the $L^{\infty}$-norm of $(I_{\abs{\gS \times \gA}} - \gamma P^{\pi}) x$ satisfies
    \begin{align*}
        \norm{(I_{\abs{\gS \times \gA}} - \gamma P^{\pi}) x}_{\infty} &=  \norm{x - \gamma P^{\pi} x}_{\infty} \\
        &\overset{(a)}{\geq} \norm{x}_{\infty} - \gamma \norm{P^{\pi} x}_{\infty} \\
        &\overset{(b)}{\geq} \norm{x}_{\infty} - \gamma \norm{x}_{\infty} \\
        &= (1- \gamma) \norm{x}_{\infty} \\
        &\overset{(c)}{>} 0, 
    \end{align*}
    where, in \emph{(a)}, we apply the triangle inequality of the $L^{\infty}$-norm, \emph{(b)} holds because $P^{\pi} x$ is an expectation over elements of $x$ and $\norm{x}_{\infty} = \max \{ \abs{x_1}, \abs{x_2}, \cdots, \abs{x_{\abs{\gS \times \gA}}} \}$, and, in \emph{(c)}, we apply the conditions that $\gamma < 1$ and $x$ is a non-zero vector. Therefore, the null space of the matrix $I_{\abs{\gS \times \gA}} - \gamma P^{\pi}$ only contains the zero vector, implying it is invertible. Thus, we can compute the successor measure by matrix inversion:
    \begin{align*}
        M^{\pi} = (1 - \gamma) (I_{\abs{\gS \times \gA}} - \gamma P^{\pi})^{-1}.
    \end{align*}
    Since the matrix $I_{\abs{\gS \times \gA}} - \gamma P^{\pi}$ is invertible, we conclude that the successor measure is also a full-rank invertible matrix with $\text{rank}(M^{\pi}) = \abs{\gS \times \gA}$.
\end{proof}
\end{lemma}

\subsection{Components for Unsupervised Pre-Training in RL}
\label{appendix:url-components}

For \emph{Step 1}, prior methods usually define a latent variable $z \in \gZ$ sampled from a prior distribution $p_\gZ \in \Delta(\gZ)$, e.g., a standard Gaussian distribution~\citep{frans2024unsupervised} or a scaled von Mises-Fisher distribution~\citep{park2024foundation, touati2021learning}, and use it to index latent-conditioned policies $\pi: \gS \times \gZ \to \Delta(\gA)$. The goal of unsupervised pre-training is to prefetch the reward-maximizing policies for downstream tasks~\citep{park2024foundation, touati2021learning, agarwal2025proto, bagatella2025td}. For \emph{Step 2}, we are presented with a reward function $r(\ph{s}, \ph{a})$ and asked to find a latent variable $z_r$ so that policy $\pi(\ph{a} \mid \ph{s}, z_r)$ achieves high reward. %

\subsection{Definition of Ground-Truth Forward-Backward Representations}
\label{appendix:def-gt-fb-reprs}

We formally define the ground-truth forward-backward representations $F^{\star}: \gS \times \gA \times \gZ \to \gZ$ and $B^{\star}: \gS \times \gA \to \gZ$ as a pair of functions satisfying the following properties.

\begin{definition}[Definition 1 of~\citet{touati2021learning}]
\label{def:fb-definition}
For any CMP with latent space $\gZ$ and any marginal probability measure $\rho \in \Delta(\gS \times \gA)$, we say that a pair of functions $F^{\star}:\gS\times\gA\times\gZ\to\gZ$ and $B^{\star}:\gS\times\gA\to\gZ$ are the ground-truth forward-backward representations if, for any latent variable $z$, any current state-action pair $(s, a)$, and any future state-action pair $(s_f,a_f)$, the latent-conditioned policy $\pi: \gS \times \gZ \to \Delta(\gA)$ defined by
\begin{align}
\pi(a \mid s, z) &= \delta \big( a \, \big\lvert \, \argmax_{a \in \gA} F^{\star}(s, a, z)^{\top} z \big), \label{eq:fb-policy}
\end{align}
has its associated successor measure ratio $M^{\pi}/ \rho: \gS \times \gA \times \gZ \times \gS \times \gA \to \mathbb{R}_{\ge 0}$ satisfy
\begin{align}
    \frac{M^{\pi}(s_f, a_f \mid s, a, z)}{\rho(s_f, a_f)} &= F^{\star}(s, a, z)^{\top} B^{\star}(s_f, a_f).
    \label{eq:fb-sm-ratio}
\end{align}
\end{definition}

\begin{definition}[heorem 2 of~\citet{touati2021learning}]
\label{def:fb-adaptation}
Augmenting the CMP with a reward function $r: \gS \times \gA \to \R$, for any latent variable $z$, and any current state-action pair $(s, a)$, the ground-truth FB representations $F^{\star}: \gS \times \gA \times \gZ \to \gZ$ and $B^{\star}: \gS \times \gA \to \gZ$ produces a latent variable
\begin{align}
    z_r = \E_{(s_f, a_f) \sim \rho(s_f, a_f)} \left[ B^{\star}(s_f, a_f) r(s_f, a_f) \right]
    \label{eq:fb-latent-of-reward}
\end{align}
that indexes the optimal policy $\pi^{\star}_r (a \mid s) = \pi(a \mid s, z_r)$ and the optimal Q-value $Q^{\star}_r(s, a) = F^{\star}(s, a, z_r)^{\top} z_r$.
\end{definition}

The optimal policy adaptation for any reward function depends on the closeness of the latent space $\gZ$. In practice, prior work~\citep{touati2021learning, touati2022does, bagatella2025td} usually sets the latent space to be the entire $d$-dimensional real space $\gZ = \R^d$. One intriguing property of this design decision is the linear closeness of $\R^d$: $\R^d$ is a vector space that is closed under vector addition and scalar multiplication.

\begin{lemma}[Closeness of the $d$-dimensional real space]
\label{lemma:closeness-of-real-space}
    For any vectors $x, y \in \R^d$ and any scalars $a, b \in \R$, the linear combination $ax + by \in \R^d$.
\end{lemma}

Importantly, we can apply Lemma~\ref{lemma:closeness-of-real-space} to the (infinite number of) backward representations $B(\ph{s_f}, \ph{a_f})$ and extend the results to an expectation (integral):

\begin{corollary}[The latent space covers the optimal latent variable for any reward function] For a latent space $\gZ = \R^d$, the ground-truth backward representations $B^{\star}: \gS \times \gA \to \gZ$, any reward function $r: \gS \times \gA \to \R$, and any marginal probability measure $\rho \in \Delta(\gS \times \gA)$, the optimal latent variable
\begin{align*}
    z_r = \E_{(s_f, a_f) \sim \rho(s_f, a_f)} \left[ B^{\star}(s_f, a_f) r(s_f, a_f) \right]
\end{align*}
is always covered by the latent space: $z_r \in \gZ$.
\end{corollary}

This corollary enables the FB algorithm to first pre-train a set of latent-conditioned policies and then use the optimal latent variable to index the optimal policy for any reward function.

\section{Theoretical Analysis}
\label{appendix:theoretical-analysis}

\subsection{Existence of the Ground-Truth FB Representations}
\label{appendix:fb-repr-existence}

Before proving the existence of FB representations, we define some special notations for the latent space and use matrices to simplify our derivations. Specifically, we will consider the latent space $\gZ = \R^d$ as both a latent manifold and a set of latent variables containing every vector in $\R^d$: $\gZ = \{z_1, \cdots, z_{\abs{\gZ}} \}$. Although the size of the latent space is infinite, we will consider $\abs{\gZ}$ as a finite number and take the limit to infinity ($\abs{\gZ} \to \infty$). We will only use this notation to simplify our theoretical analysis.

For any forward representation function $F: \gS \times \gA \times \gZ \to \gZ$ and any backward representation function $B: \gS \times \gA \to \gZ$, we stack the backward representations $B(\ph{s_f}, \ph{a_f})$ into a matrix $B \in \R^{d \times \abs{\gS \times \gA}}$:
\begin{align}
    B = \begin{bmatrix}
        B(s_1, a_1) & \cdots & B(s_{\abs{\gS}}, a_1) & \cdots & B(s_{\abs{\gS}}, a_{\abs{\gA}})
    \end{bmatrix}.
    \label{eq:backward-repr-matrix}
\end{align}
The forward representations induce the latent-conditioned policy $\pi: \gS \times \gZ \to \Delta(\gA)$ (Eq.~\ref{eq:fb-policy}). These policies maximize the inner products between the forward representation and the corresponding latent variable: a delta measure (indicator function) around the maximizer.
\begin{align}
    \pi(a \mid s, z) &= \delta \left( a \, \left\lvert \, \argmax_{a \in \gA} F(s, a, z)^{\top} z \right. \right) \nonumber \\
    &= \mathbbm{1}_{\argmax_{a \in \gA} F(s, a, z)^{\top} z} \left( a \right).
    \label{eq:latent-conditioned-policy}
\end{align}
For different latent $z_i \in \gZ$, the policy $\pi(\ph{a} \mid \ph{s}, z_i)$ induces a successor measure matrix $M^{\pi}_i \in \R^{\abs{\gS \times \gA} \times \abs{\gS \times \gA}}$ as computed in Lemma~\ref{lemma:successor-measure-expr-and-rank}: for all $i = 1, \cdots, \abs{\gZ}$,
\begin{align}
    M_i^{\pi} = (1 - \gamma) \left( I_{\abs{\gS \times \gA}} - \gamma P^{\pi(\ph{a} \mid \ph{s}, z_i)} \right)^{-1}.
    \label{eq:latent-conditioned-successor-measure-matrix}
\end{align}
We also aggregate all the $M^{\pi}_i$'s into a single matrix $M^{\pi}_{\gZ} \in \R^{\abs{\gZ \times \gS \times \gA} \times \abs{\gS \times \gA}}$:
\begin{align}
    M^{\pi}_{\gZ} = \begin{bmatrix}
        M^{\pi}_1 \\ \vdots \\ M^{\pi}_{\abs{\gZ}}
    \end{bmatrix}.
    \label{eq:successor-measure-matrix}
\end{align}
Similarly, for each $z_i$, we stack the latent-conditioned forward representations $F(\ph{s}, \ph{a}, z_i)$ into a matrix $F_i \in \R^{\abs{\gS \times \gA} \times d}$: for $i = 1, \cdots, \abs{\gZ}$,
\begin{align}
    F_i = \begin{bmatrix}
        F(s_1, a_1, z_i)^{\top} \\
        \vdots \\
        F(s_1, a_{\abs{\gA}}, z_i)^{\top} \\
        \vdots \\
        F(s_{\abs{\gS}}, a_{\abs{\gA}}, z_i)^{\top} \\ 
    \end{bmatrix},
    \label{eq:latent-conditioned-forward-repr-matrix}
\end{align}
and also aggregate all the $F_i$'s into a single matrix $F^{\pi}_{\gZ} \in \R^{\abs{\gZ \times \gS \times \gA} \times d}$:
\begin{align}
    F_{\gZ} = \begin{bmatrix}
        F(s_1, a_1, z_1)^{\top} \\
        \vdots \\
        F(s_1, a_{\abs{\gA}}, z_1)^{\top} \\
        \vdots \\
        F(s_{\abs{\gS}}, a_{\abs{\gA}}, z_1)^{\top} \\ 
        \vdots \\
        F(s_{\abs{\gS}}, a_{\abs{\gA}}, z_{\abs{\gZ}})^{\top} \\ 
    \end{bmatrix} = \begin{bmatrix} 
        F_1 \\
        \vdots \\
        F_{\abs{\gZ}}
    \end{bmatrix}.
    \label{eq:forward-repr-matrix}
\end{align}
Finally, given a marginal probability measure (probability mass) $\rho \in \Delta(\gS \times \gA)$ with full support on $\gS \times \gA$ and a reward function $r: \gS \times \gA \to \R$, with slight abuse of notation, we stack them in a marginal measure vector $\rho \in \R^{\abs{\gS \times \gA}}$ and a reward vector $r \in \R^{\abs{\gS \times \gA}}$, respectively:
\begin{align}
    \rho = \begin{bmatrix}
        \rho \left( s_1, a_1 \right) \\ \vdots \\ \rho \left( s_1, a_{\abs{\gA}} \right) \\ \vdots \\ \rho \left( s_{\abs{\gS}}, a_{\abs{\gA}} \right)
    \end{bmatrix}, \qquad r = \begin{bmatrix}
        r\left( s_1, a_1 \right) \\ \vdots \\ r\left( s_1, a_{\abs{\gA}} \right) \\ \vdots \\ r\left( s_{\abs{\gS}}, a_{\abs{\gA}} \right)
    \end{bmatrix}.
    \label{eq:rho-vec-and-reward-vec}
\end{align}

Defining these matrices allows us to simplify the notation and denote the FB representation learning procedure using linear algebra. For example, the successor measure ratio identity in Definition~\ref{def:fb-definition} and the optimal latent adaptation in Definition~\ref{def:fb-adaptation} can be written as
\begin{align*}
    M^{\pi}_{\gZ} \text{diag}(\rho)^{-1} = F_{\gZ}^{\star} B^{\star}, \qquad z_r = B^{\star}( r \odot \rho),
\end{align*}
where $\text{diag}(\rho) \in \R^{\abs{\gS \times \gA} \times \abs{\gS \times \gA}}$ is the diagonal matrix of the marginal measure vector $\rho$ and $\odot$ denotes the element-wise multiplication. 
In addition, by the relationship between the successor measure and the Q-value (Eq.~\ref{eq:relate-sr-to-q}), we can write the Q-value for a latent-conditioned policy $\pi(\ph{a} \mid \ph{s}, z_i)$ as a vector in $\R^{\abs{\gS \times \gA}}$: $Q_i^{\pi} = M^{\pi}_i r$.

We can now constrain the rank of the forward representation matrix $\text{rank} \left( F_{\gZ} \right)$, the rank of the backward representation matrix $\text{rank}(B)$, and the rank of the product of the forward-backward representation matrices $\text{rank} \left( F_{\gZ} B \right)$ by matrix dimensions.
\begin{remark}
    \label{remark:fb-rank}
    The rank of any forward representation matrix satisfies $\text{rank}(F_\gZ) \leq \text{min} \left( \abs{\gZ \times \gS \times \gA}, d \right)$. The rank of any backward representation matrix satisfies $\text{rank}(B) \leq \text{min}\left(d, \abs{\gS \times \gA} \right)$. The rank of the product of the forward-backward representation matrices satisfies $\text{rank}\left( F_{\gZ} B \right) \leq \text{min} \left( \text{rank}(F_{\gZ}), \text{rank}(B) \right)$.
\end{remark}
Using these constraints, we formally prove the existence of FB representations.

\begin{repproposition}{prop:fb-existence}
    Given any discrete CMP, a latent space $\gZ = \{z_1, \cdots, z_{\abs{\gZ}} \}$ with each $z_i \in \R^d$, and any marginal probability measure vector $\rho \in \R^{\abs{\gS \times \gA}}$ (Eq.~\ref{eq:rho-vec-and-reward-vec}), any forward representation matrix $F^{\star}_{\gZ} \in \R^{ \abs{\gZ \times \gS \times \gA} \times d }$ (Eq.~\ref{eq:forward-repr-matrix}), which induces the latent-conditioned policy $\pi: \gS \times \gZ \to \Delta(\gA)$ (Eq.~\ref{eq:latent-conditioned-policy}) with associated successor measure matrix $M^{\pi}_{\gZ}$ (Eq.~\ref{eq:successor-measure-matrix}), and any backward representation matrix $B^{\star} \in \R^{d \times \abs{\gS \times \gA}}$ (Eq.~\ref{eq:backward-repr-matrix}) that encodes this CMP's successor measure as,
    \begin{enumerate}
        \item $F_\gZ^{\star}$ and $B^{\star}$ fit the successor measure ratio (Definition~\ref{def:fb-definition}):
        \begin{align}
            M^{\pi}_{\gZ} \text{diag}(\rho)^{-1} = F^{\star}_{\gZ} B^{\star},
            \label{eq:matrix-fb-definition}
        \end{align}
        
        \item $F_\gZ^{\star}$ and $B^{\star}$ enable optimal policy adaptation (Definition~\ref{def:fb-adaptation}): for any reward vector $r \in \R^{\abs{\gS \times \gA}}$ (Eq.~\ref{eq:rho-vec-and-reward-vec}),
        \begin{align}
            z_r = B^{\star}(r \odot \rho) \in \gZ,
            \label{eq:matrix-fb-adaptation}
        \end{align}
        indexes the optimal policy $\pi^{\star}_r(\ph{a} \mid \ph{s}) = \pi(\ph{a} \mid \ph{s}, z_r) = \argmax_{\ph{a}} F^{\star}(\ph{s}, \ph{a}, z_r)^{\top} z_r$ and the optimal Q-value $Q_r^{\star} = Q_{z_r} = F^{\star}_{z_r} z_r$,
    \end{enumerate}
    must satisfy the following properties:
    \begin{enumerate}
        \item The representation dimension $d$ is at least $\abs{\gS \times \gA}$, i.e., $d \geq \abs{\gS \times \gA}$.
        \item The rank of the forward representation matrix $F^{\star}_{\gZ}$ is at least $\abs{\gS \times \gA}$ and at most $d$, i.e., $\abs{\gS \times \gA} \leq \text{rank}\left( F_{\gZ}^{\star} \right) \leq d$.
        \item The rank of the backward representation matrix $B^{\star}$ is equivalent to $\abs{\gS \times \gA}$, i.e., $\text{rank} \left( B^{\star} \right) = \abs{\gS \times \gA}$.
        \item For different latents $z_i (i = 1, \cdots, \abs{\gZ})$, the backward representation matrix $B^{\star}$, the forward representation matrix for each latent $F_i^{\star}$ (Eq.~\ref{eq:latent-conditioned-forward-repr-matrix}), and the successor measure matrix for each latent $M^{\pi}_i$ (Eq.~\ref{eq:latent-conditioned-successor-measure-matrix}) must satisfy:  
        \begin{align*}
            B^{\star} = F^{\star +}_1 M^{\pi}_1 \text{diag}(\rho)^{-1} = F^{\star +}_2 M^{\pi}_2 \text{diag}(\rho)^{-1} = \cdots = F^{\star +}_{\abs{\gZ}} M^{\pi}_{\abs{\gZ}} \text{diag}(\rho)^{-1},
        \end{align*}
        where $\ph{X}^+$ denotes the pseudoinverse (Moore–Penrose inverse)~\citep{moore1920reciprocal, Bjerhammar1951ApplicationOC} of the matrix $\ph{X}$ and $\text{diag}(\ph{x})$ is the diagonal matrix of the vector $\ph{x}$.
    \end{enumerate}
    \begin{proof} The main idea of our proof is to match the rank of the FB representation matrices $F^{\star}_{\gZ}, B^{\star}$ to the rank of the successor measure matrix $M^{\pi}_{\gZ}$. 
    
    \paragraph{Rank matching.} Since we aim to find a forward representation matrix $F^{\star}_\gZ$ and a backward representation matrix $B^{\star}$ fitting the successor measure ratio exactly, it is necessary to match the rank on both sides of Eq.~\ref{eq:matrix-fb-definition}: 
    \begin{align*}
        \text{rank} \left( M^{\pi}_{\gZ} \right) = \text{rank} \left( M^{\pi}_{\gZ} \text{diag}(\rho)^{-1} \right) = \text{rank}(F_{\gZ}^{\star} B^{\star}).
    \end{align*}
    Applying Lemma~\ref{lemma:successor-measure-expr-and-rank} to the successor measure of each latent $z_i \in \gZ$ gives us $\text{rank}(M^{\pi}_i) = \abs{\gS \times \gA}$. Meanwhile, because the successor measure matrix $M^{\pi}_{\gZ}$ is an aggregation of each $M^{\pi}_i$ (Eq.~\ref{eq:successor-measure-matrix}), we have $\text{rank}(M^{\pi}_{\gZ}) = \abs{\gS \times \gA}$, suggesting that the rank of the product of the FB representation matrices must be equivalent to $\abs{\gS \times \gA}$:
    \begin{align*}
        \text{rank}(M^{\pi}_{\gZ}) = \text{rank} \left( F^{\star}_{\gZ} B^{\star} \right) = \abs{\gS \times \gA}.
    \end{align*}
    
     We next constrain the rank of the forward representation matrix $F^{\star}_{\gZ}$ and the rank of the backward representation matrix $B^{\star}$ by using the properties in Remark~\ref{remark:fb-rank}. Since the rank of the product of the forward-backward representation matrices is at most the rank of either representation matrix, we have
     \begin{align*}
         &\hspace{2.725em} \text{rank} \left( F^{\star}_{\gZ} B^{\star} \right) \leq \min \left( \text{rank} \left(F^{\star}_{\gZ} \right), \text{rank} \left( B^{\star} \right) \right) \\ 
         &\implies \text{rank} \left( F^{\star}_{\gZ} B^{\star} \right) \leq \text{rank} \left(F^{\star}_{\gZ} \right) \quad \text{and} \quad \text{rank} \left( F^{\star}_{\gZ} B^{\star} \right) \leq \text{rank} \left(B^{\star} \right).
    \end{align*}
    Plugging in the rank of the product of FB representation matrices $\text{rank}(F^{\star}_{\gZ} B^{\star}) = \abs{\gS \times \gA}$ and the constraints for the rank of the forward representation matrix $\text{rank}(F^{\star}_{\gZ})$ and the rank of the backward representation matrix $\text{rank}(B^{\star})$, we have
    \begin{align}
        \abs{\gS \times \gA} \leq \text{rank}(F^{\star}_{\gZ}) \leq \min \left( \abs{\gZ \times \gS \times \gA}, d \right) \quad &\text{and} \quad \abs{\gS \times \gA} \leq \text{rank}(B^{\star}) \leq \min \left( d, \abs{\gS \times \gA} \right) \nonumber \\
        \implies \abs{\gS \times \gA} \leq \text{rank}(F^{\star}_{\gZ}) \leq \min \left( \abs{\gZ \times \gS \times \gA}, d \right) \quad &\text{and} \quad \abs{\gS \times \gA} \leq d, \, \text{rank}(B^{\star}) = \abs{\gS \times \gA} \nonumber \\
        \overset{(a)}{\implies} \abs{\gS \times \gA} \leq \text{rank}(F^{\star}_{\gZ}) \leq d \quad &\text{and} \quad \abs{\gS \times \gA} \leq d, \, \text{rank}(B^{\star}) = \abs{\gS \times \gA},
        \label{eq:forward-backward-matrices-rank-constraints}
     \end{align}
     where we simplify the inequalities in \emph{(a)} when $\abs{\gZ} \to \infty$. 
     
    These results suggest that the rank constraints on the FB representation matrices: 
     \begin{itemize}
         \item The rank of the forward representation matrix is at least $\abs{\gS \times \gA}$ (not necessarily full rank).
         \item The rank of the backward representation matrix is equivalent to $\abs{\gS \times \gA}$ (\emph{full} rank).
     \end{itemize}
     In addition, the necessary condition for the existence of ground-truth FB representation matrices is that the representation dimension is at least $\abs{\gS \times \gA}$, i.e., $d \geq \abs{\gS \times \gA}$. Importantly, these are three individual conditions for the representation dimension $d$, the forward representation matrix $F^{\star}_{\gZ}$, and the backward representation matrix $B^{\star}$, respectively. However, they still fail to guarantee that the product of the FB representation matrices $F^{\star}_{\gZ} B$ will fit $M^{\pi}_{\gZ} \text{diag}(\rho)^{-1}$ exactly. We need further relationships to bridge $F^{\star}_{\gZ}$ and $B^{\star}$.

     \paragraph{Bridging the FB representation matrices.} Our key observations are twofold. \emph{First}, the backward representation matrix $B^{\star}$ does not take any latent variable as input, indicating that $B$ compresses the common information throughout the entire latent space. \emph{Second}, the rank matching not only holds for the entire $F^{\star}_{\gZ}$ and $M^{\star}_{\gZ}$ matrices, but also holds for the forward representation matrix of each latent $F_i^{\star}$ (Eq.~\ref{eq:latent-conditioned-successor-measure-matrix}) and the successor measure ratio of each latent $M^{\pi}_i$ (Eq.~\ref{eq:latent-conditioned-forward-repr-matrix}). We next discuss the meaning of these two observations.
     
     When the ground-truth FB representation matrices exist, using the block matrix notations in Eq.~\ref{eq:successor-measure-matrix} and Eq.~\ref{eq:forward-repr-matrix} to rewrite Eq.~\ref{eq:matrix-fb-definition} gives us
     \begin{align*}
         \begin{bmatrix}
            M^{\pi}_1 \\
            \vdots \\
            M^{\pi}_{\abs{\gZ}}
         \end{bmatrix} \text{diag}(\rho)^{-1} &= \begin{bmatrix}
            F^{\star}_1 \\
            \vdots \\
            F^{\star}_{\abs{\gZ}}
         \end{bmatrix} B^{\star} \\
         \implies M^{\pi}_1 \text{diag}(\rho)^{-1} = F^{\star}_1 B^{\star}, \, M^{\pi}_2 \text{diag}(\rho)^{-1} &= F^{\star}_2 B^{\star}, \, \cdots, \, M^{\pi}_{\abs{\gZ}} \text{diag}(\rho)^{-1} = F^{\star}_{\abs{\gZ}} B^{\star}.
     \end{align*}
     Furthermore, since, for $i = 1, \cdots, \abs{\gZ}$, each successor measure $M^{\pi}_i \in \R^{\abs{\gS \times \gA} \times \abs{\gS \times \gA}}$ is a square matrix with rank $\text{rank}(M^{\pi}_i) = \abs{\gS \times \gA}$, the column space of each successor measure is equivalent to the $\abs{\gS \times \gA}$-dimensional real space:
     \begin{align*}
         \text{col} \left( M^{\pi}_i \right) = \R^{\abs{\gS \times \gA}}.
     \end{align*} We note that the marginal probability measure vector $\rho$ does not change the column space of $M^{\pi}_i$ because $\rho$ has the full support over $\gS \times \gA$: $\text{col} \left( M^{\pi}_i \text{diag}(\rho)^{-1} \right) = \text{col} \left( M^{\pi}_i \right)$. Meanwhile, since the rank of the entire forward representation matrix $\text{rank}(F^{\star}_{\gZ})$ is at least $\abs{\gS \times \gA}$, we know that the rank of each $\text{rank}(F_i^{\star})$ is also at least $\abs{\gS \times \gA}$. By the shape of matrix $F_i^{\star}$, this observation indicates that, for $i = 1, \cdots, \abs{\gZ}$,
     \begin{align*}
         \abs{\gS \times \gA} \leq \text{rank}(F_i^{\star}) \leq \min \left( \abs{\gS \times \gA}, d \right) \quad \text{and} \quad d \geq \abs{\gS \times \gA} \implies \text{rank}(F_i^{\star}) = \abs{\gS \times \gA}.
     \end{align*}
     Thus, the column space of each forward representation matrix $F_i^{\star}$ is also equivalent to the $\abs{\gS \times \gA}$-dimensional real space:
     \begin{align*}
         \text{col}(F_i^{\star}) = \R^{\abs{\gS \times \gA}}.
     \end{align*}
     Therefore, by the definition of the pseudoinverse of a matrix, we have
     \begin{align*}
         \text{col} \left( M^{\pi}_i \text{diag}(\rho)^{-1} \right) = \text{col} \left( M^{\pi}_i \right) = \text{col}(F_i^{\star}) \implies F_i^{\star} F_i^{\star +} M^{\pi}_i \text{diag}(\rho)^{-1} = M^{\pi}_i \text{diag}(\rho)^{-1},
     \end{align*}
     where $F_i^{\star +}$ is the pseudoinverse of matrix $F_i^{\star}$. 
     
     These intriguing observations help us find the additional conditions for the existence of ground-truth FB representation matrices: the ground-truth backward representation matrix $B^{\star}$ must be shared by each forward representation matrix $F^{\star}_i$ and each successor measure matrix $M^{\pi}_i$ as 
     \begin{align*}
          B^{\star} = F_1^{\star +} M^{\pi}_1 \text{diag}(\rho)^{-1} = F_2^{\star +} M^{\pi}_2 \text{diag}(\rho)^{-1} = \cdots = F_{\abs{\gZ}}^{\star +} M^{\pi}_{\abs{\gZ}} \text{diag}(\rho)^{-1}.
     \end{align*}
     Importantly, the ground-truth FB representation matrices are not unique because we can multiply both $F_{i}^{\star}$ and $B^{\star}$ by the same orthonormal rotation matrix $Q_{\text{rot}}$ to recover the same product.
     
     \paragraph{Verifying the optimal policy adaptation.} Until now, we have only focused on discussing the conditions for the existence of ground-truth FB representation matrices $F^{\star}_\gZ, B$ that fit the successor measure ratio $M^{\pi}_{\gZ} \text{diag}(\rho)^{-1}$. It remains unclear whether this pair of FB representation matrices will enable optimal policy adaptation (Eq.~\ref{eq:matrix-fb-adaptation}). We next prove that the latent variable $z_r$ recovers the optimal policy and the optimal Q-value for any reward vector $r \in \R^{\abs{\gS \times \gA}}$. 
     
     Formally, given the FB representation matrices $F^{\star}_\gZ, B$ and the reward vector $r$, we denote the forward representation matrix for the latent variable $z_r$ as $F_{z_r} \in \R^{\abs{\gS \times \gA} \times \abs{\gS \times \gA}}$ (an example of Eq.~\ref{eq:latent-conditioned-forward-repr-matrix}). This forward representation matrix $F^{\star}_{z_r}$ induces a latent-conditioned policy $\pi(\ph{a} \mid \ph{s}, z_r)$ with the associated successor measure matrix $M^{\pi}_{z_r} \in \R^{\abs{\gS \times \gA} \times \abs{\gS \times \gA}}$. Together with the latent variable $z_r$, the forward representation matrix $F^{\star}_{z_r}$ and the successor measure matrix $M^{\pi}_{z_r}$ satisfy
     \begin{align*}
         F^{\star}_{z_r} z_r &= F^{\star}_{z_r} B^{\star} (r \odot \rho) \\
         &\overset{(a)}{=} M^{\pi}_{z_r} \text{diag}(\rho)^{-1} (r \odot \rho) \\
         &= M^{\pi}_{z_r} r \\
         &\overset{(b)}{=} Q_{z_r},
     \end{align*}
     where, in \emph{(a)}, we apply the definition of the inverse of a diagonal matrix and the definition of elementwise product, and, in \emph{(b)}, we apply the definition of the Q-value vector. Since the latent-conditioned policy $\pi(\ph{a} \mid \ph{s}, z_r)$ is maximizing the inner product $F^{\star}(\ph{s}, \ph{a}, z_r)^{\top} z_r$ (Eq.~\ref{eq:latent-conditioned-policy}), we have $\pi(\ph{a} \mid \ph{s}, z_r) = \argmax Q_{z_r}(\ph{s}, \ph{a})$. By definition, $Q_{z_r}$ is the optimal Q-value vector for the reward vector $r$ with the optimal policy $\pi(\ph{a} \mid \ph{s}, z_r)$.
    \end{proof}
\end{repproposition}

\subsection{Low-Rank Approximation Incurs Arbitrary Errors}
\label{appendix:low-rank-arbitrary-err}

\begin{repcorollary}{corollary:low-rank-arbitrary-error}
    Given any marginal probability measure vector $\rho \in \R^{\abs{\gS \times \gA}}$ and a representation dimension $d < \abs{\gS \times \gA}$, the FB representation matrices $F_\gZ \in \R^{ \abs{\gZ \times \gS \times \gA} \times d }$ and $B \in \R^{d \times \abs{\gS \times \gA}}$ learned by the FB algorithm is a low-rank approximation of the successor measure ratio $M^{\pi}_\gZ \text{diag}(\rho)^{-1} \in \R^{\abs{\gS \times \gA}}$. 
    
    For any reward vector $r \in \R^{\abs{\gS \times \gA}}$, this low-rank approximation induces the latent variable $z_r = B (r \odot \rho) \in \R^{d}$ and the optimal Q-value prediction $Q_{z_r} = F_{z_r} z_r \in \R^{\abs{\gS \times \gA}}$. Denote the error in the optimal Q-value prediction as
    \begin{align*}
        \epsilon(r) = \norm*{Q^{\star}_r - F_{z_r} z_r}_{\infty}.
    \end{align*}
    Then, for any $c > 0$, there exists a reward vector $r_{\text{null}} \in \R^{\abs{\gS \times \gA}}$ in the null space of $B$ such that $\epsilon(r_{\text{null}}) \geq c$.

    \begin{proof}
        From Proposition~\ref{prop:fb-existence}, we know that the rank of the successor measure matrix $M^{\pi}_{\gZ}$ is $\abs{\gS \times \gA}$. When $d < \abs{\gS \times \gA}$, the FB representation matrices produce a low-rank approximation on the successor measure ratio:
        \begin{align*}
            \text{rank}(F_\gZ B) \leq d < \abs{\gS \times \gA} = \text{rank}( M^{\pi}_\gZ \text{diag}(\rho)^{-1} ).
        \end{align*}
        In this case, the rank of the backward representation matrix satisfies
        \begin{align*}
            \text{rank}(B) \leq \min\left( d, \abs{\gS \times \gA} \right) = d < \abs{\gS \times \gA}.
        \end{align*}
        Thus, the backward representation matrix is \emph{not} full column rank, suggesting that there exists a \emph{non-zero} reward vector $r_{\text{null}} \in \R^{\abs{\gS \times \gA}}$ in the null space of $B$ that induces a zero latent variable: 
        \begin{align*}
            z_{r_{\text{null}}} = B(r_{\text{null}} \odot \rho) = \begin{bmatrix}
                0 \\ 
                \vdots \\
                0
            \end{bmatrix} \in \R^{d}.
        \end{align*}
        For this zero latent variable, we always have $F_{z_{r_{\text{null}}}} z_{r_{\text{null}}} = 0$. However, the optimal Q-value for the reward $r_{\text{null}}$ is not necessarily zero. Therefore, the optimal Q-value prediction error $\epsilon(r_{\text{null}})$ can be arbitrarily large by scaling the non-zero entries in $r_{\text{null}}$. We conclude that for any $c > 0$, there exists a reward vector $r_{\text{null}}$ such that $\epsilon(r_{\text{null}}) \geq c$. In other words, the errors in the optimal Q-value prediction can be arbitrarily large.
    \end{proof}
\end{repcorollary}

\subsection{Equivariant to Affine Transformations of Rewards}
\label{appendix:equivariant-to-affine-r}

We first prove the equivariant property of the Q-value for any positive affine transition under any policy. We will then use this property to derive the equivariant property of the ground-truth forward representations in FB.

\begin{lemma}
    \label{lemma:q-equivariant}
    Given a policy $\pi: \gS \to \Delta(\gA)$, a reward function $r: \gS \times \gA \to \R$, a positive scalar $\nu > 0$, and an offset $\xi \in \R$, for a state-action pair $(s, a)$, let the transformed reward be $\nu r(s, a) + \xi$. Then, the Q-value of the reward function $r$ and the Q-value of the reward $\nu r + \xi$ satisfies $Q^{\pi}_{\nu r + \xi}(s, a) = \nu Q^{\pi}_r (s, a) + \xi$.
\end{lemma}

The one-line proof of this lemma will use the definition of Q-value, which is a sum of cumulative discounted rewards, and apply the affine transformation to each reward in that summation. Now, for the ground-truth forward representations, we also have the equivariant property:

\begin{repproposition}{prop:fb-cylinder-proj}
    For a state-action pair $(s, a)$, a reward function $r: \gS \times \gA \to \R$, a positive scalar $\nu > 0$, an offset $\xi \in \R$, the ground-truth FB representation functions $F^{\star}: \gS \times \gA \times \gZ \to \gZ$, $B^{\star}: \gS \times \gA \to \gZ$, and a marginal measure $\rho \in \Delta( \gS \times \gA )$, let the transformed reward function be $\nu r(s, a) + \xi$. Then, $z_r$ is the latent variable indexing the optimal Q-value for the reward $r$, and $z_{\nu r + \xi}$ is the latent variable indexing the optimal Q-value for the reward $\nu r + \xi$. Furthermore, the ground-truth forward-backward representations are invariant to the affine transformation with positive scaling in latent variables, i.e., $F^{\star}(s, a, z_{\nu r  + \xi}) = F^{\star}(s, a, z_r)$.
    \begin{proof}
        By Eq.~\ref{eq:fb-latent-of-reward} in Definition~\ref{def:fb-adaptation}, we can write $z_r$ and $z_{\nu r + \xi}$ as
        \begin{align}
            z_r &= \int_{\gS \times \gA} B^{\star}(s_f, a_f) r(s_f, a_f) \rho(s_f, a_f) ds_f da_f \nonumber \\
            z_{\nu r + \xi} &= \int_{\gS \times \gA} B^{\star}(s_f, a_f) (\nu r(s_f, a_f) + \xi) \rho(s_f, a_f) ds_f da_f \nonumber \\
            &= \nu \int_{\gS \times \gA} B^{\star}(s_f, a_f) r(s_f, a_f) \rho(s_f, a_f) ds_f da_f + \xi \int_{\gS \times \gA} B^{\star}(s_f, a_f) \rho(s_f, a_f) ds_f da_f \nonumber \\
            &= \nu z_r + \xi z_{\text{one}},
            \label{eq:fb-affine-latent}
        \end{align}
        where we denote the latent variable for a reward that consistently equals $1$ ($r(\ph{s}, \ph{a}) = 1$) as $z_{\text{one}}$. One intriguing property of the latent variable $z_{\text{one}}$ is that, for any other latent variable $z$
        \begin{align}
            F^{\star}(s, a, z)^{\top} z_{\text{one}} &= \int_{\gS \times \gA} F^{\star}(s, a, z)^{\top} B^{\star}(s_f, a_f) \rho(s_f, a_f) d s_f d a_f \nonumber \\
            &\overset{(a)}{=} \int_{\gS \times \gA} \frac{M^{\pi}(s_f, a_f \mid s, a, z)}{\rho(s_f, a_f)} \cdot \rho(s_f, a_f) ds_f da_f \nonumber \\
            &= \int_{\gS \times \gA} M^{\pi}(s_f, a_f \mid s, a, z) ds_f da_f \nonumber \\
            &\overset{(b)}{=} 1,
            \label{eq:fb-latent-one}
        \end{align}
        where, in \emph{(a)}, we use the definition of ground-truth FB representations in Eq.~\ref{eq:fb-sm-ratio}, and, in \emph{(b)}, we apply the definition of the successor measure.
        Since $z_r$ indexes the optimal Q-value for the reward $r$ and $z_{\nu r + \xi}$ indexes the optimal Q-value for the reward $\nu r + \xi$, we have
        \begin{align*}
            Q^{\star}_r(s, a) &= F^{\star}(s, a, z_r)^{\top} z_r \\
            Q^{\star}_{\nu r + \xi}(s, a) &= F^{\star}(s, a, z_{\nu r + \xi})^{\top} z_{\nu r + \xi} \\
            &\overset{(a)}{=} \nu F^{\star}(s, a, \nu z_r + \xi z_{\text{one}})^{\top} z_r + \xi F^{\star}(s, a, \nu z_r + \xi z_{\text{one}})^{\top} z_{\text{one}} \\
            &\overset{(b)}{=} \nu F^{\star}(s, a, z_{\nu r + \xi})^{\top} z_r + \xi,
        \end{align*}
        where, in \emph{(a)}, we apply the relationship in Eq.~\ref{eq:fb-affine-latent}, and, in \emph{(b)}, we apply the property of $z_{\text{one}}$ in Eq.~\ref{eq:fb-latent-one}. Finally, since an affine transformation with positive scaling does not change the optimal policy~\citep{russell1995modern, ng1999policy}, the Q-value $Q^{\star}_r$ and the Q-value $Q^{\star}_{\nu r + \xi}$ satisfy Lemma~\ref{lemma:q-equivariant}. Using the conclusion from Lemma~\ref{lemma:q-equivariant}, we have
        \begin{align*}
            Q^{\star}_{\nu r + \xi}(s, a) &= \nu Q^{\star}_r(s, a) + \xi \\
            \implies \nu F^{\star}(s, a, z_{\nu r + \xi})^{\top} z_r + \xi &= \nu F^{\star}(s, a, z_{r})^{\top} z_r + \xi \\
            \implies F^{\star}(s, a, z_{\nu r + \xi}) &= F^{\star}(s, a, z_{r}),
        \end{align*}
        where the last identity holds because $\nu > 0$ and $\xi \in \R$ are arbitrary.
    \end{proof}
\end{repproposition}

\subsection{Deriving the FB Representation Learning Objective}
\label{appendix:FB-repr-loss}

We derive the LSIF loss for the FB algorithm that learns forward-backward representations in a temporal-difference manner. First, we replace the successor measure in Eq.~\ref{eq:mc-fb} using the recursive Bellman equation in Eq.~\ref{eq:successor-measure-bellman-eq}, decomposing the ratio $M^{\pi}(\ph{s_f}, \ph{a_f} \mid \ph{s}, \ph{a}, \ph{z}) / \rho(\ph{s_f}, \ph{a_f})$ into a convex combination (with weight $\gamma$) between the in-place ratio $\delta(\ph{s_f}, \ph{a_f} \mid \ph{s}, \ph{a}) / \rho(\ph{s_f}, \ph{a_f})$ and the ratio at the next time step $M^{\pi}(\ph{s_f}, \ph{a_f} \mid \ph{s'}, \ph{a'}, \ph{z}) / \rho(\ph{s_f}, \ph{a_f})$: 

\begin{align*}
    \frac{1}{2} \E_{\substack{p(s, a, z), \, \rho(s_f, a_f) \\ p(s' \mid s, a), \, \pi(a' \mid s', z) }} \left[ \left( F(s, a, z)^{\top} B(s_f, a_f) - (1 - \gamma) \frac{\delta(s_f, a_f \mid s, a)}{ \rho(s_f, a_f) } - \gamma \frac{M^{\pi}(s_f, a_f \mid s', a', z)}{\rho(s_f, a_f)} \right)^2 \right].
\end{align*}

Second, we use target forward-backward representation functions $\bar{F}$ and $\bar{B}$ to replace the ground-truth ratio at the next time step $M^{\pi}(\ph{s_f}, \ph{a_f} \mid \ph{s'}, \ph{a'}, \ph{z}) / \rho(\ph{s_f}, \ph{a_f})$. The resulting loss function minimizes a Bellman error:

\begin{align*}
    \gL_{\text{TD FB}}(F, B) &= \frac{1}{2} \E_{\substack{p(s, a, z), \, \rho(s_f, a_f) \\ p(s' \mid s, a), \, \pi(a' \mid s', z) }} \left[ \left( F(s, a, z)^{\top} B(s_f, a_f) - y \right)^2 \right], \nonumber \\
    y &= (1 - \gamma) \frac{\delta(s_f, a_f \mid s, a)}{ \rho(s_f, a_f) } + \gamma \bar{F}(s', a', z)^{\top} \bar{B}(s_f, a_f).
\end{align*}

Now, expanding the mean squared error gives us
\begin{align*}
    \gL_{\text{TD FB}}(F, B) &= \frac{1}{2} \E_{\substack{p(s, a, z), \, \rho(s_f, a_f) \\ p(s' \mid s, a), \, \pi(a' \mid s', z) }} \left[ \left( F(s, a, z)^{\top} B(s_f, a_f) \right)^2 \right. \\
    &\left. \hspace{1.1em}- 2 \cdot F(s, a, z)^{\top} B(s_f, a_f) \cdot \left( (1 - \gamma) \frac{\delta(s_f, a_f \mid s, a)}{\rho(s_f, a_f)} + \gamma \bar{F}(s', a', z)^{\top} \bar{B}(s_f, a_f) \right) \right] + \text{const}. \\
    &= \frac{1}{2} \E_{ p(s, a, z), \, \rho(s_f, a_f) } \left[ \left( F(s, a, z)^{\top} B(s_f, a_f) \right)^2 \right] \\
    &\hspace{1.1em} - (1 - \gamma) \E_{p(s, a, z), \, \rho(s_f, a_f)} \left[ \frac{\delta(s_f, a_f \mid s, a)}{\rho(s_f, a_f)} F(s, a, z)^{\top} B(s_f, a_f) \right] \\
    &\hspace{1.1em} - \gamma \E_{ \substack{p(s, a, z), \, \rho(s_f, a_f) \\ p(s' \mid s, a), \, \pi(a' \mid s', z) } } \left[ \bar{F}(s', a', z)^{\top} \bar{B}(s_f, a_f) \cdot F(s, a, z)^{\top} B(s_f, a_f) \right] + \text{const.} \\
    &\overset{(a)}{=} \frac{1}{2} \E_{p(s, a, z), \, \rho(s_f, a_f)} \left[ \left( F(s, a, z)^{\top} B(s_f, a_f) \right)^2 \right] - (1 - \gamma) \E_{p(s, a, z)} \left[ F(s, a, z)^{\top} B(s, a) \right] \\
    &\hspace{1.1em} - \gamma \E_{ \substack{p(s, a, z), \, \rho(s_f, a_f) \\ p(s' \mid s, a), \, \pi(a' \mid s', z) } } \left[ \bar{F}(s', a', z)^{\top} \bar{B}(s_f, a_f) \cdot F(s, a, z)^{\top} B(s_f, a_f) \right] + \text{const.} \\
    &\overset{(b)}{=} \frac{1}{2} \E_{ \substack{p(s, a, z), \, \rho(s_f, a_f) \\ p(s' \mid s, a), \, \pi(a' \mid s', z) } } \left[ \left( F(s, a, z)^{\top} B(s_f, a_f) - \gamma \bar{F}(s', a', z)^{\top} \bar{B}(s_f, a_f) \right)^2 \right] \\
    &\hspace{1.1em} - (1 - \gamma) \E_{p(s, a, z)} \left[ F(s, a, z)^{\top} B(s, a) \right] + \text{const.},
\end{align*}
where, in \emph{(a)}, we apply the property of the delta measure, and, in \emph{(b)}, we rearrange the quadratic terms by definition. Finally, we can ignore the constant for simplicity.

\subsection{The FB Bellman Operator Is Not a $\gamma$-Contraction}
\label{appendix:FB-algo-fixed-point}

This section proves a negative result about the FB Bellman operator $\gT_{\text{FB}}$. We show that the FB Bellman operator is not a $\gamma$-contraction. Therefore, the Banach fixed-point theorem fails to provide a guarantee to the FB algorithm that iteratively applies the FB Bellman operator to the forward-backward representation functions from the previous iteration. These results suggest that we need alternative theoretical tools to prove the convergence of the FB algorithm to a fixed point.

\begin{repproposition}{prop:fb-gamma-contraction}
    For any $\gamma \in [0, 1)$ and $p \geq 1$, the FB Bellman operator $\gT_{\text{FB}}$ is not a $\gamma$-contraction under the $L^{p}$-norm. Thus, the Banach fixed-point theorem is not applicable to the FB Bellman operator.
\end{repproposition}
\begin{proof}
We prove that the FB Bellman operator is not a $\gamma$-contraction by contradiction. Let the FB Bellman operator $\gT_{\text{FB}}$ be a $\gamma$-contraction under the $L^{p}$-norm for any $\gamma \in [0, 1)$ and $p \geq 1$. This indicates that for any two pairs of forward-backward representation functions $f_1: \gS \times \gA \times \gZ \to \gZ$, $b_1: \gS \times \gA \to \gZ$, and $f_2: \gS \times \gA \times \gZ \to \gZ$, $b_2: \gS \times \gA \to \gZ$, which induced the latent-conditioned policies $\pi_1(\ph{a} \mid \ph{s}, \ph{z}) = \delta \left( \ph{a} \mid \argmax_{a \in \gA} f_1(\ph{s}, a, \ph{z})^{\top} \ph{z} \right)$ and $\pi_2(\ph{a} \mid \ph{s}, \ph{z}) = \delta \left( \ph{a} \mid \argmax_{a \in \gA} f_2(\ph{s}, a, \ph{z})^{\top} \ph{z} \right)$, we have
\begin{align}
    \norm*{\gT_{\text{FB}} (f_1^{\top} b_1) - \gT_{\text{FB}} (f_2^{\top} b_2)}_{p} \leq \gamma \norm*{ f_1^{\top} b_1 - f_2^{\top} b_2 }_{p}.
    \label{eq:fb-gamma-contraction-ineq}
\end{align}
We consider any current state-action pair $(s, a)$, any latent $z \in \gZ$, and any future state-action pair $(s_f, a_f)$. For the LHS of the inequality, we have
\begin{align}
    &\gT_{\text{FB}}( f_1(s, a, z)^{\top} b_1(s_f, a_f) ) - \gT_{ \text{FB} }( f_2(s, a, z)^{\top} b_2(s_f, a_f) ) \nonumber \\
    &= \gamma \E_{p(s' \mid s, a), \pi_1(a' \mid s', z)} \left[ f_1(s', a', z)^{\top} b_1(s_f, a_f)  \right] - \gamma \E_{p(s' \mid s, a), \pi_2(a' \mid s', z)} \left[ f_2(s', a', z)^{\top} b_2(s_f, a_f)  \right] \nonumber \\
    &= \gamma \E_{p(s' \mid s, a)} \left[ \E_{\pi_1(a' \mid s', z)} \left[ f_1(s', a', z)^{\top} b_1(s_f, a_f) \right] - \E_{\pi_2(a' \mid s', z)} \left[ f_2(s', a', z)^{\top} b_2(s_f, a_f) \right] \right].
    \label{eq:fb-gamma-contraction-ineq-lhs}
\end{align}
Without loss of generality, we can set $f_2 = Q_{\text{rot}} f_1 $ and $b_2 = Q_{\text{rot}} b_1$, where $Q_{\text{rot}} \in \R^{d \times d}$ is an orthonormal rotation matrix, as in~\citet{touati2022does}. In this case, the inner products satisfy $f_2(s, a, z)^{\top} b_2(s_f, a_f) = f_1(s, a, z)^{\top} Q_{\text{rot}}^{\top} Q_{\text{rot}} b_1(s_f, a_f) = f_1(s, a, z)^{\top} b_1(s_f, a_f)$. Thus, the RHS of Eq.~\ref{eq:fb-gamma-contraction-ineq} is always zero. However, the policy $\pi_1(\ph{a} \mid \ph{s}, \ph{z}) = \delta \left( \ph{a} \mid \argmax_{a \in \gA} f_1(\ph{s}, a, \ph{z})^{\top} \ph{z} \right)$ and the policy $\pi_2(\ph{a} \mid \ph{s}, \ph{z}) = \delta \left( \ph{a} \mid \argmax_{a \in \gA} f_1(\ph{s}, a, \ph{z})^{\top} Q_{\text{rot}} \ph{z} \right)$ do not necessarily take the same action at different states $s$ and latent $z$. Therefore, without loss of generality, when the expectations satisfy $\E_{\pi_1(a' \mid s', z)} \left[ f_1(s', a', z)^{\top} b_1(s_f, a_f) \right] > \E_{\pi_2(a' \mid s', z)} \left[ f_1(s', a', z)^{\top} b_1(s_f, a_f) \right]$, we can reduce the Eq.~\ref{eq:fb-gamma-contraction-ineq-lhs} to
\begin{align*}
    &\gT_{\text{FB}}( f_1(s, a, z)^{\top} b_1(s_f, a_f) ) - \gT_{ \text{FB} }( f_1(s, a, z)^{\top} Q_{\text{rot}}^{\top} Q_{\text{rot}} b_1(s_f, a_f) ) \\
    &= \gamma \E_{p(s' \mid s, a)} \left[ \E_{\pi_1(a' \mid s', z)} \left[ f_1(s', a', z)^{\top} b_1(s_f, a_f) \right] - \E_{\pi_2(a' \mid s', z)} \left[ f_1(s', a', z)^{\top} b_1(s_f, a_f) \right] \right] \\
    &> 0.
\end{align*}
Now every element inside the $L^p$-norm on the LHS of Eq.~\ref{eq:fb-gamma-contraction-ineq} is positive, while every element inside the $L^p$-norm on the RHS is zero, which is a contradiction. Hence, we conclude that the FB Bellman operator $\gT_{\text{FB}}$ is not a $\gamma$-contraction under the $L^p$-norm. Consequently, the Banach fixed-point theorem cannot be applied to $\gT_{\mathrm{FB}}$ to conclude the existence or uniqueness of a fixed point. It is unclear whether the FB algorithm (approximately) has a convergence guarantee or not.
\end{proof}

\subsection{Connecting One-Step FB to a Singular Value Decomposition}
\label{appendix:onestep-fb-svd-connection}

One intriguing interpretation of the one-step FB algorithm is that it learns the SVD of the behavioral successor measure ratio. We can make this connection precise by considering discrete MDPs with finite numbers of states and actions. Applying Lemma~\ref{lemma:successor-measure-expr-and-rank}, we can compute the behavioral successor measure as 
\begin{align*}
    M^{\pi_{\beta}} = (1- \gamma) \left( I_{\abs{\gS \times \gA}}-\gamma P^{\pi_{\beta}} \right)^{-1},
\end{align*}
and write the behavioral successor measure ratio using notations in Appendix~\ref{appendix:fb-repr-existence} as 
\begin{align*}
    M^{\pi_{\beta}} \text{diag}(\rho)^{-1} = (1- \gamma) \left( I_{\abs{\gS \times \gA}}-\gamma P^{\pi_{\beta}} \right)^{-1} \text{diag}(\rho)^{-1}.
\end{align*}
Applying the SVD to the matrix $M^{\pi_{\beta}} \text{diag}(\rho)^{-1}$, we have $M^{\pi_{\beta}} \text{diag}(\rho)^{-1} = U_{\beta} \Sigma_{\beta} V^{\top}_{\beta}$, where $U_{\beta}$ and $V_{\beta}$ are two orthonormal matrices and $\Sigma_{\beta}$ is the square singular matrix. Since the behavioral successor measure ratio is fixed after fixing the behavioral policy, the one-step FB algorithm uses behavioral FB representations to fit a static target. In particular, we can set the representation dimension to $d = \abs{\gS \times \gA}$ and let
\begin{align*}
    F^{\star}_{\beta} = U_\beta \Sigma_{\beta}, \quad B^{\star}_\beta = V^{\top}_{\beta}
\end{align*}
to obtain a pair of ground-truth behavioral FB representations. There are two important properties for this solution:
\begin{remark}
    The ground-truth behavioral FB representations are not unique. In particular, for a pair of ground-truth FB representations $F^{\star}_{\beta}$ and $B^{\star}_{\beta}$, and an orthonormal rotation matrix $Q_{\text{rot}} \in \R^{d \times d}$, $Q_{\text{rot}} F^{\star}_{\beta}$ and $Q_{\text{rot}} B^{\star}_{\beta}$ is also a pair of solution.
\end{remark}

\begin{remark}
    The ground-truth behavioral FB representations $F^{\star}_{\beta} = U_\beta \Sigma_{\beta}$ and $B^{\star}_\beta = V^{\top}_{\beta}$ minimizes both the TD one-step FB loss $\gL_{\text{TD one-step FB}}$ and the orthonormalization regularization $\gL_{\text{ortho}}$. In particular, $F^{\star}_{\beta} B^{\star}_{\beta}$ is the SVD of the behavioral successor measure ratio and $B_{\beta}^{\star} B^{\star \top}_{\beta} = V_{\beta}^{\top} V_{\beta} = I_{d} = I_{\abs{\gS \times \gA}}$.
\end{remark}

\subsection{One-Step FB Enables One Step of Policy Improvement}
\label{appendix:onestep-fb-policy-improvement}

In the same way that the FB representations are defined in Definition~\ref{def:fb-definition}, we fit the behavioral successor measure ratio using the behavioral FB representations. Thus, the ground-truth behavioral FB representation functions $F^{\star}_{\beta}: \gS \times \gA \to \gZ $ and $B^{\star}_{\beta}: \gS \times \gA \to \gZ$ satisfy
\begin{align}
    F^{\star}_{\beta}(s, a)^{\top} B^{\star}_{\beta}(s_f, a_f) = \frac{M^{\pi_{\beta}}(s_f, a_f \mid s, a) }{\rho(s_f, a_f)}.
    \label{eq:onestep-fb-sm-ratio}
\end{align}
For a reward function $r: \gS \times \gA \to \R$, again, in the same way that the reward-specific latent variable $z_r$ is defined in Eq.~\ref{eq:fb-latent-of-reward}, we define a new reward-specific latent variable $z_r^{\beta}$ using the behavioral backward representations as
\begin{align}
    z_r^{\beta} = \E_{(s_f, a_f) \sim \rho(s_f, a_f)} \left[ B^{\star}_{\beta}(s_f, a_f) r(s_f, a_f) \right].
    \label{eq:onestep-zero-shot-latent}
\end{align}
Now, we have $z^{\beta}_r$ indexing the behavioral Q-value $Q^{\beta}_r(s, a) = F_{\beta}^{\star}(s, a)^{\top} z^{\beta}_r$ (by Eq.~\ref{eq:onestep-fb-sm-ratio} and the relationship in Eq.~\ref{eq:relate-sr-to-q}) and the policy $\pi(a \mid s, z_r^{\beta})$ performs one-step policy improvement because
\begin{align*}
    \pi(a \mid s, z^{\beta}_r) = \argmax_{a \in \gA} F^{\star}_{\beta}(s, a)^{\top} z^{\beta}_r = \argmax_{a \in \gA} Q^{\beta}_r(s, a).
\end{align*}
Importantly, one-step policy improvement does not recover the optimal policy, which is the result of multi-step policy improvement until convergence. Thus, the one-step FB algorithm loses the optimal policy adaptation property. Nevertheless, prior work~\citep{brandfonbrener2021offline, eysenbach2022contrastive, park2025horizon, park2025transitive} has proven the success of one-step policy improvement in solving diverse RL problems. Therefore, we propose that one-step FB is a competitive method for both zero-shot adaptation (Sec.~\ref{subsec:offline-rl-experiments}) and providing a good initialization for fine-tuning on downstream tasks (Sec.~\ref{appendix:fine-tuning-experiments}).

\section{Experiment Details}
\label{appendix:experiment-details}

\subsection{Didactic Experiments for the FB Algorithm}
\label{appendix:didactic-exp-fb}

Since the latent space $\gZ$ contains an infinite number of latent variables, we parameterize the forward representation matrix $F_{\gZ}: \gZ \to \R^{\abs{\gS \times \gA} \times d}$ using a neural network. For the backward representations, we parameterize them as a differentiable matrix $B \in \R^{d \times \abs{\gS \times \gA}}$. To strictly align our empirical setup with theoretical analysis in Sec.~\ref{sec:understanding-fb}, we enforce a latent dimension of $d= \abs{\gS \times \gA}$ and fix the rank of both representations as $\abs{\gS \times \gA}$. Specifically, we consider the singular value decomposition $F_{\gZ} = U_F \Sigma_F V^\top_F$ and use neural networks to predict the elements of two orthonormal matrices $U_F$ and $V_F^{\top}$ via the Cayley transform and singular values in $\Sigma_F$. These networks independently predict the parameters of the left orthonormal matrix ($\abs{\gS \times \gA } \times \abs{\gS \times \gA}$), the singular values ($\abs{\gS \times \gA}$), and the right orthonormal matrix ($d \times d$). Each network consists of an MLP with $(32, 32, 32)$ units and GELU activations. The backward representation matrix $B$ is also constructed from learnable orthonormal matrices, $U_B$ and $V^\top_B$, and learnable singular values ($\Sigma_B$), as $B = U_B \Sigma_B V^\top_B$.

We train the algorithm for $10^5$ gradient steps with a batch size of $512$. We set the discount factor to $\gamma = 0.9$ during training. Optimization is performed using AdamW~\cite{loshchilov2018decoupled} optimizer with weight decay of $10^{-4}$, $\epsilon_{\text{adamw}} = 10^{-5}$ and learning rate of $10^{-4}$. We randomly sample latent variables $z$ at each training step and use another $1000$ randomly sampled latents for evaluation. Following the original FB implementation~\citep{touati2021learning}, we sample the latent variable $z$ from a scaled von Mises-Fisher distribution. We first sample a $d$-dimension standard Gaussian variable $x \sim \mathcal{N}(0, I_d)$ and a scalar centered Cauchy variable $u \sim \text{Cauchy}(0, 0.5)$, and then compute the latent variable as $z = \sqrt{d} u \frac{x}{\norm*{x}}$. We use the prior distribution $z \sim p_{\gZ}(z)$ to denote this sampling procedure. The latent $z$ is preprocessed as $\tilde{z} \leftarrow \frac{z}{\sqrt{1 + \norm{z}_2^2/d}}$ before being passed as input to the forward representation matrix $F_{\gZ}$. This transformation maps the infinite space of $\mathbb{R}^d$ to a bounded open ball of radius $\sqrt{d}$. Importantly, this mapping is a bijection that preserves differences in magnitude; therefore, latent vectors $z_r$ and $z_{\nu r + \xi}$ ($\nu > 0$) corresponding to differently scaled rewards remain distinct inputs to the neural network. 

Given any marginal measure vector $\rho \in \R^{\abs{\gS \times \gA}}$\footnote{We set the marginal measure to $\rho(\ph{s}, \ph{a}) \triangleq 1 / \abs{\gS \times \gA}$ in our experiments.}, to optimize the neural networks for forward-backward representations $F_{\gZ}$, $B$, we choose to use the MC FB loss $\gL_{\text{MC FB}}$ (Eq.~\ref{eq:mc-fb}) over a batch of latent variables:
\begin{align*}
    \gL_{\text{MC FB}}(F_{\gZ}, B) = \E_{z \sim p_{\gZ}(z)} \left[ \norm*{ F_{\gZ} B - M^{\pi}_{\gZ} \text{diag}(\rho)^{-1} }_F^2 \right],
\end{align*}
where $\norm{\cdot}_F$ denotes the Frobenius norm of a matrix. We choose this loss because the successor measure can be computed analytically in this discrete CMP (Lemma~\ref{lemma:successor-measure-expr-and-rank}):
for each latent variable $z$, we have
\begin{align*}
    M^{\pi}_{z} = (1- \gamma) \left(I_{\abs{\gS \times \gA}}-\gamma P^{\pi(\ph{a} \mid \ph{s}, z)}\right)^{-1}, \quad z \sim p_{\gZ}(z).
\end{align*}
Thus, the MC FB loss is an analytical analogy of the TD FB loss in Eq.~\ref{eq:td-fb}. In fact, the TD FB loss uses transition samples and target networks to approximate the MC FB loss, similar to FQE. See Sec.~\ref{subsec:fb-repr-obj} for the complete discussion. 

Instead of setting the latent-conditioned policy as the non-differentiable argmax policy: $\hat{\pi}(\ph{a} \mid \ph{s}, \ph{z}) = \delta( \ph{a} \mid \argmax_{a \in \gA} F(\ph{s}, a, \ph{z})^{\top} \ph{z})$, we choose to use the Boltzmann policy \begin{align}
\hat{\pi}(\ph{a} \mid \ph{s}, \ph{z}) = \frac{\exp( \tau_{\text{policy}} F(\ph{s}, \ph{a}, \ph{z})^{\top} \ph{z})}{ \sum_{a' \in \gA} \exp \left( \tau_{\text{policy}} F(\ph{s}, a', \ph{z})^{\top} \ph{z} \right) },
\label{eq:fb-didactic-exp-softmax-policy}
\end{align}
where $\tau_{\text{policy}}$ is a temperature for the softmax function fixed to $\tau_{\text{policy}} = 5 \times 10^{-3}$ during training. Following prior practice~\citep{touati2021learning}, we use a temperature $\tau_{\text{policy}} = 1$ during evaluation.

We evaluate the FB algorithm by aggregating statistics over $8$ random seeds using the fixed $1000$ evaluation latents $D_{\text{eval}} = \{ z_i \}_{i = 1}^{1000}$. For each sampled latent $z$, we recover the corresponding reward vector $r \in \R^{\gS \times \gA}$ as $r = (B^{-1} z) \oslash \rho$, where $\oslash$ denotes elementwise division and the backward representation matrix $B$ is invertible (full rank). We then compute the ground-truth optimal Q-value for each reward, $Q^*_r$, by running standard value iteration until convergence. We report the following four metrics as in Fig.~\ref{fig:didactic-exp}~\figleft:
\begin{enumerate}
    \item \textbf{Successor measure ratio prediction error.} This metric measures the fidelity of the learned FB representations in approximating the ground-truth successor measure ratio. We define the successor measure ratio prediction error as the mean squared error (MSE) between the ratio predicted by the FB representations and the ground-truth ratio. The ground-truth ratio is computed using the reward-maximizing policy induced by the forward representations. Formally, the successor measure ratio prediction error is
    \begin{align*}
        \epsilon_{\text{SMR}} = \E_{z \sim D_{\text{eval}}} \left[ \norm*{F_{\gZ} B - M^{\pi}_{\gZ} \text{diag}(\rho)^{-1} }_F^2 \right].
    \end{align*}

    \item \textbf{Optimal Q-value prediction error.} This metric measures the accuracy of the optimal Q-value predicted by the learned representation. For each latent variable $z$ with the corresponding reward vector $r$, the learned forward representation matrix $F_\gZ$ predicts the optimal Q-value as $\hat{Q}^{\star}_r(z) = F_z z \in \R^{\abs{\gS \times \gA}}$. On the other hand, we can compute the ground-truth optimal Q-value for reward vector $r$ using value iteration as $Q^{\star}_r(z) \in \R^{\abs{\gS \times \gA}}$. The optimal Q-value prediction error is defined as the MSE between the predicted Q-value and the ground-truth Q-values:
    \begin{align*}
        \epsilon_{Q^{\star}} = \E_{z \sim D_{\text{eval}}} \left[ \norm*{\hat{Q}^{\star}_{r}(z) - Q^{\star}_r(z)}_2^2 \right].
    \end{align*}
    \item \textbf{Forward KL divergence (optimal policy).} To evaluate the decision-making quality of the induced policy, we measure the forward KL divergence between the optimal policy derived from $Q^{\star}_r(z)$ and $\pi^{\star}(\ph{a} \mid \ph{s}, \ph{z})$ and the policy derived from $\hat{Q}_r(\ph{z})$ and $\hat{\pi}(\ph{a} \mid \ph{s}, \ph{z})$. We report the forward KL divergence averaged over all evaluation latents and all possible states:
    \begin{align*}
        \text{KL}_{\pi^{\star}} = \frac{1}{\abs{\gS}} \E_{z \sim D_{\text{eval}}} \left[ \sum_{s \in \gS} \KL \left( \pi^{\star}( \cdot \mid s, z) \parallel \hat{\pi} (\cdot \mid s, z) \right) \right]
    \end{align*}
    where $\KL (p \parallel q) = \sum_{x \in \gX} p(x) \log \left( \frac{p(x)}{q(x)} \right)$ is the standard KL divergence between two probability measures $p$ and $q$.
    
    \item \textbf{Q prediction equivariance error.} This metric assesses whether the learned Q-values respect the affine equivariance property as discussed in Lemma~\ref{lemma:q-equivariant} and Proposition~\ref{prop:fb-cylinder-proj}. Specifically, given a latent variable $z$ with the corresponding reward vector $r$, for a positive scalar, $\nu > 0$, and an offset, $\xi \in \mathbb{R}$, the predicted Q-value should satisfy the equivariance $\hat{Q}_r(z_{\nu r + \xi}) = \hat{Q}_r( \nu z + \xi z_{\text{one}} ) = \nu \hat{Q}_r(z) + \xi$, where $z_\text{one}$ is the latent variable corresponding to the all one reward vector. We sample $\nu$ and $\xi$ randomly and compute the following MSE:
    \begin{align*}
        \epsilon_{\text{equiv}} &= \E_{z \sim D_{\text{eval}}} \left[ \norm*{\hat{Q}_r(\nu z + \xi z_\text{one}) - \left( \nu \hat{Q}_r(z) + \xi \right)}_2^2 \right] \\
        &= \E_{z \sim D_{\text{eval}}} \left[ \norm*{ F_{\nu z + \xi z_{\text{one}}} \cdot (\nu z + \xi z_{\text{one}}) - \left( \nu F_z z + \xi \right)}_2^2 \right].
    \end{align*}
\end{enumerate}

\subsection{Didactic Experiments for the One-Step FB Algorithm}
\label{appendix:didactic-exp-onestep-fb}

The forward representation matrix $F_{\beta} \in \R^{\abs{\gS \times \gA} \times d}$ and the backward representation matrix $B_{\beta} \in \R^{d \times \abs{\gS \times \gA}}$ are parameterized directly as differentiable matrices. We set the latent dimension of $d = \abs{\gS \times \gA}$ and enforce a fixed rank of $\abs{\gS \times \gA}$ on both representations. Following the construction in Appendix~\ref{appendix:didactic-exp-fb}, both matrices are factorized via SVDs, where the decompositions are formed by learnable orthonormal matrices and singular values.

We train the algorithm for $10^5$ gradient steps to fit the analytical density ratio induced by a fixed policy $\pi_\beta$. For simplicity, we set the policy $\pi_{\beta}$ to be a uniform policy over the entire action space: $\pi_{\beta}(a \mid s) \triangleq 1 / \abs{\gA}$. Since the target density ratio $M^{\pi_\beta}\text{diag}(\rho)^{-1}$ is fixed given $\pi_{\beta}$, the learning one-step FB reduces to solving a supervised learning problem. Specifically, we minimize the Monte Carlo one-step FB (MC FB) loss, which resembles the MC FB loss in Eq.~\ref{eq:mc-fb}. Given the marginal measure vector $\rho \in \R^{\abs{\gS \times \gA}}$, the MC one-step FB loss is defined as:
\begin{align}
    \gL_{\text{MC one-step FB}} (F_{\beta}, B_{\beta}) = \norm{F_{\beta} B_{\beta} - M^{\pi_\beta}\text{diag}(\rho)^{-1} }_F^2,
    \label{eq:mc-onestep-fb}
\end{align}
where the fixed successor measure $M^{\pi_{\beta}}$ is computed as (Lemma~\ref{lemma:successor-measure-expr-and-rank})
\begin{align*}
    M^{\pi_{\beta}} = (1 - \gamma) \left( I_{\abs{\gS \times \gA}} - \gamma P^{\pi_{\beta}} \right)^{-1}.
\end{align*}
All optimization hyperparameters are kept identical to the FB implementation described in Appendix.~\ref{appendix:didactic-exp-fb}. 

Similar to the experiments in Appendix~\ref{appendix:didactic-exp-fb}, we evaluate one-step FB by aggregating the statistics over $8$ random seeds using a fixed batch of $1000$ evaluation latents. For each sampled latent $z$, we also recover the corresponding reward vector $r \in \R^{\abs{\gS \times \gA}}$ as $r = (B^{-1} z) \oslash \rho$. We then compute the Q-value of the fixed policy $Q^{\pi_\beta}$ by running standard value iteration until convergence. We report the following metrics:
\begin{itemize}
    \item \textbf{Successor measure ratio prediction error.} The metric measures the fidelity of the learned representation in approximating the fixed successor measure ratio. It is defined as the MSE between the ratio predicted by the one-step FB representations and the ground-truth ratio.
    \begin{align*}
        \epsilon_{\text{SMR}} = \norm{ F_{\beta} B_{\beta} - M^{\pi_\beta}\text{diag}(\rho)^{-1} }_F^2
    \end{align*}
    
    \item \textbf{Q-value prediction error.} This metric measures the accuracy of the Q-value predicted by the learned representation. For each latent variable $z$ that induces the reward vector $r$, the learned forward representation matrix $F_{\beta}$ predicts the Q-value as $\hat{Q}_r^{\pi_{\beta}}(z) = F_{\beta} z \in \R^{\abs{\gS \times \gA}}$. The Q-value prediction error is defined as the MSE between the predicted Q-value and the Q-value obtained via value iteration:
    \begin{align}
        \epsilon_{Q^{\pi_\beta}} =\E_{z \sim D_{\text{eval}}} \left[ \norm*{\hat{Q}_{r}^{\pi_{\beta}}(z) - Q^{\pi_\beta}_r(z)}_2^2 \right]
    \end{align}
    
    \item \textbf{Forward KL divergence ($\argmax_a Q^{\pi_\beta}$).} We now measure the decision-making quality of the induced policy. As discussed in Sec.~\ref{sec:method} and Appendix~\ref{appendix:onestep-fb-policy-improvement}, the latent-conditioned policy derived from the one-step FB algorithm performs one step of policy improvement over the predicted Q-value $\hat{Q}^{\pi_{\beta}}_r(z)$ as
    \begin{align*}
        \hat{\pi}_{\text{one-step}}( \ph{a} \mid \ph{s}, \ph{z} ) = \frac{ \exp(\tau_{\text{policy}} F_{\beta}(\ph{s}, \ph{a})^{\top} \ph{z}) }{ \sum_{a' \in \gA} \exp(\tau_{\text{policy}} F_{\beta}(\ph{s}, a')^{\top} \ph{z}) } \approx \argmax_{a \in \gA} \hat{Q}^{\pi_{\beta}}_r(s, a, z).
    \end{align*}
    This policy is trying to fit the one-step policy improvement over the ground-truth Q-value $Q^{\pi_{\beta}}_{r}(z)$. We will define the resulting policy from the one step of policy improvement over $Q^{\pi_{\beta}}_{r}(z)$ as
    \begin{align*}
        \pi_{\text{one-step}}( \ph{a} \mid \ph{s}, \ph{z} ) = \delta \left( \ph{a} \mid \argmax_{a \in \gA} Q^{\pi_{\beta}}_r(\ph{s}, a, \ph{z}) \right).
    \end{align*}
    To evaluate the performance of $\hat{\pi}_{\text{one-step}}$, we measure the forward KL divergence between $\pi_{\text{one-step}}$ and $\hat{\pi}_{\text{one-step}}$, averaging over all sampled latents and all possible states:
    \begin{align}
        \text{KL}_{\pi_{\text{one-step}}} = \frac{1}{\abs{\gS}} \E_{z \sim D_{\text{eval}}} \left[ \sum_{s \in \gS} \KL \left( \pi_{\text{one-step}}( \cdot \mid s, z) \parallel \hat{\pi}_{\text{one-step}} (\cdot \mid s, z) \right) \right]
    \end{align}
    
    \item \textbf{Q prediction equivariance error.} This metric assesses whether the learned Q-value respects the affine equivariance property as discussed in Lemma~\ref{lemma:q-equivariant} and Proposition~\ref{prop:fb-cylinder-proj}. Specifically, given a latent variable $z$ with the corresponding reward vector $r$, for a positive scalar, $\nu > 0$, and an offset, $\xi \in \mathbb{R}$, the predicted Q-value should satisfy the equivariance $\hat{Q}^{\pi_{\beta}}_r(z_{\nu r + \xi}) = \hat{Q}^{\pi_{\beta}}_r( \nu z + \xi z_{\text{one}} ) = \nu \hat{Q}^{\pi_{\beta}}_r(z) + \xi$. We sample $\nu$ and $\xi$ randomly and compute the following MSE:
    \begin{align*}
        \epsilon_{\text{equiv}} &= \E_{z \sim D_{\text{eval}}} \left[ \norm*{\hat{Q}^{\pi_{\beta}}_r(\nu z + \xi z_\text{one}) - \left( \nu \hat{Q}^{\pi_{\beta}}_r(z) + \xi \right)}_2^2 \right] \\
        &= \E_{z \sim D_{\text{eval}}} \left[ \norm*{ F_{\beta} \cdot (\nu z + \xi z_{\text{one}}) - \left( \nu F_{\beta} z + \xi \right)}_2^2 \right].
    \end{align*}
\end{itemize}

\subsection{Rationales of Selecting Prior Methods}
\label{appendix:baselines}

We compare one-step FB against $6$ prior unsupervised pre-training methods. The most relevant prior work is FB~\citep{touati2021learning}, which simultaneously learns a latent-conditioned policy and its occupancy measure. Other popular zero-shot methods learn state representations first and then use off-the-shelf RL algorithms (e.g., TD3~\citep{fujimoto2018addressing}) to maximize the intrinsic reward derived from those state representations. Among them, we mainly compare against three families of methods. First, some prior methods train state representations via consistency along single-step or multi-step samples from the successor measure: Laplacian~\citep{wu2018laplacian} and BYOL-$\gamma$~\citep{lawson2025self}. Second, we also compare against TD-JEPA~\citep{bagatella2025td}, a temporal-difference variant of joint-embedding predictive architecture~\citep{assran2023self} for zero-shot RL. TD-JEPA learns representations with a forward multi-step predictor similar to BYOL-$\gamma$, with an additional backward multi-step predictor to extract success features (latent representations). This method is a recent baseline for zero-shot RL. Third, another line of established methods derived state representations for the successor measure using variants of the expectile regression adapted from~\citet{kostrikov2021offline}: HILP~\citep{park2024foundation} and ICVF~\citep{ghosh2023reinforcement}. 

\subsection{Environments, Datasets, and Evaluation Protocols}
\label{appendix:envs-and-datasets}

\begin{figure}[t]
    \centering
    \begin{subfigure}[b]{0.75\textwidth}
    {\sffamily \color{gray}
    \begin{subfigure}[b]{0.24\linewidth}
        \includegraphics[width=\textwidth]{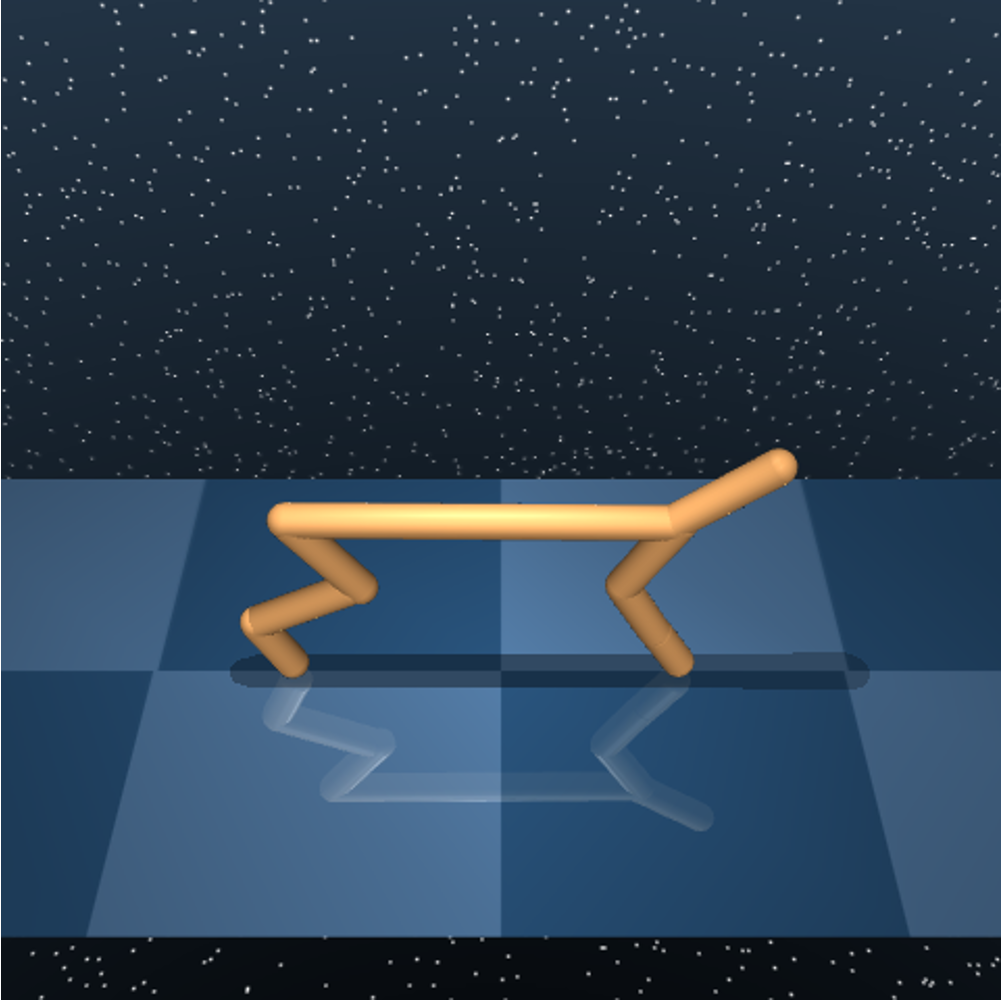}
        \centering {\fontsize{7.5pt}{7.5pt}\selectfont cheetah}
    \end{subfigure}
    \begin{subfigure}[b]{0.24\linewidth}
        \includegraphics[width=\textwidth]{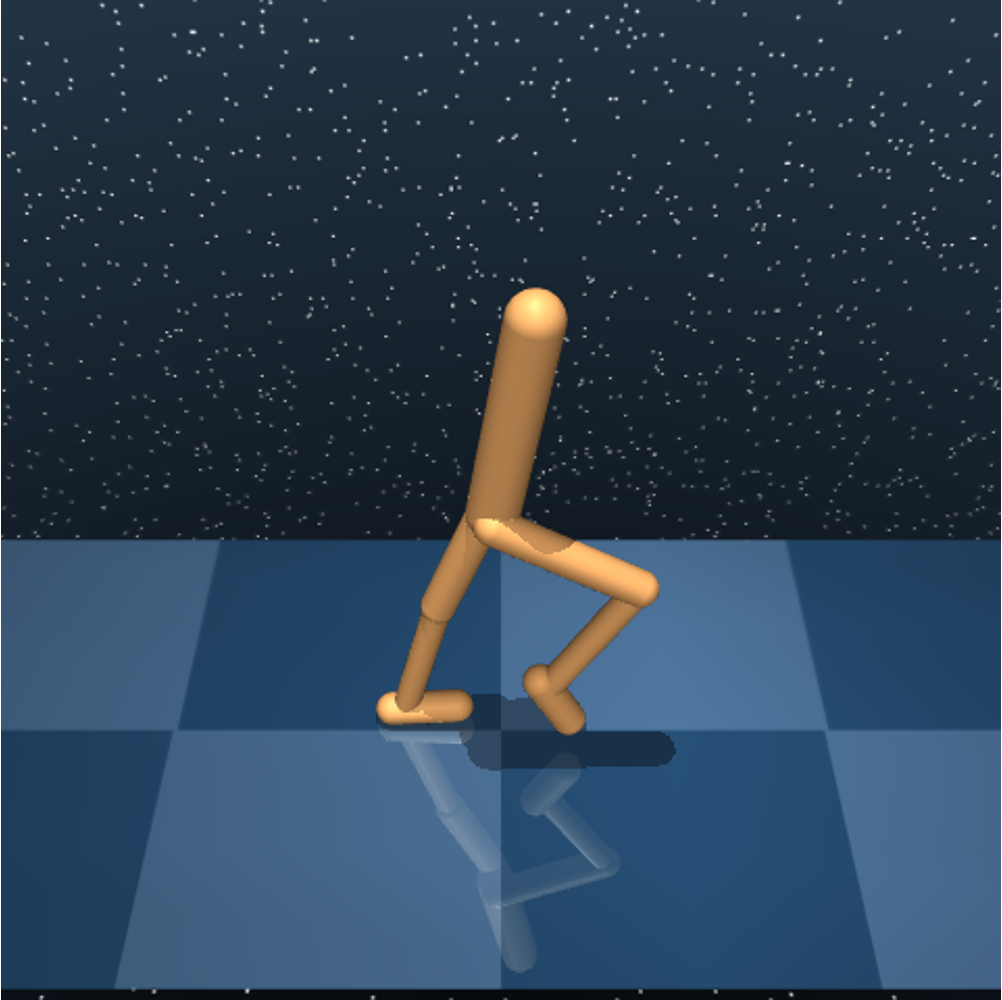}
        \centering {\fontsize{7.5pt}{7.5pt}\selectfont walker}
    \end{subfigure}
    \begin{subfigure}[b]{0.24\linewidth}
        \includegraphics[width=\textwidth]{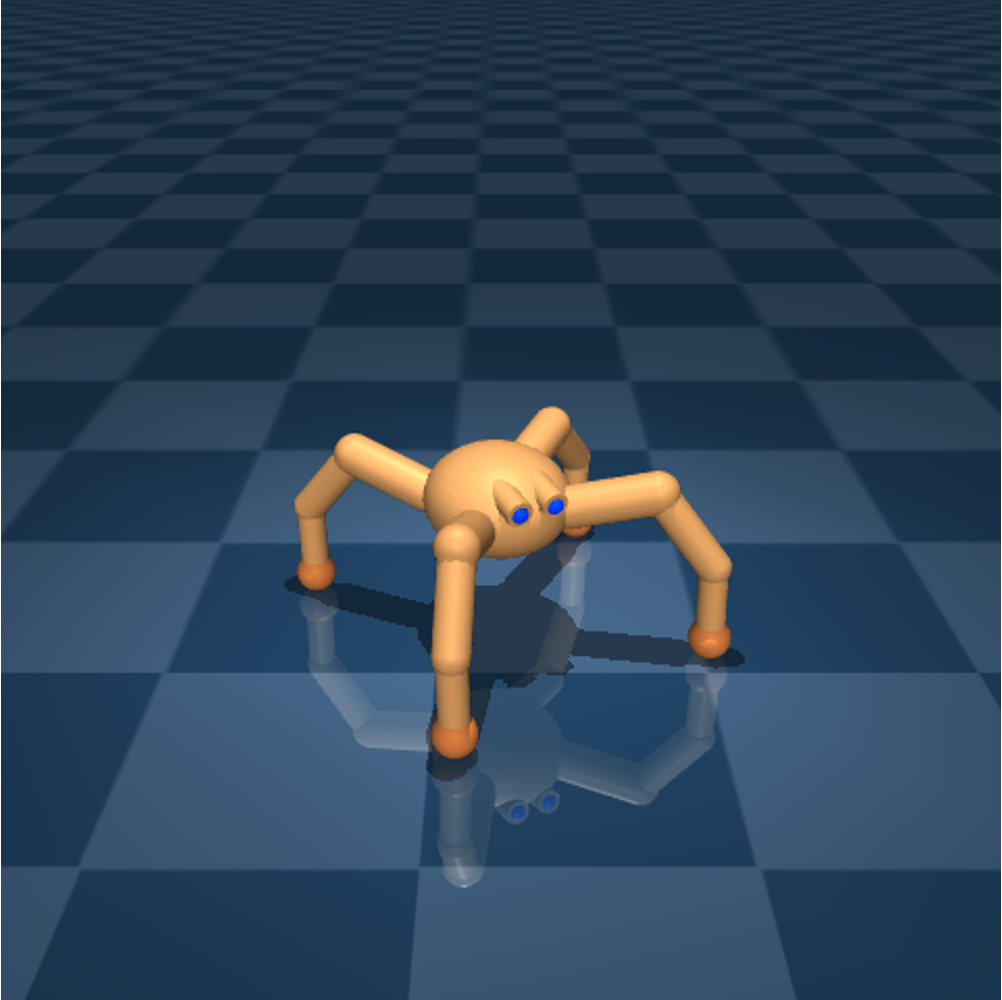}
        \centering {\fontsize{7.5pt}{7.5pt}\selectfont quadruped}
    \end{subfigure}
    \begin{subfigure}[b]{0.24\linewidth}
        \includegraphics[width=\textwidth]{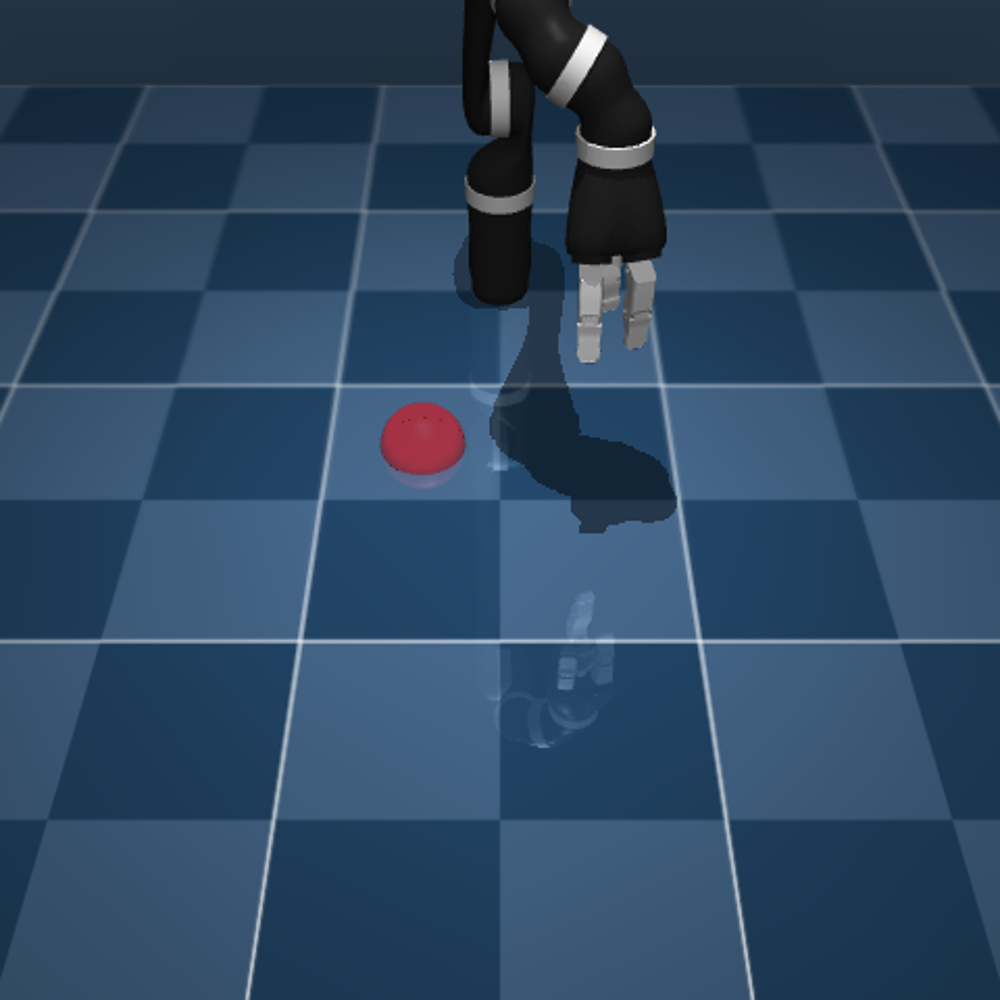}
        \centering {\fontsize{7.5pt}{7.5pt}\selectfont jaco}
    \end{subfigure}
    }  %
    \end{subfigure}
    \vfill
    \vspace{0.25em}
    \begin{subfigure}[b]{0.75\textwidth}
    {\sffamily \color{gray}
    \begin{subfigure}[b]{0.24\linewidth}
        \includegraphics[width=\textwidth]{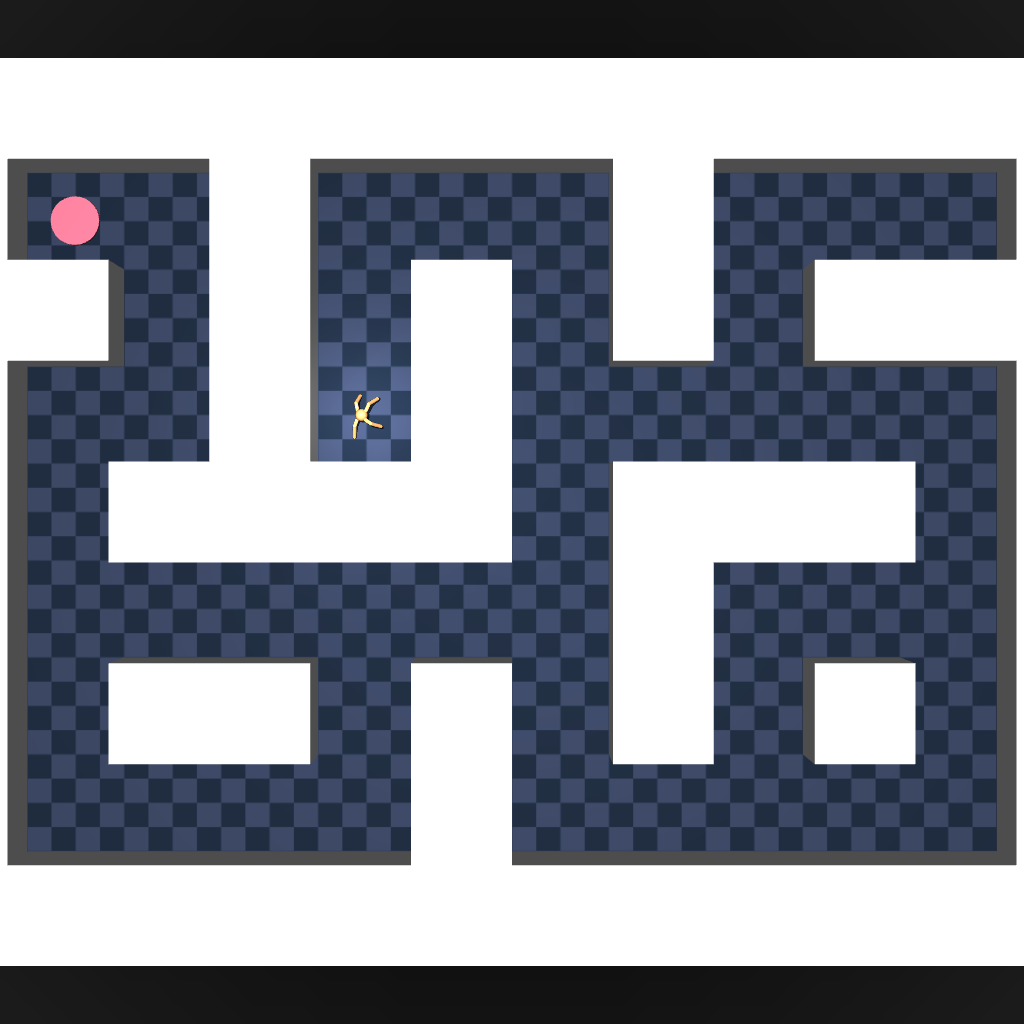}
        \centering {\fontsize{7.5pt}{7.5pt}\selectfont antmaze large navigate}
    \end{subfigure}
    \begin{subfigure}[b]{0.24\linewidth}
        \includegraphics[width=\textwidth]{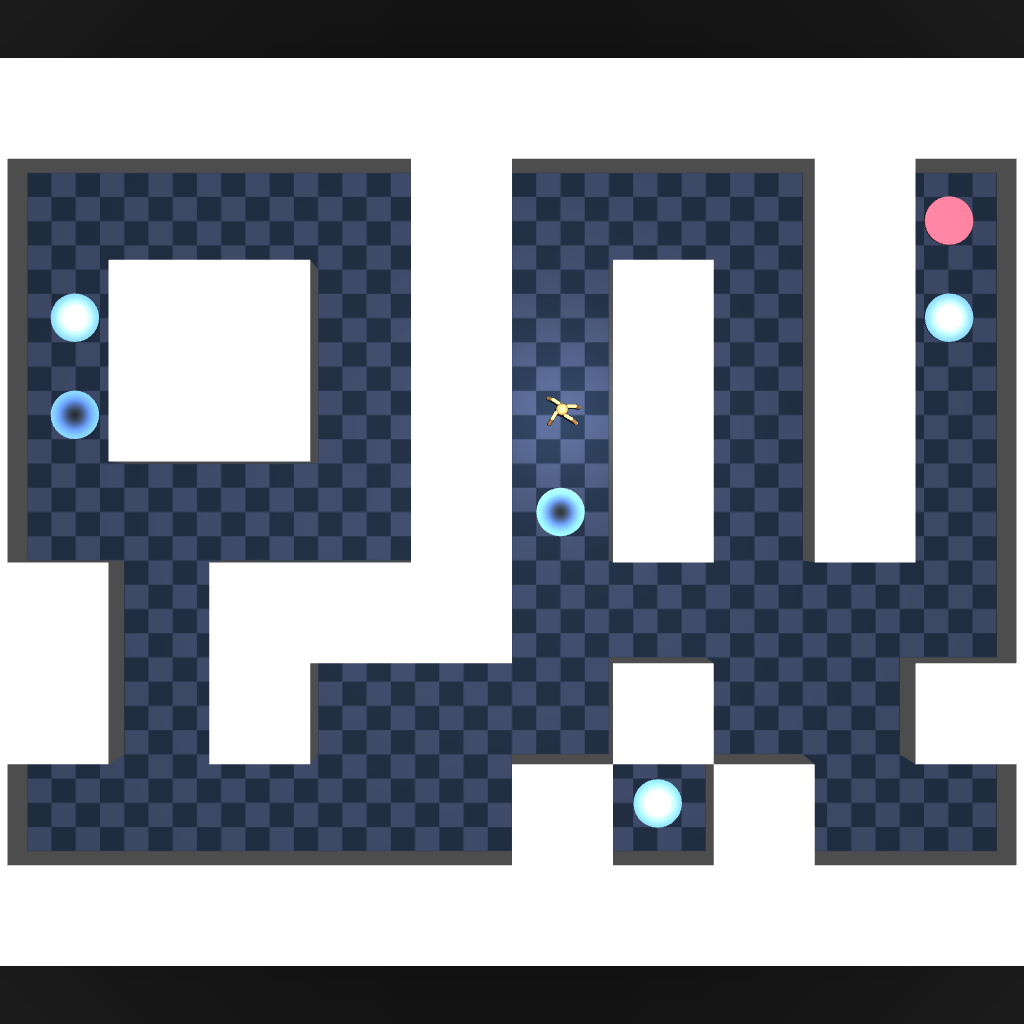}
        \centering {\fontsize{7.5pt}{7.5pt}\selectfont antmaze teleport navigate}
    \end{subfigure}
    \begin{subfigure}[b]{0.24\linewidth}
        \includegraphics[width=\textwidth]{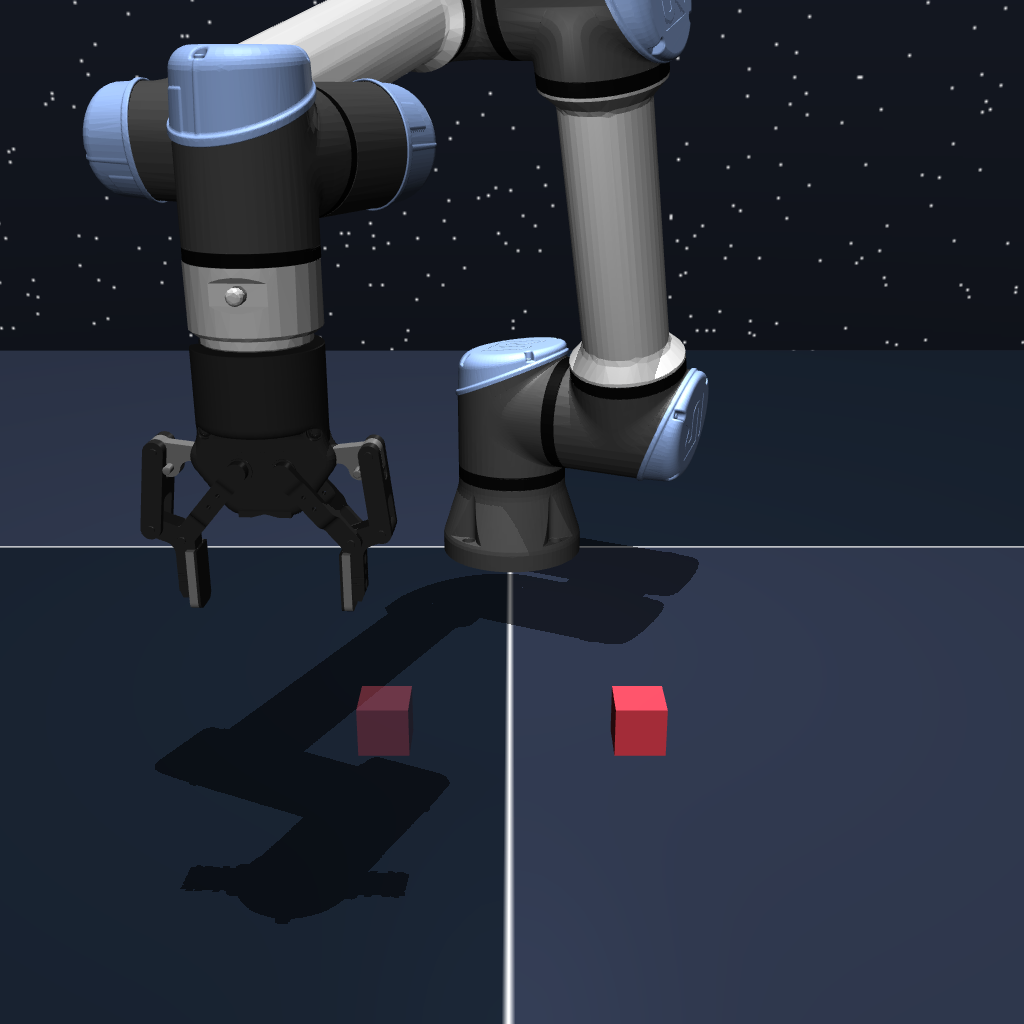}
        \centering {\fontsize{7.5pt}{7.5pt}\selectfont cube single play}
    \end{subfigure}
    \begin{subfigure}[b]{0.24\linewidth}
        \includegraphics[width=\textwidth]{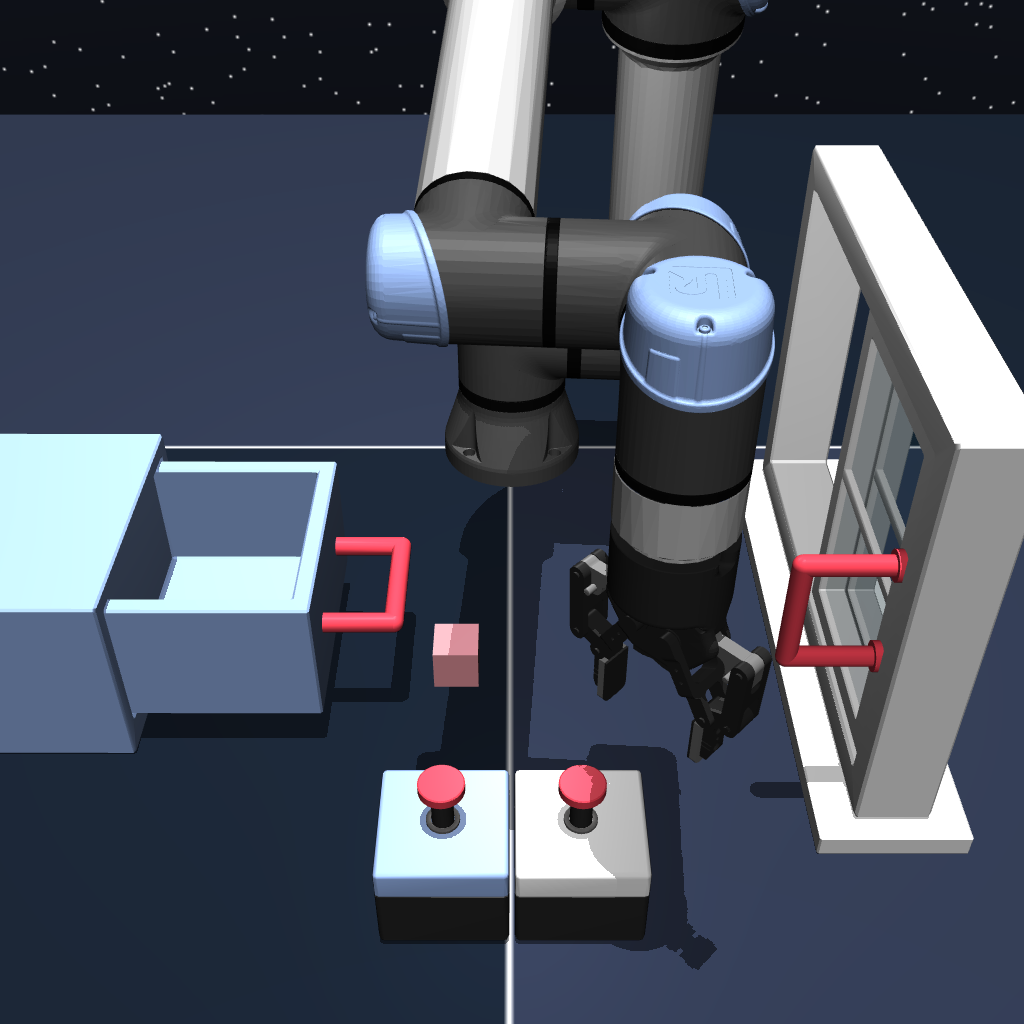}
        \centering {\fontsize{7.5pt}{7.5pt}\selectfont scene play}
    \end{subfigure}
    }  %
    \end{subfigure}
    \caption{\footnotesize \textbf{Domains for evaluation.} \figtop \, ExORL domains ($16$ state-based tasks). \figbottom \, OGBench domains ($20$ state-based tasks and $10$ image-based tasks).
    }
    \label{fig:domains}
\end{figure}

We focus on the offline setting and use standard offline RL benchmarks to compare one-step FB against prior methods. We select a set of $4$ state-based locomotion domains from the ExORL~\citep{yarats2022don} benchmark and a set of $4$ state-based robotic navigation and manipulation domains from the OGBench~\citep{park2024ogbench} benchmark (Fig.~\ref{fig:domains}). To test whether our method is able to directly take in RGB images as inputs, we additionally include $2$ image-based domains from the OGBench benchmark. Our experiments pre-train different methods for $10^6$ gradient steps and evaluate the zero-shot performance on a diverse set of tasks in each domain. 

\paragraph{Benchmarks.} The ExORL benchmark consists of a diverse set of locomotion tasks based on the DeepMind Control Suite~\citep{tassa2018deepmind}. Following prior work~\citep{park2024foundation}, we select $4$ domains from the entire benchmark, each containing $4$ tasks. These tasks involve controlling four robots (\texttt{cheetah}, \texttt{walker}, \texttt{quadruped}, and \texttt{jaco}) to complete different locomotion behaviors. For each domain, the specific tasks are as follows:
\begin{itemize}
    \item \texttt{walker}: \texttt{flip}, \texttt{run}, \texttt{stand}, and \texttt{walk}.
    \item \texttt{cheetah}: \texttt{run}, \texttt{run backward}, \texttt{walk}, and \texttt{walk backward}.
    \item \texttt{quadruped}: \texttt{jump}, \texttt{run}, \texttt{stand}, and \texttt{walk}.
    \item \texttt{jaco}: \texttt{reach bottom left}, \texttt{reach bottom right}, \texttt{reach top left}, \texttt{reach top right}.
\end{itemize}
For domains \texttt{walker}, \texttt{cheetah}, and \texttt{quadruped}, both the episode length and the maximum return are $1000$. For the domain \texttt{jaco}, both the episode length and the maximum return are $250$. As mentioned in~\citet{yarats2022don}, tasks in \texttt{walker}, \texttt{cheetah}, and \texttt{quadruped} use dense reward functions, while tasks in \texttt{jaco} use sparse reward functions. Detailed descriptions of the benchmark can be found in~\citet{yarats2022don}. Thus, zero-shot adaptation on \texttt{jaco} is more challenging than on other domains.

The OGBench benchmark consists of a diverse set of robotic navigation and manipulation tasks. These tasks are built on top of the MuJoCo simulator~\citep{todorov2012mujoco} and are designed for goal-conditioned control. We select $4$ state-based domains and $3$ image-based domains, each containing $5$ tasks. The goal of these tasks is to either control an Ant to navigate in deterministic or stochastic mazes (\texttt{antmaze large navigate}, \texttt{antmaze teleport navigate}, and \texttt{visual antmaze medium navigate}) or control a robot arm to rearrange various objects (\texttt{cube single play}, \texttt{scene play}, \texttt{visual cube single play}, and \texttt{visual scene play}). For each state-based domain, the specific tasks are:
\begin{itemize}
    \item \texttt{antmaze large navigate}: \texttt{task 1} (bottom left to top right), \texttt{task 2} (center to top left), \texttt{task 3} (center to bottom right), \texttt{task 4} (bottom right to center), and \texttt{task 5} (bottom left to center).
    \item \texttt{antmaze teleport navigate}: \texttt{task 1} (bottom right to top left), \texttt{task 2} (bottom left to top right), \texttt{task 3} (center to top right), \texttt{task 4} (top left to top right), and \texttt{task 5} (center to top left).
    \item \texttt{cube single play}: \texttt{task 1} (pick and place to left), \texttt{task 2} (pick and place to front), \texttt{task 3} (pick and place to back), \texttt{task 4} (pick and place diagonally), and \texttt{task 5} (pick and place off-diagonally).
    \item \texttt{scene play}: \texttt{task 1} (open drawer and window), \texttt{task 2} (close and lock drawer and window), \texttt{task 3} (open drawer, close window, and pick and place cube to right), \texttt{task 4} (put cube in drawer), and \texttt{task 5} (fetch cube from drawer and close window).
\end{itemize}
For each image-based domain, the specific tasks are the same as the state-based variant, except \texttt{visual antmaze medium navigate}. The \texttt{visual antmaze medium navigate} domain uses local third-person image observations as input to algorithms and includes ground and wall colors for agents to infer their locations. This domain contains five tasks: \texttt{task 1} (bottom left to top right), \texttt{task 2} (top left to bottom right), \texttt{task 3} (turn around central corner), \texttt{task 4} (top right to top left), and \texttt{task 5} (bottom right to bottom left). All visual observations are $64 \times 64 \times 3$ RGB images. These tasks are challenging because the agent must reason directly from pixels. For domains involving navigation tasks \texttt{antmaze large navigate} and \texttt{antmaze teleport navigate}, the maximum episode length is $1000$. For tasks in \texttt{cube single play}, the maximum episode length is $200$. For tasks in \texttt{scene play}, the maximum episode length is $750$. These domains are challenging because they all use sparse goal-conditioned reward functions. For other details of the benchmark, please refer to~\citet{park2024ogbench}.

\paragraph{Datasets.} On the ExORL benchmark, following the prior work~\citep{touati2022does, park2024foundation, kim2024unsupervised, zheng2025intention}, we use $5 \times 10^6$ transitions collected by an exploration method (RND~\citep{burda2018exploration}) for unsupervised pre-training, and another $10^5$ transitions collected by the same exploratory policy for zero-shot inference (predicting $z_r$). The zero-shot adaptation datasets will be labeled with task-specific dense rewards, except in \texttt{jaco}, where the reward signals are sparse. See~\citet{yarats2022don} for details of data collection.

On the OGBench benchmark, following the prior work~\citep{zheng2025intention, bagatella2025td}, for both state-based and image-based tasks, we use $10^6$ transitions collected by a non-Markovian expert policy with temporally correlated noise (the \texttt{play} datasets) for unsupervised pre-training, and another $10^5$ transitions collected by the same noisy expert policy for zero-shot inference. Unlike the ExORL benchmark, the zero-shot adaptation datasets will be labeled with \emph{semi-sparse} rewards~\citep{park2025flow}. See~\citet{park2024ogbench} for details of data collection.

\paragraph{Evaluation Protocols.} We compare the performance of one-step FB against the $6$ baselines (Sec.~\ref{appendix:baselines}) by pre-training each method for $10^6$ gradient steps on different domains ($5 \times 10^5$ gradient steps for image-based domains). We simultaneously perform zero-shot inference to measure the performance of each method. During zero-shot adaptation, we relabel the $10^5$ transitions (Appendix~\ref{appendix:envs-and-datasets}) with task-specific rewards and use them to infer the latent variable $z$ among different methods. Note that for OGBench domains, we use the semi-sparse reward instead of a success indicator for zero-shot inference. After inferring the latent variable $z_r$, we fix it inside the latent-conditioned policy $\pi(\ph{a} \mid \ph{s}, z_r)$ and use the policy to do evaluation. On domains from the ExORL benchmark, we measure the undiscounted cumulative return averaged over $50$ episodes. On domains from the OGBench benchmark, we measure the success rate average over $50$ episodes. Following prior practice~\citep{park2025flow, tarasov2023corl}, we do \emph{not} report the best performance during pre-training and instead report the evaluation results averaged over $8 \times 10^5$, $9 \times 10^5$, and $10^6$ gradient steps for state-based domains. For image-based tasks, we report the evaluation results averaged over $4 \times 10^5$, $4.5 \times 10^5$, and $5 \times 10^5$ gradient steps. Following prior work~\citep{park2025flow}, we report means and standard deviations over $8$ random seeds for state-based domains ($4$ seeds for image-based domains).

For offline-to-online fine-tuning, we only compare one-step FB to prior methods on state-based tasks. We first use the zero-shot policy inferred by different methods as the initialization and then fine-tune the policy for $10^6$ environment steps (environment step = gradient step) using TD3~\citep{fujimoto2018addressing}. Note that we do not retain offline data in the online replay buffer because the offline data lacks reward signals. Again, we measure the undiscounted cumulative return for tasks from the ExORL benchmark and measure the success rate for tasks from the OGbench benchmark. We evaluate the performance of the fine-tuned policy every $10^5$ environment steps. For completeness, we show the full learning curves aggregated over $8$ random seeds.

\subsection{Implementations and Hyperparameters}
\label{appendix:implementations-and-hyperparameters}

We compare one-step FB against $6$ baselines, measuring the performance of undiscounted cumulative returns and success rates on downstream tasks. We implement one-step FB and all baselines using JAX~\citep{jax2018github}, adapting the OGBench~\citep{park2024ogbench} codebase. 
Our open-source implementations can be found at \url{https://github.com/chongyi-zheng/onestep-fb}.
All experiments for state-based domains ran on a single A6000 GPU for up to $6$ hours, and all experiments for image-based domains ran on the same type of GPU for up to $16$ hours. 

Following prior work~\citep{tirinzoni2025zeroshot, touati2022does}, we apply two common practices that improve the overall performance of every method. First, the prior measure over latent variables $p_{\gZ}(\ph{z})$ is set to a scaled von Mises-Fisher distribution: we first sample from the standard Gaussian distribution of $d$ dimensions $x \sim \gN(0, I_d)$ and then normalize and rescale the sample to obtain a latent variable $z = \sqrt{d} x / \norm{x}_2$. Second, when sampling latent variables for training, we include both latents from the prior distribution $p_{\gZ}(\ph{z})$ and latents constructed from the current representations. Specifically, for FB and one-step FB, we use the normalized and rescaled variants of backward representations to construct latents. For all other baselines, we use the normalized and rescaled variants of state representations to construct latents. These constructed latents are mixed with latents sampled from the prior distribution with $0.5$ probability to form the final latents for pre-training. In addition to these common practices, each method adopts specific implementation details, which we describe below.

\paragraph{One-step FB.} The one-step FB consists of three main components for unsupervised pre-training: the forward representation $F_\beta$, the backward representation $B_{\beta}$, and the latent-conditioned policy $\pi(\ph{a} \mid \ph{s}, \ph{z})$. We model all of them as multilayer perceptrons (MLPs) and use different architectures for the ExORL domains and the OGBench domains, respectively (See Table~\ref{tab:common-hyperparameters}). We learn the forward-backward representations using the TD one-step FB loss and the orthonormalization regularization (Eq.~\ref{eq:onestep-fb-repr-obj}). Following prior work~\citep{park2024ogbench}, the policy network outputs the mean and the standard deviation of a Gaussian distribution with a \texttt{tanh} transformation to predict actions. We learn this Gaussian policy using the policy loss in Eq.~\ref{eq:onestep-fb-policy-obj}. Following prior work~\citep{tirinzoni2025zeroshot}, before inferring the latent variable $z^{\beta}_r$ using zero-shot transitions (Eq.~\ref{eq:onestep-zero-shot-latent}), we apply softmax weights based on the rewards of these transitions $\{ (s_i, a_i, r_i) \}_{i = 1}^N$, resulting in the following new rewards:
\begin{align}
    \tilde{r}(s_i, a_i) = w_i \cdot r(s_i, a_i), w_i = \frac{\exp \left( \tau_{\text{reward}} \cdot r(s_i, a_i)\right)}{ \sum_{j = 1}^N \exp \left( \tau_{\text{reward}} \cdot r(s_j, a_j) \right)},
    \label{eq:reward-weighting}
\end{align}
where $\tau_\text{reward}$ is the temperature. For representation dimension $d$, we set it to $50$ for ExORL domains and set it to $128$ for OGBench domains. In our initial experiments, we found that one-step FB's performance is sensitive to the behavioral-cloning regularization coefficient $\lambda_{\text{BC}}$, the orthonormalization regularization coefficient $\lambda_{\text{ortho}}$, and the reward weighting temperature $\tau_{\text{reward}}$. We perform hyperparameter sweeps over $\lambda_{\text{BC}} \in \{ 0, 0.03, 0.3, 3, 30 \}$, $\lambda_{\text{ortho}} \in \{ 0, 0.03, 0.3, 1 \}$, and $\tau_{\text{reward}} \in \{ 3, 10, 30 \}$ to select the best values for each domain. We summarize the hyperparameters for one-step FB in Table~\ref{tab:common-hyperparameters} and Table~\ref{tab:hyperparameters}.

\paragraph{FB~\citep{touati2021learning}.} Our FB implementation adapts the implementation from~\citet{tirinzoni2025zeroshot} and is similar to the one-step FB implementation. There are two main differences between the FB implementation and the one-step FB implementation. First, the forward representation network in FB takes in the latent variable $z$ as input. Second, when computing the forward-backward representation loss (Eq.~\ref{eq:td-fb}), the next action $a'$ is sampled from the latent-conditioned policy $\pi(\ph{a'} \mid \ph{s'}, \ph{z})$. During the zero-shot adaptation, similar to one-step FB, we compute the latent variable $z_r$ for a downstream task as in Eq.~\ref{eq:fb-latent-of-reward}. We use the same representation dimension as in the one-step FB implementation for consistency. We also perform hyperparameter sweeps over $\lambda_{\text{BC}} \in \{ 0, 0.03, 0.3, 3, 30 \}$, $\lambda_{\text{ortho}} \in \{ 0, 0.03, 0.3, 1 \}$, and $\tau_{\text{reward}} \in \{ 3, 10, 30 \}$ to select the best values for each domain. See Table~\ref{tab:common-hyperparameters} and Table~\ref{tab:hyperparameters} for details of hyperparameters.

\begin{table}[t]
\caption{\footnotesize Common hyperparameters for one-step FB and prior methods.}
\label{tab:common-hyperparameters}
\begin{center}
\scalebox{0.73}{
\centering
\begin{tabular}{ll}
\toprule
\textbf{Hyperparameter} & \textbf{Value} \\
\midrule
   optimizer & Adam~\citep{kingma2015adam} \\
   batch size & $1024$ on state-based domains, $256$ on image-based domains \\
   learning rate & $1 \times 10^{-4}$\\
   \cmidrule(lr){2-2}
   \multirow{2}{*}{actor MLP hidden layer sizes} & $(1024, 1024, 1024)$ on ExORL domains \\
   & $(512, 512, 512, 512)$ on OGBench domains \\
   \cmidrule(lr){2-2}
   \multirow{2}{*}{value MLP hidden layer sizes} & $(1024, 1024, 1024)$ on ExORL domains \\
   & $(512, 512, 512, 512)$ on OGBench domains \\
   \cmidrule(lr){2-2}
   \multirow{4}{*}{representation MLP hidden layer sizes} & Laplacian, BYOL-$\gamma$, ICVF, and HILP: $(256, 256, 256)$ on ExORL domains \\
   & FB and one-step FB: $(1024, 1024, 1024)$ for forward representations on ExORL domains, \\
   & \hspace{8.5em} $(512, 512)$ for backward representations on ExORL domains \\
   & All methods: $(512, 512, 512, 512)$ on OGBench domains \\
   \cmidrule(lr){2-2}
   MLP layer normalization & No for all methods except TD-JEPA (see the open-source implementation) \\
   \cmidrule(lr){2-2}
   \multirow{2}{*}{MLP activation function} & ReLU~\citep{nair2010rectified} on ExORL domains \\
   & GELU~\citep{hendrycks2023gaussian} on OGBench domains \\
   \cmidrule(lr){2-2}
   discount factor $\gamma$ & $0.98$ on ExORL domains, $0.99$ on OGBench domains \\
   target networks update coefficient & $1 \times 10^{-2}$ on ExORL domains, $5 \times 10^{-3}$ on OGBench domains \\
   \cmidrule(lr){2-2}
   \multirow{2}{*}{representation dimension $d$} & $50$ on ExORL domains and $128$ on OGBench domains \\
   & For representation dimension of TD-JEPA, see Table~\ref{tab:hyperparameters} \\
   \cmidrule(lr){2-2}
   latent mixing probability & $0.5$ \\
   actor tanh transformation & Yes \\
   TD3 action noise distribution & $\texttt{clip}(\gN(0, 0.2^2), -0.2, 0.2)$ \\
   image encoder & small IMPALA encoder~\citep{espeholt2018impala, park2025flow} \\
   image augmentation method & random cropping \\
   image augmentation probability & $0.5$ \\
   image frame stack & $3$ \\
\bottomrule
\end{tabular}
}
\end{center}
\end{table}

\begin{table}[t]
\caption{\footnotesize \textbf{Domain-specific hyperparameters for one-step FB and prior methods.}
Following prior work~\citet{park2025flow}, we tune these hyperparameters for each domain
from ExORL and OGBench benchmarks. ``-'' indicates the hyperparameters do not exist. The complete description of each hyperparameter
can be found in Appendix~\ref{appendix:implementations-and-hyperparameters}.}
\label{tab:hyperparameters}
\begin{center}
\setlength{\tabcolsep}{3pt}
\resizebox{\textwidth}{!}{
    \centering
    \begin{tabular}{l*{17}{c}}
    \toprule
    & \multicolumn{2}{c}{Laplacian}
    & \multicolumn{2}{c}{BYOL-$\gamma$}
    & \multicolumn{3}{c}{TD-JEPA}
    & \multicolumn{2}{c}{ICVF}
    & \multicolumn{2}{c}{HILP}
    & \multicolumn{3}{c}{FB}
    & \multicolumn{3}{c}{One-Step FB} \\
    \cmidrule(lr){2-3}
    \cmidrule(lr){4-5}
    \cmidrule(lr){6-8}
    \cmidrule(lr){9-10}
    \cmidrule(lr){11-12}
    \cmidrule(lr){13-15}
    \cmidrule(lr){16-18}
    Domain
    & $\lambda_{\text{BC}}$ & $\lambda_{\text{ortho}}$
    & $\lambda_{\text{BC}}$ & $\lambda_{\text{ortho}}$
    & $\lambda_{\text{BC}}$ & $\lambda_{\text{ortho}}$ & $d$
    & $\lambda_{\text{BC}}$ & $\mu$
    & $\lambda_{\text{BC}}$ & $\mu$
    & $\lambda_{\text{BC}}$ & $\lambda_{\text{ortho}}$ & $\tau_{\text{reward}}$
    & $\lambda_{\text{BC}}$ & $\lambda_{\text{ortho}}$ & $\tau_{\text{reward}}$ \\
    \midrule
    \texttt{walker}
    & $30$ & $1$
    & $10$ & $0.01$
    & $0$ & $0$ & $50$
    & $0.03$ & $0.5$
    & $0.3$ & $0.9$
    & $0$ & $0.03$ & $10$
    & $0$ & $0.1$ & $10$\\

    \texttt{cheetah}
    & $10$ & $1$
    & $3$ & $0.01$
    & $0.3$ & $1.0$ & $50$
    & $0.3$ & $0.5$
    & $3$ & $0.5$
    & $0$ & $0.3$ & $10$
    & $0$ & $1$ & $3$\\

    \texttt{quadruped}
    & $10$ & $1$
    & $3$ & $0.01$
    & $0$ & $0$ & $50$
    & $0.03$ & $0.5$
    & $3$ & $0.9$
    & $0$ & $1$ & $10$
    & $0$ & $0.03$ & $3$\\

    \texttt{jaco}
    & $30$ & $1$
    & $10$ & $1$
    & $0$ & $0.1$ & $50$
    & $0.03$ & $0.5$
    & $0.3$ & $0.9$
    & $0$ & $0.3$ & $10$
    & $0$ & $0.03$ & $3$\\

    \midrule
    \texttt{antmaze large navigate}
    & $1$ & $0.1$
    & $3$ & $0$
    & $0.03$ & $0.1$ & $128$
    & $3$ & $0.5$
    & $1$ & $0.5$
    & $0.03$ & $0$ & $10$
    & $0.03$ & $0$ & $10$\\

    \texttt{antmaze teleport navigate}
    & $30$ & $0$
    & $1$ & $0$
    & $0.3$ & $0.1$ & $50$
    & $0.3$ & $0.5$
    & $1$ & $0.5$
    & $0.03$ & $0$ & $3$
    & $0.1$ & $0$ & $10$\\

    \texttt{cube single play}
    & $1$ & $0$
    & $30$ & $0$
    & $0.3$ & $0.1$ & $50$
    & $30$ & $0.5$
    & $1$ & $0.5$
    & $0.3$ & $0$ & $10$
    & $0.3$ & $0.3$ & $300$\\

    \texttt{scene play}
    & $1$ & $0$
    & $3$ & $0$
    & $3.0$ & $0$ & $50$
    & $0.3$ & $0.5$
    & $1$ & $0.9$
    & $0.3$ & $0$ & $3$
    & $0.3$ & $0$ & $300$\\

    \midrule
    \texttt{visual cube single play}
    & - & -
    & $30$ & $0$
    & $0.3$ & $1.0$ & $50$
    & - & -
    & $1$ & $0.5$
    & $1$ & $0$ & $300$
    & $1$ & $0$ & $300$ \\

    \texttt{visual scene play}
    & - & -
    & $3$ & $0$
    & $0.1$ & $1.0$ & $50$
    & - & -
    & $3$ & $0.5$
    & $0.3$ & $0$ & $300$
    & $0.3$ & $0$ & $10$ \\
    \bottomrule
    \end{tabular}
}
\end{center}
\end{table}

Besides the FB algorithm, we also compare one-step FB against unsupervised pre-training methods for RL that first learn state representations $\phi: \gS \to \gZ$ and then use off-the-shelf RL algorithms to maximize intrinsic rewards derived from the state representations $r(s, z) = \phi(s)^{\top} z$. During the zero-shot adaptation, all methods (except one-step FB and FB) find the appropriate latent variable $z_r$ by solving a simple linear regression problem~\citep{barreto2017successor, park2024foundation}:
\begin{align*}
    z_r = \argmin_{z \in \gZ} \E_{ \rho(s, a) } \left[ \left( r(s, a) - \phi(s)^{\top} z \right)^2 \right],
\end{align*}
where $\rho \in \Delta(\gS \times \gA)$ is the marginal measure over states and actions as in one-step FB. For HILP and ICVF, both methods learn state representations using an expectile regression loss adapted from IQL~\citep{kostrikov2021offline}. They differ in how they decompose the successor measure. For BYOL-$\gamma$ and Laplacian, both methods learn state representations via consistency over successor measures (latent predictive loss). They differ in that the consistency is either along single-step or multi-step samples from the successor measure. After learning state representations, we use the same TD3 + BC~\citep{fujimoto2021minimalist} algorithm to maximize the intrinsic reward for all these baselines. The TD3 + BC implementation uses a target actor to select actions in the critic loss. We also add a clipped Gaussian noise $\texttt{clip}(\gN(0, 0.2^2), -0.2, 0.2)$ to introduce some noise into these actions. Similar to Eq.~\ref{eq:onestep-fb-policy-obj}, the actor loss maximizes Q predicted by the critic while being regularized to output the behavioral actions via a behavioral-cloning regularization. Below, we describe the details of each method.

\paragraph{HILP~\citep{park2024foundation}.} The HILP implementation is adapted from the official implementation~\citep{park2024foundation}. The motivation of representation learning in HILP is to encode the temporal (goal-conditioned) distance between pairs of states into the Euclidean distance in a $d$-dimensional representation space. To achieve this goal, the HILP learns state representations $\phi(\ph{s})$ using the following expectile loss:
\begin{align*}
    \gL_{\text{HILP}}(\phi) = \E_{p^{\pi_{\beta}}(s, s'), p_{\gG}(s_f \mid s)} \left[ L_2^{\mu} \left(-\mathbbm{1}(s \neq s_f) + \gamma \norm{\bar{\phi}(s_f) - \bar{\phi}(s')}_2 + \norm{ \phi(s_f) - \phi(s) }_2 \right) \right],
\end{align*}
where $L_2^{\mu}(x) = \abs{ \mu - \mathbbm{1}(x < 0) } x^2$ is the expectile loss with $\mu \in [0.5, 1)$, $\bar{\phi}$ is the target state representation. The future state $s_f$ is sampled from either the behavioral successor measure with probability $0.625$ or a random uniform measure with probability $0.375$, which we denote as the goal measure $p_\gG(s_f \mid s)$. For the representation dimension, we set $d = 50$ for ExORL domains and set $d = 128$ for OGBench domains. We sweep over the behavioral-cloning regularization coefficient $\lambda_{\text{BC}} \in \{ 0.03, 0.3, 3 \}$ and the expectile $\mu \in \{ 0.5, 0.9 \}$ for different domains. See Table~\ref{tab:common-hyperparameters} and Table~\ref{tab:hyperparameters} for details of hyperparameters.

\paragraph{ICVF~\citep{ghosh2023reinforcement}.} The ICVF learns state representations similar to HILP, but uses a different decomposition of the intention-conditioned successor measure. Specifically, ICVF uses an intention-conditioned value $V: \gS \times \gS \times \gS \to \R$ to model the successor measure of visiting the future state $s_f$ starting from the current state $s$ by following the intention goal $g$. Both the future state $s_f$ and the intention goal $g$ are sampled from the goal measure similar to HILP. The intention-conditioned value decomposes as $V(\ph{s}, \ph{s_f}, \ph{g}) = \phi(\ph{s})^{\top} T_{\text{ICVF}}(\ph{g}) \psi(\ph{s_f})$, where $T_{\text{ICVF}}: \gS \to \gZ \times \gZ$ predicts a latent transition matrix and $\psi: \gS \to \gZ$ is the future state representations. ICVF learns the intention-conditioned value using the following variant of the expectile loss:
{\footnotesize \begin{align*}
    \gL_{\text{ICVF}}(\phi, T_{\text{ICVF}}, \psi) = \E_{ p^{\pi_{\beta}}(s, s'), p_{\gG}(s_f \mid s), p_{\gG}(g \mid s) } \left[ \abs{ \mu - \mathbbm{1} \left( A(s, s', g) < 0 \right)} \left( V(s, s_f, g) - \mathbbm{1}(s = s_f) - \gamma V(s', s_f, g) \right)^2 \right],
\end{align*}}where $\mu \in [0.5, 1)$ and the advantage $A(s, s', g)$ is defined as
\begin{align*}
    A(s, s', g) = \mathbbm{1}(s = g) + \gamma V(s', g, g) - V(s, g, g).
\end{align*}
For the representation dimension, we use the same values as in HILP for consistency. We sweep over the behavioral-cloning regularization coefficient $\lambda_{\text{BC}} \in \{ 0.03, 0.3, 3\}$ and the expectile $\mu \in \{ 0.5, 0.7, 0.9 \}$ for different domains. See Table~\ref{tab:common-hyperparameters} and Table~\ref{tab:hyperparameters} for details of hyperparameters.

\paragraph{BYOL-$\gamma$~\citep{lawson2025self}.} BYOL-$\gamma$ learns state representations via consistency (or a latent predictive loss) over the current state and a future state sampled from the behavioral occupancy measure. Specifically, BYOL-$\gamma$ minimizes the mean squared error between state representations of the current state $s$ and the future state $s_f$ with a latent transition function $T_{\text{BYOL-$\gamma$}}: \gZ \times \gA \to \gZ$,
\begin{align*}
    \gL_{\text{BYOL-$\gamma$}}(\phi, T_{\text{BYOL-$\gamma$}}) = \E_{ p^{\pi_{\beta}}(s, a), M^{\pi_{\beta}}(s_f \mid s, a) } \left[ \norm{ T_{\text{BYOL-$\gamma$}}(\phi(s), a) - \bar{\phi}(s_f) }_2^2 \right],
\end{align*}
where $\bar{\phi}$ is the target state representation, and $M^{\pi_{\beta}}(\ph{s_f} \mid \ph{s}, \ph{a})$ is the behavioral successor measure following a geometric distribution. We also include an orthonormalization regularization loss as in Eq.~\ref{eq:onestep-fb-ortho-reg} to regularize the covariance of state representations. For the representation dimension, we follow the same values as in HILP. In our experiments, we sweep over the behavioral-cloning regularization coefficient $\lambda_{\text{BC}} \in \{ 0.03, 0.3, 3, 30 \}$ and the orthonormalization regularization coefficient $\lambda_{\text{ortho}} \in \{ 0.01, 0.1, 1.0 \}$ to find the best values for different domains. We include complete hyperparameters in Table~\ref{tab:common-hyperparameters} and Table~\ref{tab:hyperparameters}.

\paragraph{Laplacian~\citep{wu2018laplacian}.} Unlike BYOL-$\gamma$, Laplacian learns state representations via consistency over the current state and the immediate next state. Specifically, the Laplacian minimizes the mean squared error between state representations of the current state $s$ and the next state $s'$:
\begin{align*}
    \gL_{\text{Laplacian}}(\phi) = \E_{ p^{\pi_{\beta}}(s, s') } \left[ \norm{ \phi(s) - \phi(s') }_2^2 \right].
\end{align*}
Note that we do not use a latent transition function or a target state representation in this loss. Although state representations collapsed to a constant admitting a minimizer of this loss function, we do not observe this behavior in our experiments. Again, we include an orthonormalization regularization loss as in Eq.~\ref{eq:onestep-fb-ortho-reg} to regularize the covariance of state representations towards the identity matrix. For the representation dimension, we use the same values as in HILP. In our experiments, we sweep over the behavioral-cloning regularization coefficient $\lambda_{\text{BC}} \in \{ 0.3, 1, 3, 10, 30 \}$ and the orthonormalization regularization coefficient $\lambda_{\text{ortho}} \in \{0.01, 0.1, 1.0 \}$ to find the best values for different domains. We include complete hyperparameters in Table~\ref{tab:common-hyperparameters} and Table~\ref{tab:hyperparameters}.

\paragraph{TD-JEPA~\citep{bagatella2025td}.}
TD-JEPA is a latent-predictive zero-shot RL method that learns a state encoder
$\phi:\gS \to \R^{d_\phi}$, a task encoder $\psi:\gS \to \R^{d_\psi}$, 
policy-conditioned predictors $T_\phi:\R^{d_\phi}\times\gA\times\gZ\to\R^{d_\psi}$ and 
$T_\psi:\R^{d_\psi}\times\gA\times\gZ\to\R^{d_\phi}$, and a latent-conditioned policy
$\pi(\ph{a}\mid \phi(\ph{s}), \ph{z})$. Intuitively, the predictor $T_\phi$ estimates the
successor features of the latent-conditioned policy $\pi_z$ in the task-representation space:
\begin{align*}
    T_\phi(\phi(s), a, z) 
    \approx \E_{s_f \sim M^{\pi_z}(\cdot \mid s, a)}
    \left[\psi(s_f)\right].
\end{align*}
TD-JEPA uses a
temporal-difference latent-predictive loss to train the state encoder $\phi$, task encoder $\psi$, and the policy-conditioned predictor $T_\phi$:
\begin{align*}
    \gL_{\text{TD-JEPA}}(\phi, T_\phi, \psi)
    =
    \E_{\substack{(s,a,s')\sim \gD,\, z\sim p_{\gZ}(z),\\
    a'\sim \pi(\cdot\mid \bar{\phi}(s'),z)}}
    \left[
    \norm*{
    T_\phi(\phi(s), a, z)
    - \bar{\psi}(s')
    - \gamma \bar{T}_\phi(\bar{\phi}(s'), a', z)
    }_2^2
    \right],
\end{align*}
where $\bar{\phi}$, $\bar{\psi}$, and $\bar{T}_\phi$ denote target networks. Additionally, this method also trains the
task encoder $\psi$ through the symmetric objective $\gL_{\text{TD-JEPA}}(\psi, T_\psi, \phi)$,
obtained by swapping the roles of $\phi$ and $\psi$:
\begin{align*}
    \gL_{\text{TD-JEPA}}(\psi, T_\psi, \phi)
    =
    \E_{\substack{(s,a,s')\sim \gD,\, z\sim p_{\gZ}(z),\\
    a'\sim \pi(\cdot\mid \bar{\phi}(s'),z)}}
    \left[
    \norm*{
    T_\psi(\psi(s), a, z)
    - \bar{\phi}(s')
    - \gamma \bar{T}_\psi(\bar{\psi}(s'), a', z)
    }_2^2
    \right].
\end{align*}
Following TD-JEPA, we regularize the learned encoders with an orthonormality regularization similar to Eq~\ref{eq:onestep-fb-ortho-reg} in one-step FB (see Algorithm 1 in~\citep{bagatella2025td} for details).
The representation objective is therefore
\begin{align*}
    \gL(\phi,T_\phi,\psi,T_\psi)
    =
    \gL_{\text{TD-JEPA}}(\phi,T_\phi,\psi)
    +
    \gL_{\text{TD-JEPA}}(\psi,T_\psi,\phi)
    +
    \lambda_{\text{ortho}}
    \left(
    \widehat{\gL}_{\text{ortho}}(\phi)
    +
    \widehat{\gL}_{\text{ortho}}(\psi)
    \right).
\end{align*}

In our experiments, we tune the behavioral-cloning regularization coefficient $\lambda_{\text{BC}} \in \{0, 0.03, 0.1, 0.3, 1, 3 \}$ and the orthonormalization regularization coefficient $\lambda_{\text{ortho}} \in \{0, 0.01, 0.1, 1.0 \}$ to find the best values for different domains. We report the best selected hyperparameters in Table~\ref{tab:hyperparameters}.

\section{Additional Experiments}

\subsection{Does One-Step FB Enable Efficient Fine-Tuning?}
\label{appendix:fine-tuning-experiments}

\begin{figure}[t]
    \begin{minipage}[c]{0.49\textwidth}
        \centering
        \includegraphics[width=\linewidth]{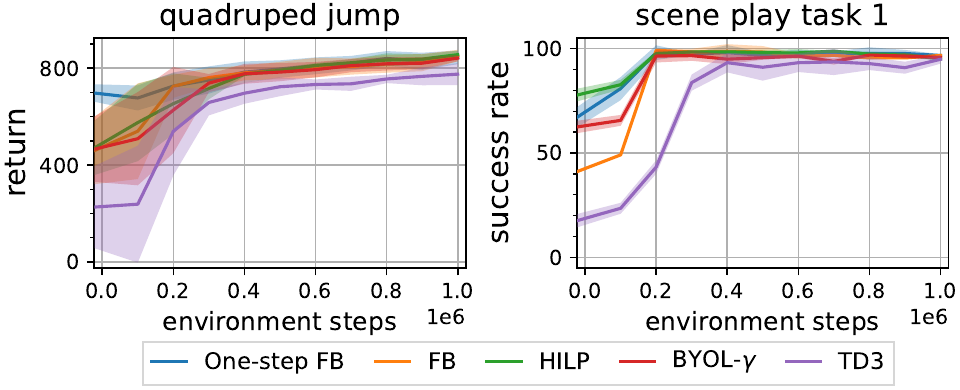}
        \caption{\textbf{Fine-tuning pre-trained agents on downstream tasks.} After offline pre-training, we conduct online fine-tuning on various methods using the same off-the-shelf RL algorithm (TD3). One-step FB continues to provide higher sample efficiency ($+40\%$ on average) during fine-tuning, as compared with the original FB method.
        }
        \label{fig:finetuning}
    \end{minipage}
    \hfill
    \begin{minipage}[c]{0.49\textwidth}
        \centering
        \includegraphics[width=\linewidth]{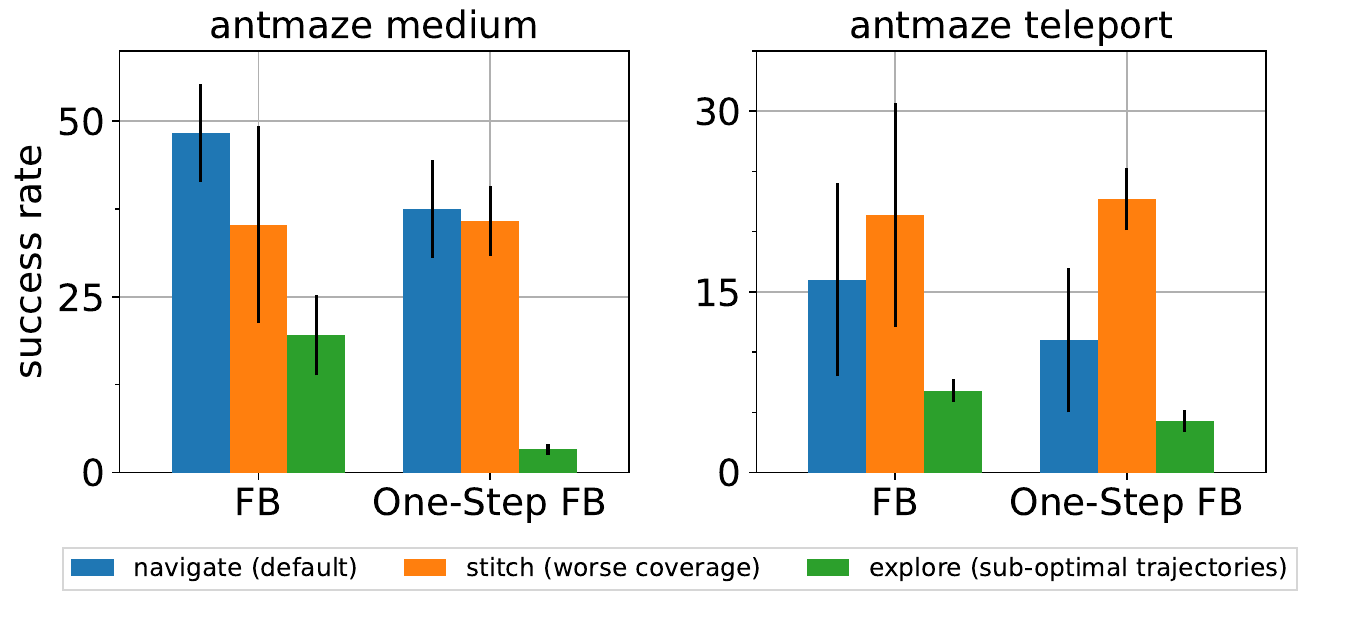}
        \caption{\textbf{The effects of the dataset quality on FB and one-step FB.} Using \texttt{stitch} dataset has minor effects on both FB and one-step FB, and sometimes even improves the performance, while using \texttt{explore} datasets reduces performance for both methods.
        }
        \label{fig:dataset-type-ablation}
    \end{minipage}
\end{figure}

While our prior experiments in Sec.~\ref{subsec:offline-rl-experiments} focus on zero-shot performance, we also want to study whether one-step FB enables efficient online fine-tuning. To test this hypothesis, we conduct offline-to-online experiments on one task in the ExORL benchmarks (\texttt{quadruped jump}) and another task taken from the OGBench benchmarks (\texttt{scene task 1}). Following the evaluation protocols in Appendix~\ref{appendix:envs-and-datasets}, we first use $10^5$ zero-shot transitions to derive the latent variable $z_r$ for the policy $\pi(\ph{a} \mid \ph{s}, z_r)$, and then use it to initialize a TD3~\citep{fujimoto2018addressing} agent for fine-tuning. We selectively compare one-step FB against $3$ baselines in our offline experiments: BYOL-$\gamma$, HILP, and FB, and also include the performance of a TD3 agent trained from scratch for reference.

As shown in Fig.~\ref{fig:finetuning}, the policy derived from one-step FB achieves strong fine-tuning performance compared to alternative unsupervised pre-training, with $40\%$ higher sample efficiency on average. Compared with the TD3 variant trained from scratch, the sample efficiency of fine-tuning the policy derived from one-step FB is $+2.25 \times$ higher, suggesting the importance of the unsupervised pre-training phase. We also observe that the fine-tuned policies reach the asymptotic performance of TD3 at the end of training. Interestingly, we do \emph{not} observe a degradation in performance at the beginning of fine-tuning; an observation that has been identified in prior offline-to-online methods~\citep{nakamoto2023calql}. In particular, we do not retain the pre-training data when performing online fine-tuning~\citep{zhou2025efficient}, helping to explain this observation. Taken together, one-step FB is a simpler method that provides benefits for both offline unsupervised pre-training and online fine-tuning.

\subsection{The Effects of the Dataset Quality on FB and One-Step FB}
\label{appendix:dataset-effects}

Since one-step FB performs one step of policy improvement over the behavioral datasets, it is important to examine the effects of dataset quality on our algorithm. We hypothesize that our method can be effective when \emph{(1)} the dataset has good coverage of the state and action spaces, and \emph{(2)} the dataset contains nearly optimal trajectories.

To test these hypotheses, we conduct ablations on two types of OGBench datasets: \texttt{stitch} (worse coverage on low-reward regions) and \texttt{explore} (highly sub-optimal trajectories). Specifically, we select two domains from the OGBench benchmarks, $\texttt{antmaze medium}$ and $\texttt{antmaze teleport}$, and compare the zero-shot performances of both FB and one-step FB on three different types of datasets: \texttt{navigate} (default), \texttt{stitch}, and \texttt{explore}.

Results in Fig.~\ref{fig:dataset-type-ablation} suggest that using the \texttt{stitch} dataset has minor effects on both FB and one-step FB, and sometimes even improves the performance, while using \texttt{explore} datasets reduces performance for both methods. We conjecture that worse coverage may not be a main issue when the dataset contains high-reward transitions. In contrast, highly sub-optimal trajectories introduce noise when inferring the latent variable $z$ during policy adaptation, which significantly lowers the zero-shot performance.

\subsection{Confounding Effects in the Didactic Experiments}
\label{appendix:confounding}

In this section, we provide a more detailed discussion of potential confounding effects in our didactic experiments. The aim is to clarify the observation that the practical FB algorithm fails to converge, while our one-step FB algorithm converges to its stated fixed-point.

\paragraph{Learning rate.} \emph{Does the learning rate affect the convergence of FB?} We ablate over different learning rates within $\{10^{-2}, 10^{-3}, 10^{-4}, 10^{-5}, 10^{-6} \}$, choosing a wider range than the learning rate $10^{-4}$ used in Sec.~\ref{subsec:fb-didactic-exp} and Sec.~\ref{subsec:onestep-fb-didactic-exp}. As seen in Fig.~\ref{fig:didactic_lr_sweep}, using learning rates either larger than or smaller than $10^{-4}$ results in a higher $\hat{Q}$ equivariance error ($\epsilon_\text{equiv} > 10^{-4}$), indicating that the original FB algorithm still does \emph{not} converge to the ground-truth fixed point. These observations are consistent with our conclusions in Sec.~\ref{subsec:fb-didactic-exp}.

\begin{figure}[t]
    \vspace{-0.5em}
    \centering
    \includegraphics[width=1\linewidth]{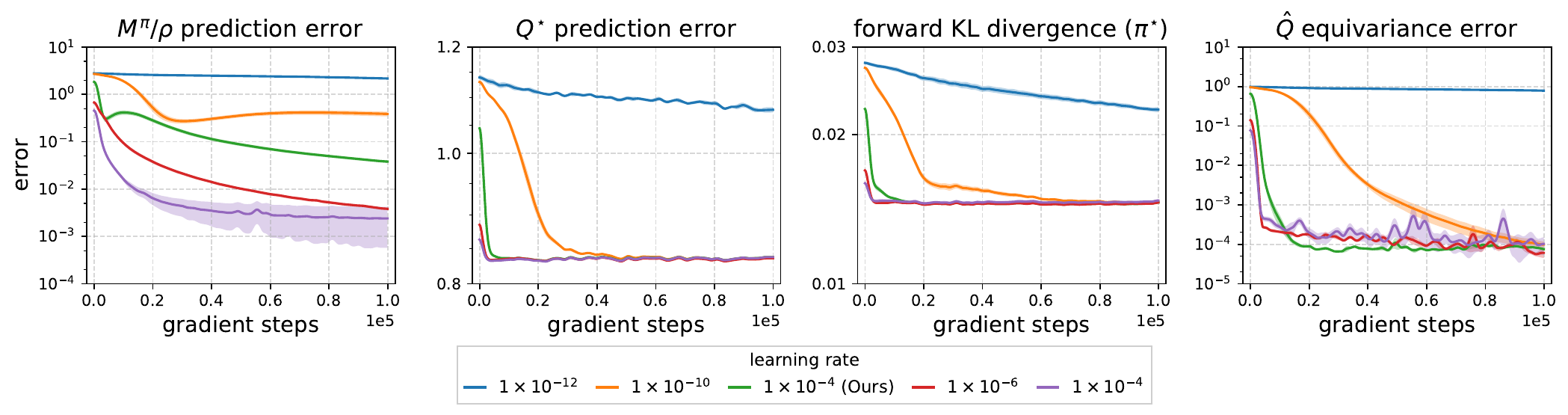}
    \caption{\textbf{Learning rate ablations for FB on the three-state CMP.} We conduct ablations to study the effect of learning rate on FB. Using learning rates other than $10^{-4}$ results in higher equivariance error of Q-value prediction ($\epsilon_\text{equiv} > 10^{-4}$). Thus, the failure mode of FB is \emph{not} explained by the choice of learning rate.} %
    \label{fig:didactic_lr_sweep}
    \vspace{-1.25em}
\end{figure}

\paragraph{Policy temperature.} \emph{Does the policy temperature $\tau_{\text{policy}}$ affect the convergence of FB?} 
As mentioned in Appendix~\ref{appendix:didactic-exp-fb}, we use a softmax policy (Eq.~\ref{eq:fb-didactic-exp-softmax-policy}) with temperature $\tau_{\text{policy}}$ to approximate the greedy policy with respect to the inner product $F(\ph{s}, \ph{a}, \ph{z})^{\top} \ph{z}$. One confounding factor is the choice of the temperature $\tau_{\text{policy}}$. We conduct ablation experiments studying the effects of the policy temperature $\tau_{\text{policy}}$ on the convergence of FB using the same three-state CMP as in Sec.~\ref{subsec:fb-didactic-exp}. Specifically, we choose to sweep over $\tau_{\text{policy}} \in \{ 1, 10^{-6}, 10^{-3}, 1, 10, 200 \}$, showing results in Fig.~\ref{fig:didactic_softmax_temp_sweep}. We observe that a decreasing $\tau_{\text{policy}}$ results in an increasing Q prediction equivariance error $\epsilon_{\text{equiv}}$, suggesting that $\tau_{\text{softmax}}$ is an important hyperparameter balancing learnability and convergent accuracy. However, FB still fails to converge to a ground-truth fixed point.

\begin{figure}[t]
    \centering
    \includegraphics[width=0.99\linewidth]{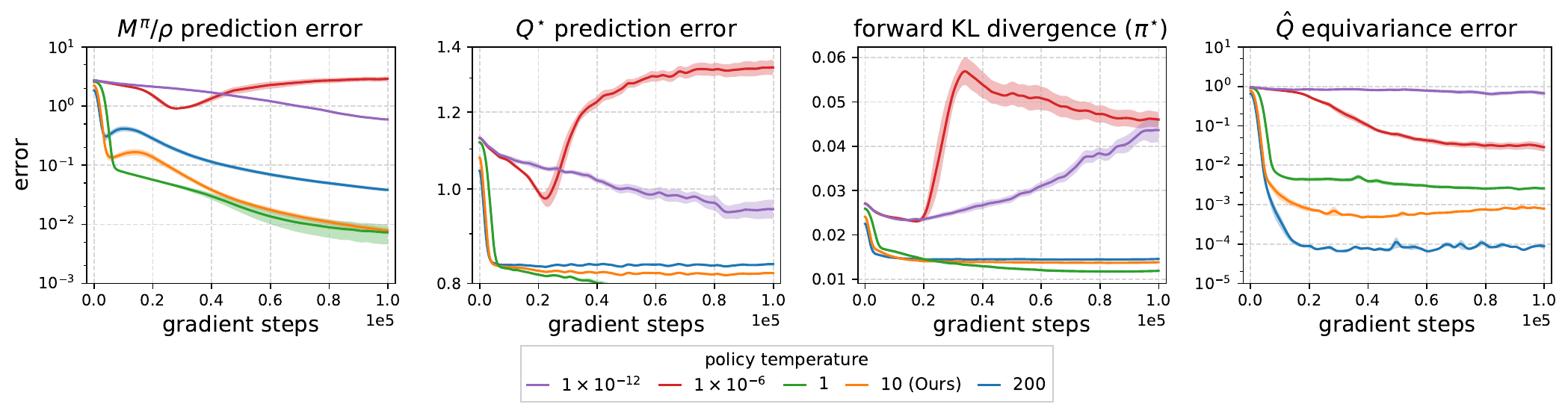}
    \caption{\textbf{Policy temperature ablations for FB on the three-state CMP.} We study the effect of policy temperature $\tau_{\text{policy}}$ on the convergence of FB: a decreasing $\tau_{\text{policy}}$ results in an increasing Q prediction equivariance error $\epsilon_{\text{equiv}}$, suggesting that FB still fails to converge to the ground-truth fixed point. 
    }
    \label{fig:didactic_softmax_temp_sweep}
    \vspace{-1em}
\end{figure}

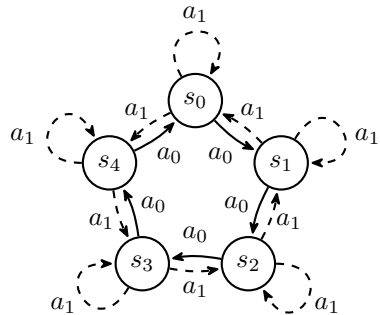
\begin{wrapfigure}[19]{R}{0.5\textwidth}
    \vspace{-2.5em}
    \centering
    \begin{tikzpicture} 
        \foreach \i/\angle in {0/90, 1/18, 2/-54, 3/-126, 4/162} {
            \node[state_style] (s\i) at (\angle:1.2) {$s_\i$};
        }
    
        \foreach \i/\j in {0/1, 1/2, 2/3, 3/4, 4/0} {
            \path[edge_style] (s\i) edge[bend right=10] (s\j);
        
            \coordinate (m) at ($(s\i)!0.5!(s\j)$);
        
            \node[font=\sffamily, inner sep=1pt]
                at ($(0,0)!0.6!(m)$) {$a_0$};
        }

        \foreach \i/\j in {0/1, 1/2, 2/3, 3/4, 4/0} {
            \path[stoch_style] (s\j) edge[bend right=10] (s\i);

            \coordinate (m) at ($(s\j)!0.5!(s\i)$);

            \node[font=\sffamily, inner sep=1pt]
                at ($(0,0)!1.35!(m)$) {$a_1$};
        }
    
        \foreach \i/\angle in {0/90, 1/18, 2/-54, 3/-126, 4/162} {
            \ifnum\i=0
                \path[stoch_style] (s\i) edge[loop above, loop_style, in=60, out=120] (s\i);
                \node[font=\sffamily] at ($(s\i)+(0,1.15)$) {$a_1$};
            \fi
            \ifnum\i=1
                \path[stoch_style] (s\i) edge[loop right, loop_style, in=0, out=60] (s\i);
                \node[font=\sffamily] at ($(s\i)+(1.15,0.35)$) {$a_1$};
            \fi
            \ifnum\i=2
                \path[stoch_style] (s\i) edge[loop below, loop_style, in=-60, out=0] (s\i);
                \node[font=\sffamily] at ($(s\i)+(1.05,-0.55)$) {$a_1$};
            \fi
            \ifnum\i=3
                \path[stoch_style] (s\i) edge[loop left, loop_style, in=-180, out=-120] (s\i);
                \node[font=\sffamily] at ($(s\i)+(-1.05,-0.55)$) {$a_1$};
            \fi
            \ifnum\i=4
                \path[stoch_style] (s\i) edge[loop above, loop_style, in=-240, out=-180] (s\i);
                \node[font=\sffamily] at ($(s\i)+(-1.15,0.35)$) {$a_1$};
            \fi
        }
    \end{tikzpicture}
    \caption{\footnotesize \textbf{The five-state circular CMP.} Agents start from state $s_0$ and take action $a_i$ ($i = 0, 1$). At every state $s_i$, choosing $a_0$ convergence transits to the next state $s_{(i + 1) \mod 5}$, forming circular transitions. At every state $s_i$, choosing action $a_1$ transits to state $s_{(i - 1) \mod 5}$ with a probability of $0.7$ and stays in the same state with a probability of $0.3$, forming the stochastic transitions. Appendix~\ref{appendix:confounding} uses this simple CMP to study the convergence of the FB and the one-step FB algorithms.}
    \label{fig:didactic-five-state-circular-cmp}
\end{wrapfigure}

\paragraph{Representation parameterization.} \emph{Does a different parameterization of the FB representations affect its convergence?} As mentioned in Appendix~\ref{appendix:didactic-exp-fb}, we used differentiable matrices $F_z \in \R^{\abs{\gS \times \gA} \times d}$ and $B \in \R^{\abs{\gS \times \gA} \times d}$ to represent the FB representations for a latent variable $z$. This parameterization simplifies the optimization procedure. To study the effects of this parameterization, we conduct ablation experiments using the same three-state CMP as in Sec.~\ref{subsec:fb-didactic-exp}. We choose to compare the differentiable matrix parameterization against a monolithic feed-forward neural network (MLP) parameterization. Results in Fig.~\ref{fig:didactic_repr_parameterization_ablation} show that using differentiable matrices yields lower error in successor measure ratio predictions, while using MLP parameterizations results in lower equivalence errors in the Q-value prediction. However, the similar failure mode for FB's convergence persists.

\begin{figure}[t]
    \centering
    \includegraphics[width=0.99\linewidth]{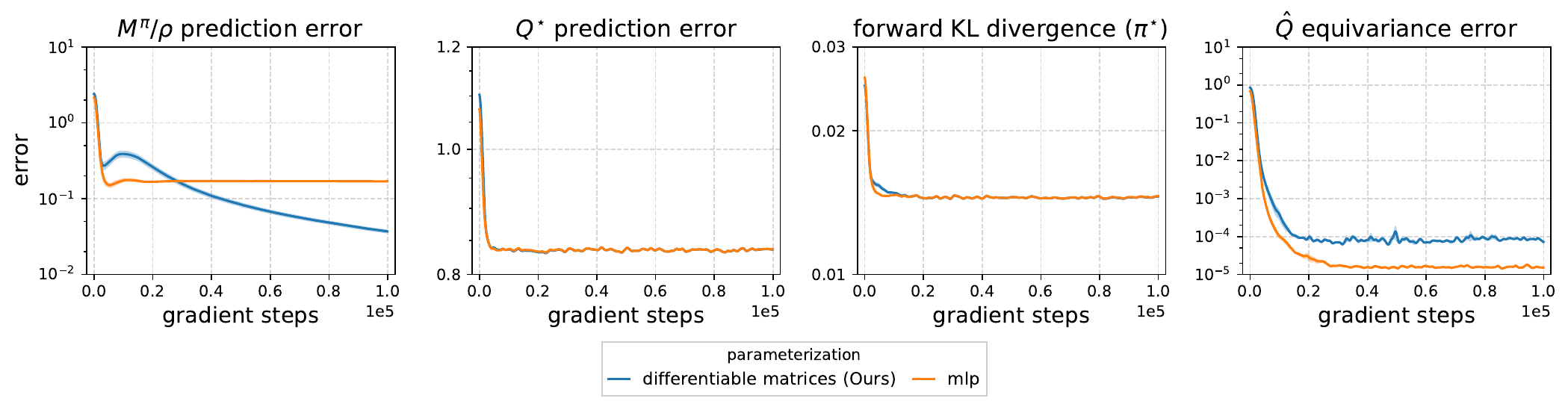}
    \caption{\textbf{Representation parameterization ablations for FB on the three-state CMP.} We study the effect of representation parameterization on the convergence of FB: using differentiable matrices yields lower error in successor measure ratio predictions, while using MLP parameterizations results in lower equivalence errors in the Q-value prediction. However, the similar failure mode for FB's convergence persists.
    }
    \label{fig:didactic_repr_parameterization_ablation}
    \vspace{-1em}
\end{figure}

\paragraph{Representation learning objective.} \emph{Does directly optimizing the TD LSIF loss (Eq.~\ref{eq:td-fb}) with a target network help FB converge?} Since our didactic experiments use a CMP with a discrete state and action space, we choose to compute the successor measure ratio $M^{\pi}_{\gZ} \text{diag}(\rho)^{-1}$ analytically, and fit the FB representations $F_{\gZ}$ and $B$ using the mean squared error (See Appendix~\ref{appendix:didactic-exp-fb}). However, the practical FB algorithm optimizes the TD LSIF loss (Eq.~\ref{eq:td-fb}) directly without computing the successor measure ratio (impossible to compute for continuous CMPs). We conduct ablation experiments comparing the effects of using these two objectives for learning FB representations on FB's convergence. Results on the same three-state CMP (Fig.~\ref{fig:didactic_repr_obj_ablation}) show that while analytically computing the successor measure ratio consistently achieved lower errors on different metrics than the TD LSIF loss. The final learned FB representations still failed to converge to the ground-truth FB representations.

\begin{figure}[t]
    \centering
    \includegraphics[width=0.99\linewidth]{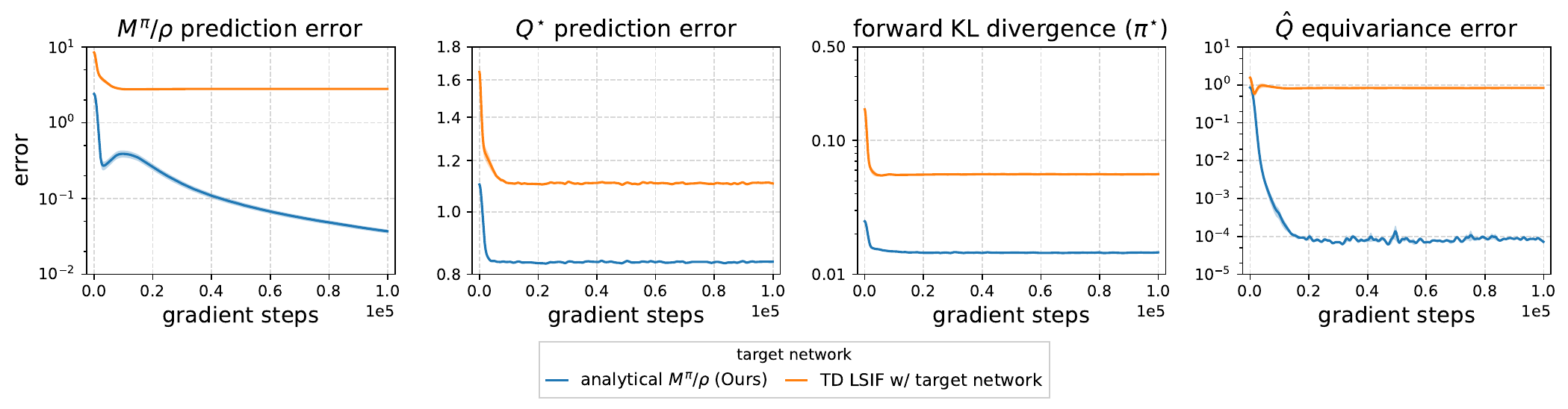}
    \caption{\textbf{Representation learning objective ablations for FB on the three-state CMP.} We conduct ablations to compare two representation learning objectives in our didactic experiments: \emph{(1)} analytically compute the successor measure ratio and conduct bilinear decomposition into FB representations, and \emph{(2)} learn the FB representations using the TD LSIF loss (Eq.~\ref{eq:td-fb}) directly. While the former learning objective consistently achieved lower errors on different metrics than the latter objective. The final learned FB representations still failed to converge to the ground-truth FB representations.
    }
    \label{fig:didactic_repr_obj_ablation}
    \vspace{-1em}
\end{figure}

\paragraph{Different discrete CMPs.} \emph{Is the three-state CMP a special case where FB fails to converge?} To rule out the confounding factors originating from the choice of CMPs, we conduct additional didactic experiments on a new CMP, similar to Sec.~\ref{subsec:fb-didactic-exp} and Sec.~\ref{subsec:onestep-fb-didactic-exp}. Specifically, we construct a five-state circular CMP (Fig.~\ref{fig:didactic-five-state-circular-cmp}), where agents start from state $s_0$ and take action $a_i$ ($i = 0, 1$). At every state $s_i$, choosing $a_0$ convergence transits to the next state $s_{(i + 1) \mod 5}$, forming circular transitions. At every state $s_i$, choosing action $a_1$ transits to state $s_{(i - 1) \mod 5}$ with a probability of $0.7$ and stays in the same state with a probability of $0.3$, forming the stochastic transitions. 

As shown in Fig.~\ref{fig:didactic-exp-five-state-circular-cmp}, we track the important statistics for both FB and one-step FB, similar to Sec.~\ref{subsec:fb-didactic-exp} and Sec.~\ref{subsec:onestep-fb-didactic-exp}. Importantly, in this new five-state circular CMP, we observe convergences similar to those in the three-state CMP for both methods. These results also highlight that FB fails to converge to a pair of ground-truth FB representations, while one-step FB exactly fits the ground-truth one-step FB representations within $4 \times {10}^4$ gradient steps. Taken together, our conclusions are consistent across different didactic CMPs.

\begin{figure}[t]
    \centering
    \includegraphics[width=.7\linewidth]{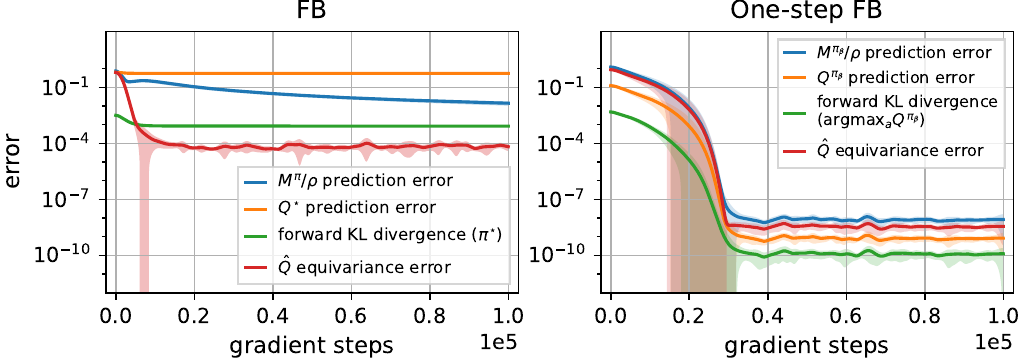}
    \caption{\footnotesize \textbf{Learning FB representations in the five-state circular CMP.} \figleft \, After training for $10^5$ gradient steps, FB fails to converge to a pair of ground-truth FB representations. \figright \, Given a fixed policy, one-step FB exactly fits the ground-truth one-step FB representations within $4 \times 10^4$ gradient steps. These observations are consistent with our analysis on the three-state CMP (Sec.~\ref{subsec:fb-didactic-exp} and Sec.~\ref{subsec:onestep-fb-didactic-exp}). 
    }
    \label{fig:didactic-exp-five-state-circular-cmp}
    \vspace{-1.5em}
\end{figure}

\subsection{Key Components of One-Step FB}
\label{appendix:ablation}

In this section, we conduct ablation experiments studying key components of one-step FB. We choose two ExORL domains \texttt{walker} and \texttt{cheetah}, and two OGBench domains \texttt{antmaze large navigate} and \texttt{scene play} to conduct experiments. As mentioned in Sec.~\ref{appendix:implementations-and-hyperparameters}, there are four key hyperparameters of one-step FB: \emph{(1)} the behavioral-cloning regularization coefficient $\lambda_{\text{BC}}$, \emph{(2)} the orthonormalization regularization coefficient $\lambda_{\text{ortho}}$, \emph{(3)} the reward weighting temperature $\tau_{\text{reward}}$, and \emph{(4)} the representation dimension $d$. Following the same evaluation protocol as in Appendix~\ref{appendix:experiment-details}, we compute means and standard deviations over $8$ random seeds for each domain.

We first study the effects of the behavioral-cloning regularization coefficient $\lambda_{\text{BC}}$, ablating over different values of $\lambda_{\text{BC}}$ within $\{ 0, 0.03, 0.3, 3, 30 \}$. We selectively show the comparisons between several $\lambda_{\text{BC}}$ values in Fig.~\ref{fig:ablation-bc}. These results suggest that removing the behavioral-cloning regularization $\lambda_{\text{BC}} = 0$ is important on ExORL domains, where our choices achieved $+1.6 \times$ improvement on average. This observation is consistent with findings from prior work~\citep{park2024foundation}, where they also exclude the behavioral-cloning regularization on ExORL tasks. In contrast, on OGBench domains, selecting a non-zero $\lambda_{\text{BC}}$ boosts the performance.

Next, to better understand the role of orthonormalizing the backward representations, we ablate over different values of the orthonormalization coefficient $\lambda_{\text{ortho}} \in \{0.01, 0.1, 0.0, 1.0, 10.0 \}$ and compare the zero-shot performance of one-step FB variants. We present the results with selective values of $\lambda_{\text{ortho}}$ in Fig.~\ref{fig:ablation-ortho}. Overall, the performance of one-step FB is sensitive to the choice of $\lambda_{\text{ortho}}$ on both ExORL and OGBench domains. We observe that setting $\lambda_{\text{ortho}} < 0.1$ results in lower performance on ExORL domains used for ablation experiments.
In contrast, using a very large $\lambda_{\text{ortho}}$ has negative effects on one-step FB for OGBench domains. 
Additionally, we observe that simply removing the orthonormalization regularization ($\lambda_{\text{ortho}} = 0$) can boost the success rate by at most $+2\times$ on OGBench domains. This indicates that using an appropriate value of $\lambda_{\text{ortho}}$ is key to the performance of one-step FB.

We also study the effect of our reward weighting strategy in Eq.~\ref{eq:reward-weighting} by ablating over different temperatures, $\tau_{\text{reward}} \in \{3, 10, 30, 300 \}$, in the softmax function. In Fig.~\ref{fig:ablation-tau}, we compare the zero-shot performance of three variants of one-step FB on each domain. These results suggest that one-step FB is less sensitive to the choice of $\tau_{\text{reward}}$ on each domain. However, using larger values of $\tau_{\text{reward}}$ can slightly boost performance on OGBench domains. Therefore, we still tune the reward weighting temperature $\tau_{\text{reward}}$ for each domain separately and select the best candidates. 

Finally, we study the effects of the representation dimension $d$. Both prior work~\citep{park2024foundation, touati2022does} and our Proposition~\ref{prop:fb-existence} have suggested that the representation dimension plays an important role for one-step FB. We sweep over $d = \{ 25, 50, 100, 128, 256, 512\}$ and selectively show performance of several values. As shown in Fig.~\ref{fig:ablation-repr}, the choice of representation dimension $d$ can vary the performance of one-step FB significantly on both ExORL and OGBench domains. We find that using $d = 50$, which is the same value as in~\citet{park2024foundation} and~\citet{touati2022does}, is sufficient for ExORL domains. However, when increasing the representation dimension, we do \emph{not} observe a consistent improvement over zero-shot performance. We conjecture that a finite representation dimension $d < \infty$ always learns a low-rank approximation of the successor measure ratio as in Corollary~\ref{corollary:finite-d-is-insufficient}. Thus, some choices of $d$ might result in a better low-rank approximation. On OGBench domains, we select $d = 512$ for consistency with prior work~\citep{bagatella2025td}, although $d = 128$ gives better performance on some domains.

Taken together, we tune the behavioral-cloning regularization coefficient $\lambda_{\text{BC}}$, the orthonormalization coefficient $\lambda_{\text{ortho}}$, the reward weighting temperature $\tau_{\text{reward}}$, and the representation dimension $d$ on different domains. In general, our choices of hyperparameters are effective for one-step FB on both ExORL and OGBench domains.

\begin{figure*}[t]
    \centering
    
    \begin{subfigure}{\linewidth}
        \centering
        \includegraphics[width=.99\linewidth]{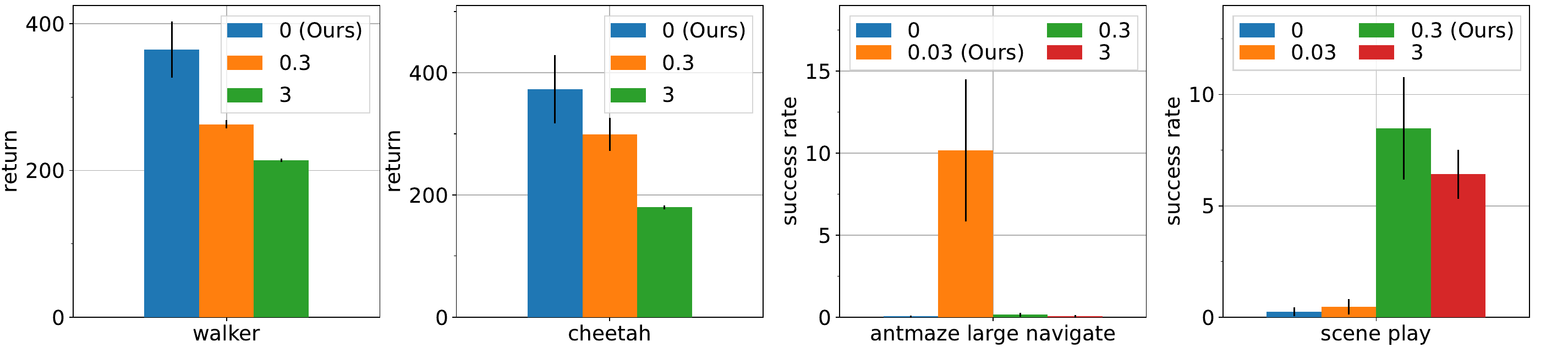}
        \caption{Behavioral-cloning regularization coefficient $\lambda_{\text{BC}}$.}
        \label{fig:ablation-bc}
    \end{subfigure}
    
    \vspace{0.25cm} %
    
    \begin{subfigure}{\linewidth}
        \centering
        \includegraphics[width=.99\linewidth]{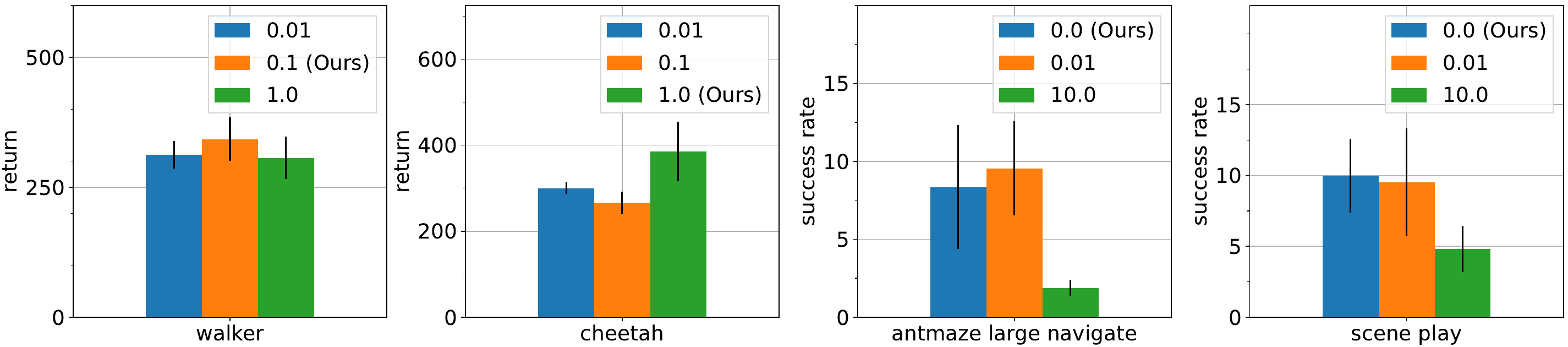}
        \caption{Orthonormalization regularization coefficient $\lambda_{\text{ortho}}$.}
        \label{fig:ablation-ortho}
    \end{subfigure}
    
    \vspace{0.25cm}
    
    \begin{subfigure}{\linewidth}
        \centering
        \includegraphics[width=.99\linewidth]{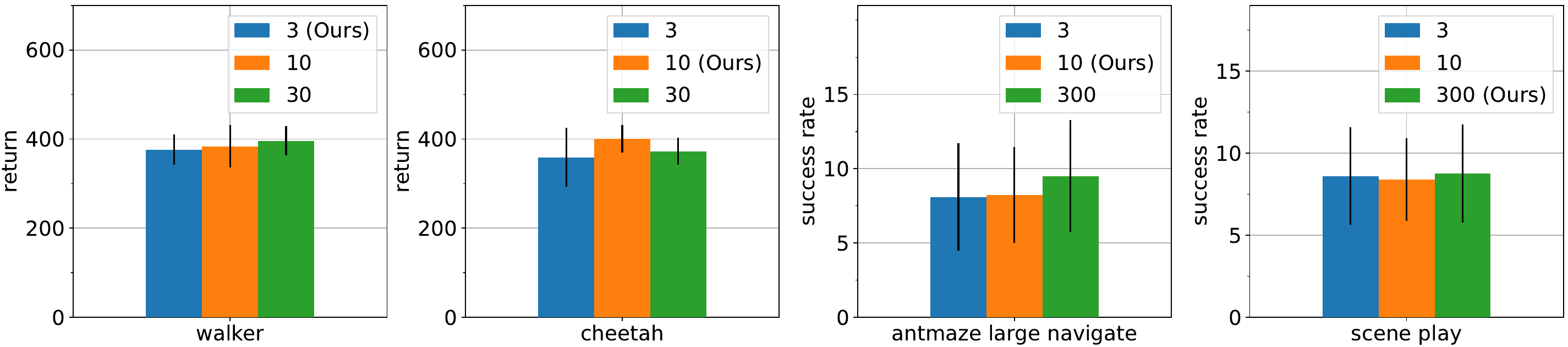}
        \caption{Reward weighting temperature $\tau_{\text{reward}}$.}
        \label{fig:ablation-tau}
    \end{subfigure}
    
    \vspace{0.25cm}
    
    \begin{subfigure}{\linewidth}
        \centering
        \includegraphics[width=.99\linewidth]{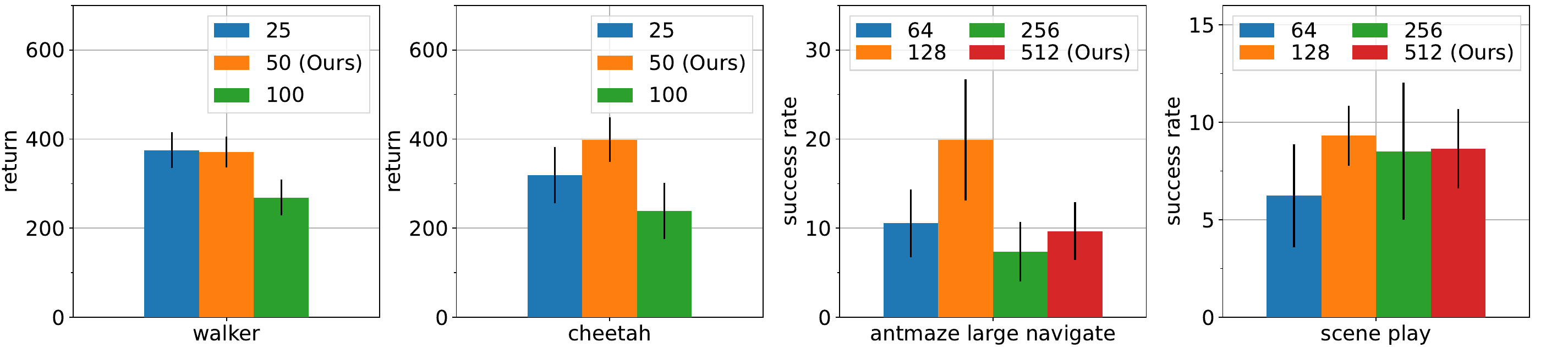}
        \caption{Representation dimension $d$.}
        \label{fig:ablation-repr}
    \end{subfigure}

    \caption{\footnotesize \textbf{Hyperparameter ablations.} We conduct ablation experiments to study the effect of key components of one-step FB on \texttt{walker}, \texttt{cheetah}, \texttt{antmaze large navigate}, and \texttt{scene play}.}
    \label{fig:hyperparameter_ablations}
\end{figure*}


\begin{thebibliography}{116}
\providecommand{\natexlab}[1]{#1}
\providecommand{\url}[1]{\texttt{#1}}
\expandafter\ifx\csname urlstyle\endcsname\relax
  \providecommand{\doi}[1]{doi: #1}\else
  \providecommand{\doi}{doi: \begingroup \urlstyle{rm}\Url}\fi

\bibitem[Achiam et~al.(2023)Achiam, Adler, Agarwal, Ahmad, Akkaya, Aleman, Almeida, Altenschmidt, Altman, Anadkat, et~al.]{achiam2023gpt}
Josh Achiam, Steven Adler, Sandhini Agarwal, Lama Ahmad, Ilge Akkaya, Florencia~Leoni Aleman, Diogo Almeida, Janko Altenschmidt, Sam Altman, Shyamal Anadkat, et~al.
\newblock Gpt-4 technical report.
\newblock \emph{arXiv preprint arXiv:2303.08774}, 2023.

\bibitem[Agarwal et~al.(2019)Agarwal, Jiang, Kakade, and Sun]{agarwal2019reinforcement}
Alekh Agarwal, Nan Jiang, Sham~M Kakade, and Wen Sun.
\newblock Reinforcement learning: Theory and algorithms.
\newblock \emph{CS Dept., UW Seattle, Seattle, WA, USA, Tech. Rep}, 32:\penalty0 96, 2019.

\bibitem[Agarwal et~al.(2025)Agarwal, Sikchi, Stone, and Zhang]{agarwal2025proto}
Siddhant Agarwal, Harshit Sikchi, Peter Stone, and Amy Zhang.
\newblock Proto successor measure: Representing the behavior space of an {RL} agent.
\newblock In \emph{Forty-second International Conference on Machine Learning}, 2025.
\newblock URL \url{https://openreview.net/forum?id=mUDnPzopZF}.

\bibitem[Assran et~al.(2023)Assran, Duval, Misra, Bojanowski, Vincent, Rabbat, LeCun, and Ballas]{assran2023self}
Mahmoud Assran, Quentin Duval, Ishan Misra, Piotr Bojanowski, Pascal Vincent, Michael Rabbat, Yann LeCun, and Nicolas Ballas.
\newblock Self-supervised learning from images with a joint-embedding predictive architecture.
\newblock In \emph{Proceedings of the IEEE/CVF Conference on Computer Vision and Pattern Recognition}, pages 15619--15629, 2023.

\bibitem[Bagatella et~al.(2025)Bagatella, Pirotta, Touati, Lazaric, and Tirinzoni]{bagatella2025td}
Marco Bagatella, Matteo Pirotta, Ahmed Touati, Alessandro Lazaric, and Andrea Tirinzoni.
\newblock Td-jepa: Latent-predictive representations for zero-shot reinforcement learning.
\newblock \emph{arXiv preprint arXiv:2510.00739}, 2025.

\bibitem[Banach(1922)]{banach1922operations}
Stefan Banach.
\newblock Sur les op{\'e}rations dans les ensembles abstraits et leur application aux {\'e}quations int{\'e}grales.
\newblock \emph{Fundamenta mathematicae}, 3\penalty0 (1):\penalty0 133--181, 1922.

\bibitem[Barreto et~al.(2017)Barreto, Dabney, Munos, Hunt, Schaul, van Hasselt, and Silver]{barreto2017successor}
Andr{\'e} Barreto, Will Dabney, R{\'e}mi Munos, Jonathan~J Hunt, Tom Schaul, Hado~P van Hasselt, and David Silver.
\newblock Successor features for transfer in reinforcement learning.
\newblock \emph{Advances in neural information processing systems}, 30, 2017.

\bibitem[Belghazi et~al.(2018)Belghazi, Baratin, Rajeshwar, Ozair, Bengio, Courville, and Hjelm]{belghazi2018mutual}
Mohamed~Ishmael Belghazi, Aristide Baratin, Sai Rajeshwar, Sherjil Ozair, Yoshua Bengio, Aaron Courville, and Devon Hjelm.
\newblock Mutual information neural estimation.
\newblock In \emph{International conference on machine learning}, pages 531--540. PMLR, 2018.

\bibitem[Berlinet and Thomas-Agnan(2011)]{berlinet2011reproducing}
Alain Berlinet and Christine Thomas-Agnan.
\newblock \emph{Reproducing kernel Hilbert spaces in probability and statistics}.
\newblock Springer Science \& Business Media, 2011.

\bibitem[Bhatt and Borkar(1996)]{bhatt1996occupation}
Abhay~G Bhatt and Vivek~S Borkar.
\newblock Occupation measures for controlled markov processes: Characterization and optimality.
\newblock \emph{The Annals of Probability}, pages 1531--1562, 1996.

\bibitem[Bjerhammar(1951)]{Bjerhammar1951ApplicationOC}
Arne Bjerhammar.
\newblock Application of calculus of matrices to method of least squares: with special reference to geodetic calculations.
\newblock 1951.
\newblock URL \url{https://api.semanticscholar.org/CorpusID:118134976}.

\bibitem[Blier et~al.(2021)Blier, Tallec, and Ollivier]{blier2021learning}
L{\'e}onard Blier, Corentin Tallec, and Yann Ollivier.
\newblock Learning successor states and goal-dependent values: A mathematical viewpoint.
\newblock \emph{arXiv preprint arXiv:2101.07123}, 2021.

\bibitem[Borsa et~al.(2019)Borsa, Barreto, Quan, Mankowitz, van Hasselt, Munos, Silver, and Schaul]{borsa2019universal}
Diana Borsa, Andr{\'e} Barreto, John Quan, Daniel Mankowitz, Hado van Hasselt, R{\'e}mi Munos, David Silver, and Tom Schaul.
\newblock Universal successor features approximators.
\newblock In \emph{International Conference on Learning Representations (ICLR)}, 2019.
\newblock URL \url{https://openreview.net/forum?id=S1VWjiRcKX}.

\bibitem[Bradbury et~al.(2018)Bradbury, Frostig, Hawkins, Johnson, Leary, Maclaurin, Necula, Paszke, Vander{P}las, Wanderman-{M}ilne, and Zhang]{jax2018github}
James Bradbury, Roy Frostig, Peter Hawkins, Matthew~James Johnson, Chris Leary, Dougal Maclaurin, George Necula, Adam Paszke, Jake Vander{P}las, Skye Wanderman-{M}ilne, and Qiao Zhang.
\newblock {JAX}: composable transformations of {P}ython+{N}um{P}y programs.
\newblock GitHub repository, 2018.
\newblock URL \url{http://github.com/jax-ml/jax}.
\newblock Version 0.3.13.

\bibitem[Brandfonbrener et~al.(2021)Brandfonbrener, Whitney, Ranganath, and Bruna]{brandfonbrener2021offline}
David Brandfonbrener, Will Whitney, Rajesh Ranganath, and Joan Bruna.
\newblock Offline rl without off-policy evaluation.
\newblock \emph{Advances in neural information processing systems}, 34:\penalty0 4933--4946, 2021.

\bibitem[Brown et~al.(2020)Brown, Mann, Ryder, Subbiah, Kaplan, Dhariwal, Neelakantan, Shyam, Sastry, Askell, et~al.]{brown2020language}
Tom Brown, Benjamin Mann, Nick Ryder, Melanie Subbiah, Jared~D Kaplan, Prafulla Dhariwal, Arvind Neelakantan, Pranav Shyam, Girish Sastry, Amanda Askell, et~al.
\newblock Language models are few-shot learners.
\newblock \emph{Advances in neural information processing systems}, 33:\penalty0 1877--1901, 2020.

\bibitem[Burda et~al.(2019)Burda, Edwards, Storkey, and Klimov]{burda2018exploration}
Yuri Burda, Harrison Edwards, Amos Storkey, and Oleg Klimov.
\newblock Exploration by random network distillation.
\newblock In \emph{International Conference on Learning Representations}, 2019.
\newblock URL \url{https://openreview.net/forum?id=H1lJJnR5Ym}.

\bibitem[Chen et~al.(2021)Chen, Lu, Rajeswaran, Lee, Grover, Laskin, Abbeel, Srinivas, and Mordatch]{chen2021decision}
Lili Chen, Kevin Lu, Aravind Rajeswaran, Kimin Lee, Aditya Grover, Michael Laskin, Pieter Abbeel, Aravind Srinivas, and Igor Mordatch.
\newblock Decision transformer: Reinforcement learning via sequence modeling.
\newblock In A.~Beygelzimer, Y.~Dauphin, P.~Liang, and J.~Wortman Vaughan, editors, \emph{Advances in Neural Information Processing Systems}, 2021.
\newblock URL \url{https://openreview.net/forum?id=a7APmM4B9d}.

\bibitem[Cheng and Chu(2004)]{cheng2004semiparametric}
Kuang~Fu Cheng and Chih-Kang Chu.
\newblock Semiparametric density estimation under a two-sample density ratio model.
\newblock \emph{Bernoulli}, 10\penalty0 (4):\penalty0 583--604, 2004.

\bibitem[Choi et~al.(2021)Choi, Sharma, Lee, Levine, and Gu]{choi2021variational}
Jongwook Choi, Archit Sharma, Honglak Lee, Sergey Levine, and Shixiang~Shane Gu.
\newblock Variational empowerment as representation learning for goal-conditioned reinforcement learning.
\newblock In \emph{International conference on machine learning}, pages 1953--1963. PMLR, 2021.

\bibitem[Dayan(1993)]{dayan1993improving}
Peter Dayan.
\newblock Improving generalization for temporal difference learning: The successor representation.
\newblock \emph{Neural computation}, 5\penalty0 (4):\penalty0 613--624, 1993.

\bibitem[Ding et~al.(2019)Ding, Florensa, Abbeel, and Phielipp]{ding2019}
Yiming Ding, Carlos Florensa, Pieter Abbeel, and Mariano Phielipp.
\newblock Goal-conditioned imitation learning.
\newblock In H.~Wallach, H.~Larochelle, A.~Beygelzimer, F.~d\textquotesingle Alch\'{e}-Buc, E.~Fox, and R.~Garnett, editors, \emph{Advances in Neural Information Processing Systems}, volume~32. Curran Associates, Inc., 2019.
\newblock URL \url{https://proceedings.neurips.cc/paper_files/paper/2019/file/c8d3a760ebab631565f8509d84b3b3f1-Paper.pdf}.

\bibitem[Ernst et~al.(2005)Ernst, Geurts, and Wehenkel]{ernst2005tree}
Damien Ernst, Pierre Geurts, and Louis Wehenkel.
\newblock Tree-based batch mode reinforcement learning.
\newblock \emph{Journal of Machine Learning Research}, 6, 2005.

\bibitem[Espeholt et~al.(2018)Espeholt, Soyer, Munos, Simonyan, Mnih, Ward, Doron, Firoiu, Harley, Dunning, et~al.]{espeholt2018impala}
Lasse Espeholt, Hubert Soyer, Remi Munos, Karen Simonyan, Vlad Mnih, Tom Ward, Yotam Doron, Vlad Firoiu, Tim Harley, Iain Dunning, et~al.
\newblock Impala: Scalable distributed deep-rl with importance weighted actor-learner architectures.
\newblock In \emph{International conference on machine learning}, pages 1407--1416. PMLR, 2018.

\bibitem[Eysenbach et~al.(2020)Eysenbach, Geng, Levine, and Salakhutdinov]{eysenbach2020rewriting}
Ben Eysenbach, Xinyang Geng, Sergey Levine, and Russ~R Salakhutdinov.
\newblock Rewriting history with inverse rl: Hindsight inference for policy improvement.
\newblock \emph{Advances in neural information processing systems}, 33:\penalty0 14783--14795, 2020.

\bibitem[Eysenbach et~al.(2018)Eysenbach, Gupta, Ibarz, and Levine]{eysenbach2018diversity}
Benjamin Eysenbach, Abhishek Gupta, Julian Ibarz, and Sergey Levine.
\newblock Diversity is all you need: Learning skills without a reward function.
\newblock \emph{arXiv preprint arXiv:1802.06070}, 2018.

\bibitem[Eysenbach et~al.(2022{\natexlab{a}})Eysenbach, Geist, Salakhutdinov, and Levine]{eysenbach2022a}
Benjamin Eysenbach, Matthieu Geist, Ruslan Salakhutdinov, and Sergey Levine.
\newblock A connection between one-step regularization and critic regularization in reinforcement learning.
\newblock In \emph{Deep Reinforcement Learning Workshop NeurIPS 2022}, 2022{\natexlab{a}}.
\newblock URL \url{https://openreview.net/forum?id=GXiWE8kDTcn}.

\bibitem[Eysenbach et~al.(2022{\natexlab{b}})Eysenbach, Udatha, Salakhutdinov, and Levine]{eysenbach2022imitating}
Benjamin Eysenbach, Soumith Udatha, Ruslan Salakhutdinov, and Sergey Levine.
\newblock Imitating past successes can be very suboptimal.
\newblock In Alice~H. Oh, Alekh Agarwal, Danielle Belgrave, and Kyunghyun Cho, editors, \emph{Advances in Neural Information Processing Systems}, 2022{\natexlab{b}}.
\newblock URL \url{https://openreview.net/forum?id=iqCO3jbPjYF}.

\bibitem[Eysenbach et~al.(2022{\natexlab{c}})Eysenbach, Zhang, Levine, and Salakhutdinov]{eysenbach2022contrastive}
Benjamin Eysenbach, Tianjun Zhang, Sergey Levine, and Russ~R Salakhutdinov.
\newblock Contrastive learning as goal-conditioned reinforcement learning.
\newblock \emph{Advances in Neural Information Processing Systems}, 35:\penalty0 35603--35620, 2022{\natexlab{c}}.

\bibitem[Frans et~al.(2024)Frans, Park, Abbeel, and Levine]{frans2024unsupervised}
Kevin Frans, Seohong Park, Pieter Abbeel, and Sergey Levine.
\newblock Unsupervised zero-shot reinforcement learning via functional reward encodings.
\newblock In \emph{International Conference on Machine Learning}, pages 13927--13942. PMLR, 2024.

\bibitem[Fujimoto and Gu(2021)]{fujimoto2021minimalist}
Scott Fujimoto and Shixiang~Shane Gu.
\newblock A minimalist approach to offline reinforcement learning.
\newblock \emph{Advances in neural information processing systems}, 34:\penalty0 20132--20145, 2021.

\bibitem[Fujimoto et~al.(2018)Fujimoto, Hoof, and Meger]{fujimoto2018addressing}
Scott Fujimoto, Herke Hoof, and David Meger.
\newblock Addressing function approximation error in actor-critic methods.
\newblock In \emph{International conference on machine learning}, pages 1587--1596. PMLR, 2018.

\bibitem[Fujimoto et~al.(2022)Fujimoto, Meger, Precup, Nachum, and Gu]{fujimoto2022should}
Scott Fujimoto, David Meger, Doina Precup, Ofir Nachum, and Shixiang~Shane Gu.
\newblock Why should i trust you, bellman? the bellman error is a poor replacement for value error.
\newblock \emph{arXiv preprint arXiv:2201.12417}, 2022.

\bibitem[Ghosh et~al.(2023)Ghosh, Bhateja, and Levine]{ghosh2023reinforcement}
Dibya Ghosh, Chethan~Anand Bhateja, and Sergey Levine.
\newblock Reinforcement learning from passive data via latent intentions.
\newblock In \emph{International Conference on Machine Learning}, pages 11321--11339. PMLR, 2023.

\bibitem[Gretton et~al.(2009)Gretton, Smola, Huang, Schmittfull, Borgwardt, Sch{\"o}lkopf, et~al.]{gretton2009covariate}
Arthur Gretton, Alex Smola, Jiayuan Huang, Marcel Schmittfull, Karsten Borgwardt, Bernhard Sch{\"o}lkopf, et~al.
\newblock Covariate shift by kernel mean matching.
\newblock \emph{Dataset shift in machine learning}, 3\penalty0 (4):\penalty0 5, 2009.

\bibitem[Grill et~al.(2020)Grill, Strub, Altch{\'e}, Tallec, Richemond, Buchatskaya, Doersch, Avila~Pires, Guo, Gheshlaghi~Azar, et~al.]{grill2020bootstrap}
Jean-Bastien Grill, Florian Strub, Florent Altch{\'e}, Corentin Tallec, Pierre Richemond, Elena Buchatskaya, Carl Doersch, Bernardo Avila~Pires, Zhaohan Guo, Mohammad Gheshlaghi~Azar, et~al.
\newblock Bootstrap your own latent-a new approach to self-supervised learning.
\newblock \emph{Advances in neural information processing systems}, 33:\penalty0 21271--21284, 2020.

\bibitem[Gulcehre et~al.(2020)Gulcehre, Wang, Novikov, Paine, Colmenarejo, Zolna, Agarwal, Merel, Mankowitz, Paduraru, Dulac-Arnold, Li, Norouzi, Hoffman, Nachum, Tucker, Heess, and deFreitas]{gulcehre2020rl}
Caglar Gulcehre, Ziyu Wang, Alexander Novikov, Tom~Le Paine, Sergio~Gómez Colmenarejo, Konrad Zolna, Rishabh Agarwal, Josh Merel, Daniel Mankowitz, Cosmin Paduraru, Gabriel Dulac-Arnold, Jerry Li, Mohammad Norouzi, Matt Hoffman, Ofir Nachum, George Tucker, Nicolas Heess, and Nando deFreitas.
\newblock Rl unplugged: Benchmarks for offline reinforcement learning, 2020.

\bibitem[Gutmann and Hyv{\"a}rinen(2010)]{gutmann2010noise}
Michael Gutmann and Aapo Hyv{\"a}rinen.
\newblock Noise-contrastive estimation: A new estimation principle for unnormalized statistical models.
\newblock In \emph{Proceedings of the thirteenth international conference on artificial intelligence and statistics}, pages 297--304. JMLR Workshop and Conference Proceedings, 2010.

\bibitem[Haarnoja et~al.(2018)Haarnoja, Zhou, Abbeel, and Levine]{haarnoja2018soft}
Tuomas Haarnoja, Aurick Zhou, Pieter Abbeel, and Sergey Levine.
\newblock Soft actor-critic: Off-policy maximum entropy deep reinforcement learning with a stochastic actor.
\newblock In \emph{International conference on machine learning}, pages 1861--1870. PMLR, 2018.

\bibitem[Hendrycks and Gimpel(2023)]{hendrycks2023gaussian}
Dan Hendrycks and Kevin Gimpel.
\newblock Gaussian error linear units (gelus), 2023.
\newblock URL \url{https://arxiv.org/abs/1606.08415}.

\bibitem[Hornik et~al.(1989)Hornik, Stinchcombe, and White]{hornik1989multilayer}
Kurt Hornik, Maxwell Stinchcombe, and Halbert White.
\newblock Multilayer feedforward networks are universal approximators.
\newblock \emph{Neural networks}, 2\penalty0 (5):\penalty0 359--366, 1989.

\bibitem[Hu et~al.(2023)Hu, Yang, Ye, Mai, and Zhang]{hu2023unsupervised}
Hao Hu, Yiqin Yang, Jianing Ye, Ziqing Mai, and Chongjie Zhang.
\newblock Unsupervised behavior extraction via random intent priors.
\newblock \emph{Advances in Neural Information Processing Systems}, 36:\penalty0 51491--51514, 2023.

\bibitem[Janner et~al.(2020)Janner, Mordatch, and Levine]{janner2020gamma}
Michael Janner, Igor Mordatch, and Sergey Levine.
\newblock Gamma-models: Generative temporal difference learning for infinite-horizon prediction.
\newblock \emph{Advances in neural information processing systems}, 33:\penalty0 1724--1735, 2020.

\bibitem[Kanamori et~al.(2008)Kanamori, Hido, and Sugiyama]{kanamori2008efficient}
Takafumi Kanamori, Shohei Hido, and Masashi Sugiyama.
\newblock Efficient direct density ratio estimation for non-stationarity adaptation and outlier detection.
\newblock \emph{Advances in neural information processing systems}, 21, 2008.

\bibitem[Kanamori et~al.(2009)Kanamori, Hido, and Sugiyama]{kanamori2009least}
Takafumi Kanamori, Shohei Hido, and Masashi Sugiyama.
\newblock A least-squares approach to direct importance estimation.
\newblock \emph{The Journal of Machine Learning Research}, 10:\penalty0 1391--1445, 2009.

\bibitem[Kato(2025)]{kato2025riesz}
Masahiro Kato.
\newblock Riesz regression as direct density ratio estimation.
\newblock \emph{arXiv preprint arXiv:2511.04568}, 2025.

\bibitem[Kim et~al.(2024)Kim, Park, and Levine]{kim2024unsupervised}
Junsu Kim, Seohong Park, and Sergey Levine.
\newblock Unsupervised-to-online reinforcement learning.
\newblock \emph{arXiv preprint arXiv:2408.14785}, 2024.

\bibitem[Kingma and Ba(2015)]{kingma2015adam}
Diederik~P Kingma and Jimmy Ba.
\newblock Adam: A method for stochastic optimization.
\newblock In \emph{International Conference on Learning Representations (ICLR)}, 2015.

\bibitem[Klyubin et~al.(2005)Klyubin, Polani, and Nehaniv]{klyubin2005empowerment}
Alexander~S Klyubin, Daniel Polani, and Chrystopher~L Nehaniv.
\newblock Empowerment: A universal agent-centric measure of control.
\newblock In \emph{2005 ieee congress on evolutionary computation}, volume~1, pages 128--135. IEEE, 2005.

\bibitem[Kostrikov et~al.(2021)Kostrikov, Nair, and Levine]{kostrikov2021offline}
Ilya Kostrikov, Ashvin Nair, and Sergey Levine.
\newblock Offline reinforcement learning with implicit q-learning.
\newblock \emph{arXiv preprint arXiv:2110.06169}, 2021.

\bibitem[Kulkarni et~al.(2016)Kulkarni, Saeedi, Gautam, and Gershman]{kulkarni2016deep}
Tejas~D Kulkarni, Ardavan Saeedi, Simanta Gautam, and Samuel~J Gershman.
\newblock Deep successor reinforcement learning.
\newblock \emph{arXiv preprint arXiv:1606.02396}, 2016.

\bibitem[Kumar et~al.(2019)Kumar, Peng, and Levine]{kumar2019rewardconditionedpolicies}
Aviral Kumar, Xue~Bin Peng, and Sergey Levine.
\newblock Reward-conditioned policies, 2019.
\newblock URL \url{https://arxiv.org/abs/1912.13465}.

\bibitem[Lange et~al.(2012)Lange, Gabel, and Riedmiller]{lange2012batch}
Sascha Lange, Thomas Gabel, and Martin Riedmiller.
\newblock Batch reinforcement learning.
\newblock In \emph{Reinforcement learning: State-of-the-art}, pages 45--73. Springer, 2012.

\bibitem[Lawson et~al.(2025)Lawson, Hugessen, Cloutier, Berseth, and Khetarpal]{lawson2025self}
Daniel Lawson, Adriana Hugessen, Charlotte Cloutier, Glen Berseth, and Khimya Khetarpal.
\newblock Self-predictive representations for combinatorial generalization in behavioral cloning, 2025.
\newblock URL \url{https://arxiv.org/abs/2506.10137}.

\bibitem[Lefschetz(1926)]{lefschetz1926intersections}
Solomon Lefschetz.
\newblock Intersections and transformations of complexes and manifolds.
\newblock \emph{Transactions of the American Mathematical Society}, 28\penalty0 (1):\penalty0 1--49, 1926.

\bibitem[Li et~al.(2020)Li, Pinto, and Abbeel]{li2020generalized}
Alexander Li, Lerrel Pinto, and Pieter Abbeel.
\newblock Generalized hindsight for reinforcement learning.
\newblock \emph{Advances in Neural Information Processing Systems}, 33, 2020.

\bibitem[Li et~al.(2025)Li, Luo, Zhang, Dai, Kanervisto, Tirinzoni, Weng, Kitani, Guzek, Touati, et~al.]{li2025bfm}
Yitang Li, Zhengyi Luo, Tonghe Zhang, Cunxi Dai, Anssi Kanervisto, Andrea Tirinzoni, Haoyang Weng, Kris Kitani, Mateusz Guzek, Ahmed Touati, et~al.
\newblock Bfm-zero: A promptable behavioral foundation model for humanoid control using unsupervised reinforcement learning.
\newblock \emph{arXiv preprint arXiv:2511.04131}, 2025.

\bibitem[Lillicrap et~al.(2015)Lillicrap, Hunt, Pritzel, Heess, Erez, Tassa, Silver, and Wierstra]{lillicrap2015continuous}
Timothy~P Lillicrap, Jonathan~J Hunt, Alexander Pritzel, Nicolas Heess, Tom Erez, Yuval Tassa, David Silver, and Daan Wierstra.
\newblock Continuous control with deep reinforcement learning.
\newblock \emph{arXiv preprint arXiv:1509.02971}, 2015.

\bibitem[Lopez-Paz and Oquab(2016)]{lopez2016revisiting}
David Lopez-Paz and Maxime Oquab.
\newblock Revisiting classifier two-sample tests.
\newblock \emph{arXiv preprint arXiv:1610.06545}, 2016.

\bibitem[Loshchilov and Hutter(2019)]{loshchilov2018decoupled}
Ilya Loshchilov and Frank Hutter.
\newblock Decoupled weight decay regularization.
\newblock In \emph{International Conference on Learning Representations}, 2019.
\newblock URL \url{https://openreview.net/forum?id=Bkg6RiCqY7}.

\bibitem[Lu et~al.(2019)Lu, Jin, and Karniadakis]{lu2019deeponet}
Lu~Lu, Pengzhan Jin, and George~Em Karniadakis.
\newblock Deeponet: Learning nonlinear operators for identifying differential equations based on the universal approximation theorem of operators.
\newblock \emph{arXiv preprint arXiv:1910.03193}, 2019.

\bibitem[Lyapunov(1992)]{lyapunov1992general}
Aleksandr~Mikhailovich Lyapunov.
\newblock The general problem of the stability of motion.
\newblock \emph{International journal of control}, 55\penalty0 (3):\penalty0 531--534, 1992.

\bibitem[Ma et~al.(2023)Ma, Sodhani, Jayaraman, Bastani, Kumar, and Zhang]{ma2023vip}
Yecheng~Jason Ma, Shagun Sodhani, Dinesh Jayaraman, Osbert Bastani, Vikash Kumar, and Amy Zhang.
\newblock {VIP}: Towards universal visual reward and representation via value-implicit pre-training.
\newblock In \emph{The Eleventh International Conference on Learning Representations}, 2023.
\newblock URL \url{https://openreview.net/forum?id=YJ7o2wetJ2}.

\bibitem[Ma and Collins(2018)]{ma2018noise}
Zhuang Ma and Michael Collins.
\newblock Noise contrastive estimation and negative sampling for conditional models: Consistency and statistical efficiency.
\newblock \emph{arXiv preprint arXiv:1809.01812}, 2018.

\bibitem[Mazoure et~al.(2023)Mazoure, Eysenbach, Nachum, and Tompson]{mazoure2023contrastive}
Bogdan Mazoure, Benjamin Eysenbach, Ofir Nachum, and Jonathan Tompson.
\newblock Contrastive value learning: Implicit models for simple offline {RL}.
\newblock In \emph{7th Annual Conference on Robot Learning}, 2023.
\newblock URL \url{https://openreview.net/forum?id=oqOfLP6bJy}.

\bibitem[Mnih et~al.(2015)Mnih, Kavukcuoglu, Silver, Rusu, Veness, Bellemare, Graves, Riedmiller, Fidjeland, Ostrovski, et~al.]{mnih2015human}
Volodymyr Mnih, Koray Kavukcuoglu, David Silver, Andrei~A Rusu, Joel Veness, Marc~G Bellemare, Alex Graves, Martin Riedmiller, Andreas~K Fidjeland, Georg Ostrovski, et~al.
\newblock Human-level control through deep reinforcement learning.
\newblock \emph{nature}, 518\penalty0 (7540):\penalty0 529--533, 2015.

\bibitem[Mohamed and Jimenez~Rezende(2015)]{mohamed2015variational}
Shakir Mohamed and Danilo Jimenez~Rezende.
\newblock Variational information maximisation for intrinsically motivated reinforcement learning.
\newblock \emph{Advances in neural information processing systems}, 28, 2015.

\bibitem[Momennejad et~al.(2017)Momennejad, Russek, Cheong, Botvinick, Daw, and Gershman]{Momennejad2017successor}
I~Momennejad, EM~Russek, JH~Cheong, MM~Botvinick, ND~Daw, and SJ~Gershman.
\newblock The successor representation in human reinforcement learning.
\newblock \emph{Nature Human Behaviour}, 1\penalty0 (9):\penalty0 680--692, 2017.
\newblock \doi{10.1038/s41562-017-0180-8}.

\bibitem[Moore(1920)]{moore1920reciprocal}
Eliakim~H Moore.
\newblock On the reciprocal of the general algebraic matrix.
\newblock \emph{Bulletin of the american mathematical society}, 26:\penalty0 294--295, 1920.

\bibitem[Myers et~al.(2024)Myers, Zheng, Dragan, Levine, and Eysenbach]{myers2024learning}
Vivek Myers, Chongyi Zheng, Anca Dragan, Sergey Levine, and Benjamin Eysenbach.
\newblock Learning temporal distances: Contrastive successor features can provide a metric structure for decision-making.
\newblock In \emph{Forty-first International Conference on Machine Learning}, 2024.
\newblock URL \url{https://openreview.net/forum?id=xQiYCmDrjp}.

\bibitem[Nachum et~al.(2019)Nachum, Chow, Dai, and Li]{nachum2019dualdice}
Ofir Nachum, Yinlam Chow, Bo~Dai, and Lihong Li.
\newblock Dualdice: Behavior-agnostic estimation of discounted stationary distribution corrections.
\newblock \emph{Advances in neural information processing systems}, 32, 2019.

\bibitem[Nair and Hinton(2010)]{nair2010rectified}
Vinod Nair and Geoffrey~E Hinton.
\newblock Rectified linear units improve restricted boltzmann machines.
\newblock In \emph{Proceedings of the 27th International Conference on Machine Learning (ICML)}, pages 807--814, 2010.

\bibitem[Nakamoto et~al.(2023)Nakamoto, Zhai, Singh, Mark, Ma, Finn, Kumar, and Levine]{nakamoto2023calql}
Mitsuhiko Nakamoto, Yuexiang Zhai, Anikait Singh, Max~Sobol Mark, Yi~Ma, Chelsea Finn, Aviral Kumar, and Sergey Levine.
\newblock Cal-{QL}: Calibrated offline {RL} pre-training for efficient online fine-tuning.
\newblock In \emph{Thirty-seventh Conference on Neural Information Processing Systems}, 2023.
\newblock URL \url{https://openreview.net/forum?id=GcEIvidYSw}.

\bibitem[Ng et~al.(1999)Ng, Harada, and Russell]{ng1999policy}
Andrew~Y Ng, Daishi Harada, and Stuart Russell.
\newblock Policy invariance under reward transformations: Theory and application to reward shaping.
\newblock In \emph{Icml}, volume~99, pages 278--287. Citeseer, 1999.

\bibitem[Nguyen et~al.(2010)Nguyen, Wainwright, and Jordan]{nguyen2010estimating}
XuanLong Nguyen, Martin~J Wainwright, and Michael~I Jordan.
\newblock Estimating divergence functionals and the likelihood ratio by convex risk minimization.
\newblock \emph{IEEE Transactions on Information Theory}, 56\penalty0 (11):\penalty0 5847--5861, 2010.

\bibitem[Ni et~al.(2024)Ni, Eysenbach, Seyedsalehi, Ma, Gehring, Mahajan, and Bacon]{ni2024bridging}
Tianwei Ni, Benjamin Eysenbach, Erfan Seyedsalehi, Michel Ma, Clement Gehring, Aditya Mahajan, and Pierre-Luc Bacon.
\newblock Bridging state and history representations: Understanding self-predictive rl.
\newblock \emph{arXiv preprint arXiv:2401.08898}, 2024.

\bibitem[Oord et~al.(2018)Oord, Li, and Vinyals]{oord2018representation}
Aaron van~den Oord, Yazhe Li, and Oriol Vinyals.
\newblock Representation learning with contrastive predictive coding.
\newblock \emph{arXiv preprint arXiv:1807.03748}, 2018.

\bibitem[Ouyang et~al.(2022)Ouyang, Wu, Jiang, Almeida, Wainwright, Mishkin, Zhang, Agarwal, Slama, Ray, et~al.]{ouyang2022training}
Long Ouyang, Jeffrey Wu, Xu~Jiang, Diogo Almeida, Carroll Wainwright, Pamela Mishkin, Chong Zhang, Sandhini Agarwal, Katarina Slama, Alex Ray, et~al.
\newblock Training language models to follow instructions with human feedback.
\newblock \emph{Advances in neural information processing systems}, 35:\penalty0 27730--27744, 2022.

\bibitem[Park et~al.(2023)Park, Rybkin, and Levine]{park2023metra}
Seohong Park, Oleh Rybkin, and Sergey Levine.
\newblock Metra: Scalable unsupervised rl with metric-aware abstraction.
\newblock \emph{arXiv preprint arXiv:2310.08887}, 2023.

\bibitem[Park et~al.(2024{\natexlab{a}})Park, Frans, Eysenbach, and Levine]{park2024ogbench}
Seohong Park, Kevin Frans, Benjamin Eysenbach, and Sergey Levine.
\newblock Ogbench: Benchmarking offline goal-conditioned rl.
\newblock \emph{arXiv preprint arXiv:2410.20092}, 2024{\natexlab{a}}.

\bibitem[Park et~al.(2024{\natexlab{b}})Park, Kreiman, and Levine]{park2024foundation}
Seohong Park, Tobias Kreiman, and Sergey Levine.
\newblock Foundation policies with hilbert representations.
\newblock In \emph{International Conference on Machine Learning}, pages 39737--39761. PMLR, 2024{\natexlab{b}}.

\bibitem[Park et~al.(2025{\natexlab{a}})Park, Frans, Mann, Eysenbach, Kumar, and Levine]{park2025horizon}
Seohong Park, Kevin Frans, Deepinder Mann, Benjamin Eysenbach, Aviral Kumar, and Sergey Levine.
\newblock Horizon reduction makes rl scalable.
\newblock \emph{arXiv preprint arXiv:2506.04168}, 2025{\natexlab{a}}.

\bibitem[Park et~al.(2025{\natexlab{b}})Park, Li, and Levine]{park2025flow}
Seohong Park, Qiyang Li, and Sergey Levine.
\newblock Flow q-learning.
\newblock \emph{arXiv preprint arXiv:2502.02538}, 2025{\natexlab{b}}.

\bibitem[Park et~al.(2025{\natexlab{c}})Park, Oberai, Atreya, and Levine]{park2025transitive}
Seohong Park, Aditya Oberai, Pranav Atreya, and Sergey Levine.
\newblock Transitive rl: Value learning via divide and conquer.
\newblock \emph{arXiv preprint arXiv:2510.22512}, 2025{\natexlab{c}}.

\bibitem[Peters and Schaal(2007)]{peters2007reinforcement}
Jan Peters and Stefan Schaal.
\newblock Reinforcement learning by reward-weighted regression for operational space control.
\newblock In \emph{Proceedings of the 24th international conference on Machine learning}, pages 745--750, 2007.

\bibitem[Peters et~al.(2010)Peters, M\"{u}lling, and Alt\"{u}n]{peters2010reps}
Jan Peters, Katharina M\"{u}lling, and Yasemin Alt\"{u}n.
\newblock Relative entropy policy search.
\newblock In \emph{Proceedings of the Twenty-Fourth AAAI Conference on Artificial Intelligence}, AAAI'10, page 1607–1612. AAAI Press, 2010.

\bibitem[Poole et~al.(2019)Poole, Ozair, Van Den~Oord, Alemi, and Tucker]{poole2019variational}
Ben Poole, Sherjil Ozair, Aaron Van Den~Oord, Alex Alemi, and George Tucker.
\newblock On variational bounds of mutual information.
\newblock In \emph{International Conference on Machine Learning}, pages 5171--5180. PMLR, 2019.

\bibitem[Qin(1998)]{qin1998inferences}
Jing Qin.
\newblock Inferences for case-control and semiparametric two-sample density ratio models.
\newblock \emph{Biometrika}, 85\penalty0 (3):\penalty0 619--630, 1998.

\bibitem[Radford et~al.(2021)Radford, Kim, Hallacy, Ramesh, Goh, Agarwal, Sastry, Askell, Mishkin, Clark, et~al.]{radford2021learning}
Alec Radford, Jong~Wook Kim, Chris Hallacy, Aditya Ramesh, Gabriel Goh, Sandhini Agarwal, Girish Sastry, Amanda Askell, Pamela Mishkin, Jack Clark, et~al.
\newblock Learning transferable visual models from natural language supervision.
\newblock In \emph{International conference on machine learning}, pages 8748--8763. PmLR, 2021.

\bibitem[Riedmiller(2005)]{riedmiller2005neural}
Martin Riedmiller.
\newblock Neural fitted q iteration--first experiences with a data efficient neural reinforcement learning method.
\newblock In \emph{European conference on machine learning}, pages 317--328. Springer, 2005.

\bibitem[Russell et~al.(1995)Russell, Norvig, and Intelligence]{russell1995modern}
Stuart Russell, Peter Norvig, and Artificial Intelligence.
\newblock A modern approach.
\newblock \emph{Artificial Intelligence. Prentice-Hall, Egnlewood Cliffs}, 25\penalty0 (27):\penalty0 79--80, 1995.

\bibitem[Savinov et~al.(2018)Savinov, Dosovitskiy, and Koltun]{savinov2018semi}
Nikolay Savinov, Alexey Dosovitskiy, and Vladlen Koltun.
\newblock Semi-parametric topological memory for navigation.
\newblock \emph{arXiv preprint arXiv:1803.00653}, 2018.

\bibitem[Schaul et~al.(2015)Schaul, Horgan, Gregor, and Silver]{schaul2015universal}
Tom Schaul, Daniel Horgan, Karol Gregor, and David Silver.
\newblock Universal value function approximators.
\newblock In \emph{International conference on machine learning}, pages 1312--1320. PMLR, 2015.

\bibitem[Shah et~al.(2025)Shah, Yang, Yang, Zheng, and Eysenbach]{shah2025structured}
Devan Shah, Owen Yang, Daniel Yang, Chongyi Zheng, and Benjamin Eysenbach.
\newblock Structured response diversity with mutual information.
\newblock In \emph{Workshop on Scaling Environments for Agents}, 2025.
\newblock URL \url{https://openreview.net/forum?id=xL7Bt4jS2U}.

\bibitem[Sikchi et~al.(2025)Sikchi, Tirinzoni, Touati, Xu, Kanervisto, Niekum, Zhang, Lazaric, and Pirotta]{sikchi2025fast}
Harshit Sikchi, Andrea Tirinzoni, Ahmed Touati, Yingchen Xu, Anssi Kanervisto, Scott Niekum, Amy Zhang, Alessandro Lazaric, and Matteo Pirotta.
\newblock Fast adaptation with behavioral foundation models.
\newblock In \emph{Reinforcement Learning Conference}, 2025.
\newblock URL \url{https://openreview.net/forum?id=soeW8RGo1N}.

\bibitem[Sugiyama et~al.(2008)Sugiyama, Suzuki, Nakajima, Kashima, Von~B{\"u}nau, and Kawanabe]{sugiyama2008direct}
Masashi Sugiyama, Taiji Suzuki, Shinichi Nakajima, Hisashi Kashima, Paul Von~B{\"u}nau, and Motoaki Kawanabe.
\newblock Direct importance estimation for covariate shift adaptation.
\newblock \emph{Annals of the Institute of Statistical Mathematics}, 60\penalty0 (4):\penalty0 699--746, 2008.

\bibitem[Sutton et~al.(1998)Sutton, Barto, et~al.]{sutton1998reinforcement}
Richard~S Sutton, Andrew~G Barto, et~al.
\newblock \emph{Reinforcement learning: An introduction}, volume~1.
\newblock MIT press Cambridge, 1998.

\bibitem[Tarasov et~al.(2023)Tarasov, Nikulin, Akimov, Kurenkov, and Kolesnikov]{tarasov2023corl}
Denis Tarasov, Alexander Nikulin, Dmitry Akimov, Vladislav Kurenkov, and Sergey Kolesnikov.
\newblock Corl: Research-oriented deep offline reinforcement learning library.
\newblock \emph{Advances in Neural Information Processing Systems}, 36:\penalty0 30997--31020, 2023.

\bibitem[Tassa et~al.(2018)Tassa, Doron, Muldal, Erez, Li, Casas, Budden, Abdolmaleki, Merel, Lefrancq, et~al.]{tassa2018deepmind}
Yuval Tassa, Yotam Doron, Alistair Muldal, Tom Erez, Yazhe Li, Diego de~Las Casas, David Budden, Abbas Abdolmaleki, Josh Merel, Andrew Lefrancq, et~al.
\newblock Deepmind control suite.
\newblock \emph{arXiv preprint arXiv:1801.00690}, 2018.

\bibitem[Team et~al.(2023)Team, Anil, Borgeaud, Alayrac, Yu, Soricut, Schalkwyk, Dai, Hauth, Millican, et~al.]{team2023gemini}
Gemini Team, Rohan Anil, Sebastian Borgeaud, Jean-Baptiste Alayrac, Jiahui Yu, Radu Soricut, Johan Schalkwyk, Andrew~M Dai, Anja Hauth, Katie Millican, et~al.
\newblock Gemini: a family of highly capable multimodal models.
\newblock \emph{arXiv preprint arXiv:2312.11805}, 2023.

\bibitem[Tirinzoni et~al.(2025)Tirinzoni, Touati, Farebrother, Guzek, Kanervisto, Xu, Lazaric, and Pirotta]{tirinzoni2025zeroshot}
Andrea Tirinzoni, Ahmed Touati, Jesse Farebrother, Mateusz Guzek, Anssi Kanervisto, Yingchen Xu, Alessandro Lazaric, and Matteo Pirotta.
\newblock Zero-shot whole-body humanoid control via behavioral foundation models.
\newblock In \emph{The Thirteenth International Conference on Learning Representations}, 2025.
\newblock URL \url{https://openreview.net/forum?id=9sOR0nYLtz}.

\bibitem[Todorov et~al.(2012)Todorov, Erez, and Tassa]{todorov2012mujoco}
Emanuel Todorov, Tom Erez, and Yuval Tassa.
\newblock Mujoco: A physics engine for model-based control.
\newblock In \emph{2012 IEEE/RSJ international conference on intelligent robots and systems}, pages 5026--5033. IEEE, 2012.

\bibitem[Touati and Ollivier(2021)]{touati2021learning}
Ahmed Touati and Yann Ollivier.
\newblock Learning one representation to optimize all rewards.
\newblock \emph{Advances in Neural Information Processing Systems}, 34:\penalty0 13--23, 2021.

\bibitem[Touati et~al.(2022)Touati, Rapin, and Ollivier]{touati2022does}
Ahmed Touati, J{\'e}r{\'e}my Rapin, and Yann Ollivier.
\newblock Does zero-shot reinforcement learning exist?
\newblock In \emph{The Eleventh International Conference on Learning Representations}, 2022.

\bibitem[Wang et~al.(2018)Wang, Xiong, Han, Liu, Zhang, et~al.]{wang2018exponentially}
Qing Wang, Jiechao Xiong, Lei Han, Han Liu, Tong Zhang, et~al.
\newblock Exponentially weighted imitation learning for batched historical data.
\newblock \emph{Advances in Neural Information Processing Systems}, 31, 2018.

\bibitem[Wantlin et~al.(2025)Wantlin, Zheng, and Eysenbach]{wantlin2025consistent}
Kathryn Wantlin, Chongyi Zheng, and Benjamin Eysenbach.
\newblock Consistent zero-shot imitation with contrastive goal inference.
\newblock \emph{arXiv preprint arXiv:2510.17059}, 2025.

\bibitem[Wu et~al.(2018)Wu, Tucker, and Nachum]{wu2018laplacian}
Yifan Wu, George Tucker, and Ofir Nachum.
\newblock The laplacian in rl: Learning representations with efficient approximations.
\newblock \emph{arXiv preprint arXiv:1810.04586}, 2018.

\bibitem[Xie et~al.(2025)Xie, Daw, and Eysenbach]{xie2025lowrank}
Eva~Yi Xie, Nathaniel~D. Daw, and Benjamin Eysenbach.
\newblock Low-rank successor representations capture human-like generalization.
\newblock In \emph{UniReps: 3rd Edition of the Workshop on Unifying Representations in Neural Models}, 2025.
\newblock URL \url{https://openreview.net/forum?id=nJJtcmDVvp}.

\bibitem[Yarats et~al.(2022)Yarats, Brandfonbrener, Liu, Laskin, Abbeel, Lazaric, and Pinto]{yarats2022don}
Denis Yarats, David Brandfonbrener, Hao Liu, Michael Laskin, Pieter Abbeel, Alessandro Lazaric, and Lerrel Pinto.
\newblock Don't change the algorithm, change the data: Exploratory data for offline reinforcement learning.
\newblock \emph{arXiv preprint arXiv:2201.13425}, 2022.

\bibitem[Zhang et~al.(2021)Zhang, McAllister, Calandra, Gal, and Levine]{zhang2021learning}
Amy Zhang, Rowan~Thomas McAllister, Roberto Calandra, Yarin Gal, and Sergey Levine.
\newblock Learning invariant representations for reinforcement learning without reconstruction.
\newblock In \emph{International Conference on Learning Representations}, 2021.
\newblock URL \url{https://openreview.net/forum?id=-2FCwDKRREu}.

\bibitem[Zhang et~al.(2017)Zhang, Springenberg, Boedecker, and Burgard]{zhang2017deep}
Jingwei Zhang, Jost~Tobias Springenberg, Joschka Boedecker, and Wolfram Burgard.
\newblock Deep reinforcement learning with successor features for navigation across similar environments.
\newblock In \emph{2017 IEEE/RSJ International Conference on Intelligent Robots and Systems (IROS)}, pages 2371--2378. IEEE, 2017.

\bibitem[Zheng et~al.(2023)Zheng, Eysenbach, Walke, Yin, Fang, Salakhutdinov, and Levine]{zheng2023stabilizing}
Chongyi Zheng, Benjamin Eysenbach, Homer Walke, Patrick Yin, Kuan Fang, Ruslan Salakhutdinov, and Sergey Levine.
\newblock Stabilizing contrastive rl: Techniques for robotic goal reaching from offline data.
\newblock \emph{arXiv preprint arXiv:2306.03346}, 2023.

\bibitem[Zheng et~al.(2024{\natexlab{a}})Zheng, Salakhutdinov, and Eysenbach]{zheng2024contrastive}
Chongyi Zheng, Ruslan Salakhutdinov, and Benjamin Eysenbach.
\newblock Contrastive difference predictive coding.
\newblock \emph{The Twelfth International Conference on Learning Representations}, 2024{\natexlab{a}}.
\newblock URL \url{https://openreview.net/forum?id=0akLDTFR9x}.

\bibitem[Zheng et~al.(2024{\natexlab{b}})Zheng, Tuyls, Peng, and Eysenbach]{zheng2024can}
Chongyi Zheng, Jens Tuyls, Joanne Peng, and Benjamin Eysenbach.
\newblock Can a misl fly? analysis and ingredients for mutual information skill learning.
\newblock \emph{arXiv preprint arXiv:2412.08021}, 2024{\natexlab{b}}.

\bibitem[Zheng et~al.(2025)Zheng, Park, Levine, and Eysenbach]{zheng2025intention}
Chongyi Zheng, Seohong Park, Sergey Levine, and Benjamin Eysenbach.
\newblock Intention-conditioned flow occupancy models.
\newblock \emph{arXiv preprint arXiv:2506.08902}, 2025.

\bibitem[Zhou et~al.(2025)Zhou, Peng, Li, Levine, and Kumar]{zhou2025efficient}
Zhiyuan Zhou, Andy Peng, Qiyang Li, Sergey Levine, and Aviral Kumar.
\newblock Efficient online reinforcement learning fine-tuning need not retain offline data.
\newblock In \emph{The Thirteenth International Conference on Learning Representations}, 2025.
\newblock URL \url{https://openreview.net/forum?id=HN0CYZbAPw}.

\end{thebibliography}
\end{document}